\documentclass{article}

\usepackage[preprint]{neurips_2025}

\usepackage[utf8]{inputenc} % allow utf-8 input
\usepackage[T1]{fontenc}    % use 8-bit T1 fonts
\usepackage{hyperref}       % hyperlinks
\usepackage{url}            % simple URL typesetting
\usepackage{booktabs}       % professional-quality tables
\usepackage{amsfonts}       % blackboard math symbols
\usepackage{nicefrac}       % compact symbols for 1/2, etc.
\usepackage{microtype}      % microtypography
\usepackage{xcolor}         % colors
\usepackage{colortbl}

\usepackage{amsmath}
\usepackage[makeroom]{cancel}
\usepackage{algorithm}
\usepackage{algpseudocode}  %do we need this idk
\usepackage{amssymb}
\usepackage{multicol}
\usepackage{caption}
\usepackage{subcaption}
\usepackage{tcolorbox}
\usepackage{physics}
\usepackage{amsthm}

\newcommand{\Y}{{\boldsymbol Y}}

\def\ddefloop#1{\ifx\ddefloop#1\else\ddef{#1}\expandafter\ddefloop\fi}
\def\ddef#1{\expandafter\def\csname c#1\endcsname{\ensuremath{\mathcal{#1}}}}
\ddefloop ABCDEFGHIJKLMNOPQRSTUVWXYZ\ddefloop
\def\ddef#1{\expandafter\def\csname s#1\endcsname{\ensuremath{\mathsf{#1}}}}
\ddefloop ABCDEFGHIJKLMNOPQRSTUVWXYZ\ddefloop
\def\ddef#1{\expandafter\def\csname b#1\endcsname{\ensuremath{\mathbf{#1}}}}
\ddefloop ABCDEFGHIJKLMNOPQRSTUVWXYZ\ddefloop
\def\ddef#1{\expandafter\def\csname b#1\endcsname{\ensuremath{\mathbf{#1}}}}
\ddefloop abcdeghijklmnopqrstuvwxyz\ddefloop

\newcommand{\E}{\mathop\mathbb{E}}

\newcommand{\R}{\mathbb{R}}

\usepackage[makeroom]{cancel}
\usepackage{mdframed}
\newtheorem{theorem}{Theorem}

\newtheorem{assumption}{Assumption}

\usepackage{color}

\title{Beyond Linear Diffusions: Improved Representations for Rare Conditional Generative Modeling}

\author{%
  Kulunu Dharmakeerthi\thanks{This work was completed during an internship at JPMorgan Chase with the Quantitative Research Markets Capital Group.} \\
  University of Chicago\\
  Chicago, IL \\
  \texttt{kulunud@uchicago.edu} 
  \And
  Yousef El-Laham \\
  J.P. Morgan AI Research \\ 
  New York, NY \\
  \texttt{yousef.el-laham@jpmchase.com}
  \AND
  Henry H. Wong\\
  J.P. Morgan Quantitative Research \\
  New York, NY \\
  \texttt{henry.h.wong@jpmorgan.com} 
  \And
  Vamsi K. Potluru \\
  J.P. Morgan AI Research \\ 
  New York, NY \\
  \texttt{vamsi.k.potluru@jpmchase.com}
  \And
  Changhong He\\
  J.P. Morgan Quantitative Research \\
  New York, NY \\
  \texttt{changhong.he@jpmorgan.com} 
  \And
  Taosong He \\
  J.P. Morgan Quantitative Research \\
  New York, NY \\
  \texttt{taosong.he@jpmorgan.com} 
}

\begin{document}

\maketitle

% \YE{Reworking abstract - cevt as a tractable case of a broader modeling philsophy - feel free to edit}
% \VP{Use your "\textbackslash FIRST LAST" name initials for adding comments}
\begin{abstract}
Diffusion models have emerged as powerful generative frameworks with widespread applications across machine learning and artificial intelligence systems. While current research has predominantly focused on linear diffusions, these approaches can face significant challenges when modeling a conditional distribution, $P(Y|X=x)$, when $P(X=x)$ is small. In these regions, few samples, if any, are available for training, thus modeling the corresponding conditional density may be difficult. Recognizing this, we show it is possible to adapt the data representation and forward scheme so that the sample complexity of learning a score-based generative model is small in low probability regions of the conditioning space. Drawing inspiration from conditional extreme value theory we characterize this method precisely in the special case in the tail regions of the conditioning variable, $X$. We show how diffusion with a data-driven choice of nonlinear drift term is best suited to model tail events under an appropriate representation of the data. Through empirical validation on two synthetic datasets and a real-world financial dataset, we demonstrate that our tail-adaptive approach significantly outperforms standard diffusion models in accurately capturing response distributions at the extreme tail conditions.
%These results highlight the potential of our framework for applications requiring precise modeling of rare but consequential events, such as financial stress testing and climate modeling.%\textbf{ (probably should say advantages of method is also theoretically justified)}
% Simply put, we aim to model $P(Y|X=x)$ with score-based diffusions and specifically target performance in rare events:  $\{X=x\}$ where $P(X=x)$ is small. 
\end{abstract}

\section{Introduction}
In recent years, diffusion models have emerged as among the most powerful generative modeling techniques for synthesizing data across a diverse set of modalities. From image generation to audio synthesis and time series modeling, these models have demonstrated superior capabilities for capturing intricate data distributions as compared to other generative frameworks. %, such as  generative adversarial networks (GANs) and variational autoencoders (VAEs). 
The work \cite{ho2020denoising} introduced denoising diffusion probabilistic models (DDPMs), which frame the generative process by defining a forward process that gradually transforms data into noise, followed by a learned reverse process that reconstructs data from noise. This approach has since been extended to a continuous-time formulation using Langevin diffusions, providing a mathematically elegant framework that connects stochastic processes with generative modeling. The continuous-time formulation views this as a stochastic differential equation (SDE), where the forward process follows a Langevin diffusion that converges to a standard multivariate Gaussian distribution. This perspective has enabled significant theoretical advances while maintaining state-of-the-art empirical performance across applications.

Conditional diffusion models extend this framework by additionally incorporating conditioning information to guide the generation process. %Typically implemented through classifier guidance or classifier-free approaches, these models have enabled remarkable control over the generative process.  
However, a fundamental challenge emerges when dealing with extreme values in the sample space of the conditioning, where data is inherently sparse. Traditional diffusion models struggle to accurately capture conditional distributions in these tail regions, particularly when the underlying distributions deviate significantly from Gaussian assumptions. This limitation becomes especially problematic in domains where rare but consequential events drive critical decisions, such as financial risk assessment and climate modeling. To effectively sample from a conditional distribution $P(Y|X=x)$ using score-based diffusion models, we need to estimate a sequence of score functions, $\{\nabla \log p_{\mu_t(.|x)}\}_{t=0}^T$. Here, $ p_{\mu_t(.|x)}$ refers to the marginal density of the conditional distribution $t$ steps into a (discretized) Langevin diffusion. The bottleneck is estimating these conditional score functions at low-probability, or rare, conditions. When $P(X=x)$ is small, it is unlikely that we see enough samples in our training data to estimate $\nabla \log p_{\mu_t(.|x)}$ accurately. If the score functions in these low probability regions have high sample complexity, sampling from tail conditions seems an improbable task. 
 
 %In practice, one cannot circumvent this issue by prompting the population phenomenon for more data. Even if it is feasible to resample from the population distribution, the events we are interested in modeling are, by definition, low probability. It is unlikely we see relevant data by resampling. 

% In this work, we will motivate a data-adaptive methodology for score-based diffusions that directly tackles this sampling difficulty at the tail. We couple two observations: (i) the principal difficulty of conditional diffusion is the need to learn a complex function where we have few samples, and (ii) the complexity of this function can be manipulated by our choice of diffusion scheme, and a certain transformation of the data. Put simply, our method aims to induce the following characteristic in the sequence of denoising maps we need to learn: the conditional ``denoising'' functions have low sample complexity in regions where $P(X=x)$ is small. The method to do so will involve a transformation of the data, and a data-driven choice of (non-linear) diffusion process. To highlight the effectiveness of this approach, we describe the method in detail, establishing in closed-form this transformation and new diffusion scheme, when our data follows a (mild) extreme value assumption.

We present a data-adaptive methodology for score-based diffusions that addresses this challenge through two key insights: (i) conditional diffusion requires learning complex functions with few samples, and (ii) function complexity can be controlled through the diffusion scheme and data transformation. Our method ensures the conditional denoising functions maintain low sample complexity where $P(X=x)$ is small, using data transformation and a data-driven nonlinear diffusion process. We demonstrate this approach in detail under mild extreme value assumptions. Specifically, our work explores nonlinear conditional diffusion modeling with tail-adaptive drift schemes. We examine the method where data follows extreme value assumptions \cite{hefftawn2004, heffernan2007limit, KEEF2013396}. Our contributions include:

%\KD{Not necessary i think:}
%\textit{Conditional extreme value theory (CEVT) offers a principled approach to modeling the behavior of random variables under extreme conditions. The seminal work of Heffernan and Tawn (\cite{hefftawn2004}) established a semi-parametric framework for understanding how the conditional distribution of a response variable $Y$ behaves as the conditioning variable $X$ approaches extreme values. Their key insight reveals that under mild assumptions, extreme conditional distributions often admit a simplified representation through appropriate normalization functions. Specifically, when marginals follow standard Laplace distributions, the relationship $Y = a(X) + b(X)\cdot Z$ emerges, where $Z$ follows a distribution independent of $X$, and the functions $a(X)$ and $b(X)$ often take simple parametric forms. This result has enabled practitioners to extrapolate conditional distributions into regions with limited or no observed data, making it invaluable for risk assessment and scenario analysis in financial risk management.}

% Specifically, our work explores nonlinear conditional diffusion modeling, motivating how to choose new, non-linear drift schemes in a tail-adaptive manner. As proof of concept we explore, in detail, the method in the special case where the underlying data follows certain extreme value assumptions \cite{hefftawn2004, heffernan2007limit, KEEF2013396}. Our main contributions are as follows:
\begin{enumerate}
    \item We identify current limitations of standard linear diffusion models with Gaussian equilibrium for conditional generation under extreme tail conditions with limited samples, based on recent neural network sample complexity results.
    \item We propose a novel score-based diffusion method that addresses the aforementioned sample complexity issue by utilizing well-designed data transformation and nonlinear Langevin diffusions. We explore this method in detail assuming the data follows some mild extreme value conditions (CEVT); although, we emphasize that our broader modeling philosophy is agnostic to any data assumptions. %However, under some mild assumptions, we can readily characterize the proposed method explicitly.     
    \item We validate our method on synthetic and real financial datasets, demonstrating superior conditional distribution modeling at tail extremes compared to standard diffusion variants.
\end{enumerate}

\section{Background}

\subsection{The Difficulty of Conditional Diffusion}
Conditional diffusion models frame sampling as the time reversal of a noising process governed by a diffusion SDE. A forward diffusion process $\{ Y_t\}_{t = 0}^T$ is indexed by a continuous time variable $t \in [0, T]$, such that $Y_0 \sim \mu_0(\cdot|X = x)$ is our sampling target, and $Y_T \sim \mu_T \approx \pi$, admits a tractable form to generate samples efficiently. A continuous time evolution, an Ito SDE, governs the forward process $\mu_0 \to \mu_T$. We limit ourselves to Langevin processes. The forward Langevin diffusion process is a stochastic differential equation of the form,
\begin{align}
	\label{eqn:langevin}
	\dd Y_t = -\nabla f(Y_t) \dd t + \sqrt{2\beta^{-1}} \dd B_t, \quad Y_0 \sim \mu_0(\cdot|X=x) \;.%~~ \nabla f\in C^\infty_c(Y, Y)
\end{align}
where the conditional probability measure of $Y_t$ is denoted $\mu_t(\cdot|x)$ with density $p_{\mu_t(\cdot|x)}$, $\{B_t\}_{t\geq 0}$ denotes a Brownian motion, and $\beta>0$ is a scale parameter that the determines the noise level of the diffusion. Under mild conditions on $f$, this evolution admits $e^{-f}$ as equilibrium density as $t \to \infty$.
% \begin{align*}
%     &Y_{t+\eta} = Y_t - \eta \cdot \nabla f(Y_t) + \sqrt{2\beta^{-1} \eta} \cdot \cN(0, 1)\\
%     &Y_{t-\eta} = Y_t - \eta \cdot (\nabla f(Y_t) +  2\beta^{-1}\nabla \log p_{\mu_t}(Y_t) ) + \sqrt{2\beta^{-1} \eta} \cdot \cN(0, 1)
% \end{align*}
% We denote the conditional probability measure of $Y_t$ as $\mu_t(\cdot|x)$ and its density by $p_{\mu_t(\cdot|x)}$. In the continuous time limit, it is well known that the above evolution admits $e^{-f}$ as the equilibrium density ($t \to \infty$) under mild conditions on $f$.
%evolves according to the Fokker-Planck partial %differential equation
%\begin{align}
%	\label{eqn:fokker-planck}
%	\partial_t \rho_t =  \nabla \cdot \big( \rho_t ( \nabla f + \beta^{-1} \nabla \log \rho_t ) \big) \;.
%\end{align}
Backward denoising uses the reverse-time SDE \cite{ANDERSON1982313}:
\begin{align}
	\label{eqn:rev-langevin}
	\dd Y_t^{\leftarrow} = -(\nabla f(Y_t^{\leftarrow}) +  2\beta^{-1}\nabla \log p_{\mu_t(\cdot|X)}(Y_t^{\leftarrow}) )\dd t + \sqrt{2\beta^{-1}} \dd \bar{B}_t, \quad Y_T^{\leftarrow} \sim \mu_T,
\end{align}
% The SDE admits an equivalent ODE formulation, 
% \begin{align}
% 	\label{eqn:rev-ode}
% 	\dd Y_t^{\leftarrow} = -(\nabla f(Y_t^{\leftarrow}) +  \beta^{-1}\nabla \log p_{\mu_t(\cdot|X)}(Y_t^{\leftarrow}) )\dd t, \quad Y_T^{\leftarrow} \sim \mu_T \;.
% \end{align}
where we use $Y_t^{\leftarrow}$ to denote the time-reversal and $\{\bar{B}_t\}_{t\geq 0}$ denotes another Brownian motion.
%In the geometry of $\cP_2(\R)$, the infintesimal backward step can be characterized as a measure transportation,
%\begin{align}
%	\label{eqn:forward-one-step}
%	\bar \mu_{t} := \bb^{\mu, \eta}_{\#} \mu_{t+\eta}, ~\text{where}~ \bb^{\mu, \eta} = \bi + \eta (\nabla f + \beta^{-1} \nabla \log p_{\mu_{t+\eta}}) \;.
%\end{align}

\subsubsection{Denoising Complexity}
\label{sss: denoising}
Implementing backward denoising process via \eqref{eqn:rev-langevin} % or \ref{eqn:rev-ode}, 
requires learning the conditional score function $\nabla \log p_{\mu_t(\cdot|x)}$. We instead target the estimation of the function, $\nabla f + \beta^{-1} \nabla \log p_{\mu_t(\cdot|x)}$. The complexity of these functions determines the sample size required for accurately estimation. For neural network predictors, non-asymptotic bounds have been established in \cite{farrell2021deep} that relate target smoothness to estimation accuracy (see Theorem \ref{thm: smoothness_nn_rate} in Appendix \ref{app: theory}).

%\VP{can we define $d$ as the dimensionality for what? $x$? Can we define $C$?}:
% \VP{Why not state the above as a theorem if it is snipped from the reference where it probably is?}
% \VP{Can we make the observation that the approximation error is $O^*(1/n)$ for the best setting and $O^*(1/\sqrt{n})$ in the worse case where we ignore log factors?}

Sampling from $P(Y|X=x)$, requires accurately learning the sequence of maps: 
$$\{\bB_t(y; x) \}_{t=0}^T = \{\nabla f(y) + \beta^{-1} \nabla \log p_{\mu_t(.|x)}(y)\}_{t=0}^T, \ \forall x \in \cX$$
Since few training examples exists for rare events where $P(X=x)$ is small, accurate estimation of the denoising maps in these ``rare regions" is futile,  preventing effective denoising. More rigorously, we can adapt theoretical results from \cite{song2021maximum} to show that the accuracy of denoising is directly tied to how well the denoising maps are learned. Let $\mu_\theta$ denote the estimated density resulting from \eqref{eqn:rev-langevin} after appropriately estimating the score function sequence,  $s_\theta(y;t,x) \approx \bB_t(y;x)$. If $\mu_0$ refers to the target, then the Kullback-Leibler (KL) divergence between the two can be upper bounded by the integrated error of score estimation (see Appendix \ref{app: theory} for a proof for completeness): 
\begin{align*}
    % s_\theta(y;t,x) &\approx \bB_t(y;x) = \nabla f(y) + \beta^{-1} \nabla \log p_{\mu_t(.|x)}(y)\\
    KL(\mu_0(\cdot|x)||\mu_\theta(\cdot|x)) &\lesssim  \int_0^T \E_{p_{\mu_t(\cdot|x)}(y)}[\|(\nabla f(y) + \beta^{-1}\nabla \log p_{\mu_t(\cdot|x)}(y))  - s_\theta(y; t,x)\|^2]dt \\
    &=\int_0^T \E_{p_{\mu_t(\cdot|x)}(y)}[\|(\bB_t(y;x) - s_\theta(y; t,x)\|^2]dt.
\end{align*}
% This result is included in Appendix B for completeness.

\paragraph{Linear Gaussian Dynamics:} Standard score-based diffusion models employ the Ornstein-Uhlenbeck process (with $f(x) = \tfrac{1}{2}x^2$):
% For computational ease, score-based diffusions typically rely on the Ornstein-Uhlenbeck (OU) process. Setting $f(x) = \tfrac{1}{2}x^2$ in \eqref{eqn:langevin}, we have a process that admits the standard Gaussian density at stationarity:
\begin{align*}
    \dd Y_t = -Y_t \dd t + \sqrt{2\beta^{-1}} \dd B_t ,\quad Y_0 \sim \mu_0(\cdot|x),
\end{align*}
which yields at Gaussian stationary distribution. For this case, the denoising sequence in \eqref{eqn:rev-langevin} becomes:
\begin{align*}
    \{\bB_t(y;x)\}_{t=0}^T = \{y + \beta^{-1} \nabla \log p_{\mu_t(\cdot|x)}(y)\}_{t=0}^T
\end{align*}
As previously mentioned, when $P(X=x)$ is small and $\{\bB_t(y;x)\}_{t=0}^T$ complex, this standard paradigm faces sample complexity challenges.

% As discussed, the difficulty of sampling is directly tied to the difficulty of learning the denoising plans, $\{\bB_t(y;x)\}_{t=0}^T$. However, for $\{X = x\}$, rare, and $\{\bB_t(y;x)\}_{t=0}^T$ complex, learning  with the standard (OU) paradigm is challenging from a sample complexity perspective. 
% The number of samples we have access to is the unchangeable and fixed difficulty of the problem. How about the sample complexity of the sequence, $\{\bB_t(y;x)\}_{t=0}^T$?
% \begin{center}
%     \textit{Can we adapt the diffuse-denoise process to make this sequence easy to estimate?}
% \end{center}
%\VP{The reader is asked a question and left hanging as to the answer}

\subsection{Extreme Value Theory}
Extreme value theory characterizes the tail behavior of random variables. Classical work examines limiting behavior like $P(Y=y|Y>u) \to G(y)$ as $u \to \infty$. \cite{hefftawn2004} extends this to conditional distributions $P(Y|X=x)$ for large $x$, contrasting with traditional multivariate theory where all variables grow simultaneously.
The Heffernan-Tawn model \cite{KEEF2013396} is a flexible approach to model the conditional distribution $P(Y|X=x)$ when $x$ is large. For a broad class of dependency structures between $X$ and $Y$, this work establishes a semi-parametric relationship that allows one to model a broad range of asymptotic independence/dependence structures at the tail of the condition (see Appendix A for more details).

\begin{assumption}[CEVT \cite{hefftawn2004, KEEF2013396}] 
\label{assumption: cevt}
Suppose the marginals of $X$ and $Y$ are standard Laplace. Then,  as $X= x \to \infty$, we assume $X, Y$ admit the asymptotic dependency,
\begin{align}
\label{eqn:CEVT}
    \lim_{x \to \infty} P\bigg(\tfrac{Y - a(X)}{b(X)} < z |X = x\bigg) = G(z)
\end{align}
where $G$ is some distribution independent of $X$. In other words, for tail values, $X = x \to \infty$: %we model, 
    \begin{align}
    Y = a(X) + b(X) \cdot Z, \quad Z \sim G.
\end{align}
\end{assumption}
\section{Proposed Methodology}
% Simply put, our methodology aims to induce the following characteristic in a new sequence of denoising maps: 
In this section, we propose a general methodology that aims to ensure that denoising maps discussed in Section \ref{sss: denoising} maintain low sample complexity for rare conditions:
\begin{align}
\label{eqn:simple-phen}
    \{\bB_t(y;x)\}_{t=0}^T \ \text{is easy to estimate when  } P(X=x) \text{ is small.}
\end{align}
While we demonstrate this approach under CEVT assumptions for explicit characterization, the framework applies broadly—any transformation yielding favorable tail behavior suffices. The general procedure consists of three steps:
\begin{enumerate}
    \item Transform $(X,Y) \stackrel{T}{\to} (X^\star, Z)$ such that $P(Z|X^\star=x) \approx e^{-g}$ for rare $x$
    \item Design forward diffusion with $e^{-g}$ as the stationary density and train using the score-matching objective
    \item Sample $Z \sim P(Z|X^\star=x)$ and apply inverse transformation to recover $Y$.
\end{enumerate}
In the following, we describe an implementation of each of the above steps in the context of CEVT. 

% Under assumption (\ref{eqn:CEVT}), there is a clear sequence of pre-processing steps and adjustments to the diffusion scheme that must be taken for \ref{eqn:simple-phen} to hold. 

% However, we emphasize that there are many relationships between $X$ and $Y$ that will display behavior \ref{eqn:simple-phen} after a suitable transformation. See Synthetic Example 1 for reference. Thus, a reader should not view our method as relying on CEVT. Rather, under the CEVT model, a strategy to bring about this favorable behavior can be clearly delineated and shown to be effective. 
% In particular, under the CEVT assumption \ref{eqn:CEVT}, we will define a transformation, 
% \begin{align*}
%     (X,Y) \quad \to \quad (X^\star,\bar{Y})
% \end{align*}
% such that,
% \begin{align*}
%     \{\bB_t(y;x)\}_{t=0}^T \ \text{is easy to estimate when  } X^\star = x \text{ is large.}
% \end{align*}

% A snapshot of the steps that must be taken to arrive at this will be outlined below. Some implementation details are left to the appendix.

\subsection{Step 1 -- Data Transformation}
When CEVT holds, we can obtain explicit transformations $(X,Y) \to (X^\star,Z)$ ensuring \eqref{eqn:simple-phen} for large $X^\star$. Consider the following chain of transformations applied to $X$ and $Y$:
\begin{align}
    \left(X, Y\right) \quad \stackrel{\text{Laplace Marginals}}{\to} \quad \left(X^\star, Y^\star\right) \quad \stackrel{\text{Normalize}}{\to} \quad \left(X^\star, Z\right), 
\end{align}
where $Z=b(X^\star)^{-1}\left(Y^\star-a(X^\star)\right)$. In the first part of the transformation, we transform $(X, Y)$ to $(X^\star, Y^\star)$ such that marginal distributions of both $X^\star$ and $Y^\star$ are standard Laplace distributions. Since $X^\star$ and $Y^\star$ are both standard Laplace, we can further apply a normalization based on the Heffernan-Tawn model to transform $Y^\star\rightarrow Z$, where
\begin{equation*}
    Z = \frac{Y^\star-a(X^\star)}{b(X^\star)}
\end{equation*}
To apply this normalization, we learn the functions $a(x)$ and $b(x)$ (which often take simple parametric form) using maximum likelihood estimation with samples from the tail of $X^\star$.  Despite the small amount of samples available after partitioning the samples of $X^\star$ based on the tail, learning is plausible due to the simple structure of $a(x)$ and $b(x)$ (see Appendix \ref{app: normalizing_fuctions}) for details. After applying these sequence of transformation, we have a set of data of the random variables $(X_i, Z_i)$ that satisfy $P(Z|X^\star=x) \approx G$ for large values of $x$.  %Suppose we target sampling from the transformed distribution, $P( Z|X^\star= x) $.  By Assumption \ref{assumption: cevt}, we know that $P(Z|X^\star=x) \approx G$ for large $x$. 

\subsection{Step 2 -- Learning the Conditional of $Z$}
We learn the conditional distribution, $P(Z|X^\star)$, via score-based diffusion models. We provide pseudocode of our training procedure in Algorithm \ref{alg:cevt_diffusion_training} in Appendix \ref{app: algorithms}. In the following, we describe the design of the forward process of our conditional diffusion as well as our approach to score matching. We also provide an intuitive argument as to why the denoising maps for this diffusion model have low sample complexity at the tails of the condition.

\paragraph{Designing the Forward Process}  We implement the forward process via a simple Langevin diffusion, but choose the drift term, $\nabla g$, based on extreme value behavior in our observed data. In particular, by Assumption \ref{assumption: cevt}, for tail values in the condition, \{$X > x$, $x$ large\}, we model, 
\begin{align*}
    Z \sim G, \quad \text{and} \quad Z\perp X.
\end{align*}
However, in practice, the distribution $G$ is unknown. To approximate $G$, we train a lightweight density estimator on tail samples, $\{(X_i, Z_i): \ X_i > x\}$, to gauge the density of $G$, $e^{-g}$. We do so by comparing the smooth estimate to common parametric forms. For example, a wide range of easy-to-sample distributions admit an exponential form, $e^{-g}$, with convex $g$, such as Gaussian, Laplace, and Gumbel.

% In many cases, this asymptotic distribution has simple structure and resembles a common parametric distribution  (see table () in ())) .
Consider the following Langevin diffusion:
\begin{align*}
    dZ_t = -\nabla g(Z_t) dt +\sqrt{2}\dd B_t
\end{align*}
The Langevin diffusion above, for arbitrary convex $g$ does not admit path trajectories that can be expressed in closed form. In practice, we resort to a simple discretization,
\begin{align*}
    Z_{t+1} = Z_t- \eta \cdot \nabla g(Z_t) +\sqrt{2\eta}\cdot \cN(0, 1).
\end{align*}
We remark that the convergence of the discretized process, which amounts to unadjusted Langevin dynamics, is sensitive to the curvature of $g$. We elaborate in Appendix \ref{app: smoothness_f} how we can modify $g$ so that it is appropriately smooth, while still accurately capturing a stationary distribution close to $e^{-g}$. We also remark that with a nonlinear drift term $\nabla g$, we also lose the ancestral sampling property that score-based diffusions exploit for efficient training. That is, if $\nabla g$ is linear, then the t-step ahead distribution $P(Z_t|Z_0, X)$ is readily available. This is not possible for general $g$. Instead, in Appendix \ref{app: taylor_forward} we show how Taylor-approximations can enable faster sampling in the general case, similar in line to \cite{singhal2024s}.
% \begin{remark} Some remarks on the new forward process:
%     \begin{itemize}
%         \item A reader will note that the the convergence of the discretized process, which amounts to Unadjusted Langevin Dynamics, is sensitive to the curvature of $g$. We elaborate in the appendix how we can modify $g$ so that it is appropriately smooth, without losing accuracy.
%         \item We also lose the ancestral sampling property that score-based diffusions exploit for efficient training. That is, if $\nabla g$ is linear, then the t-step ahead distribution $P(Z_t|Z_0, X)$ is readily available. This is not possible for general $g$. Instead, in Appendix C we show how Taylor-approximations can enable faster sampling in the general case. 
%     \end{itemize}
% \end{remark}

\paragraph{Score Estimation}
Unlike standard diffusion models, rather than tracking the conditional score function, $\nabla \log p_{\mu_t(.|x)}$, we instead target $(\nabla g + \nabla \log p_{\mu_t(.|x)})$. We train a time-dependent conditional score model $s_\theta(z; x, t)$ based on a slightly modified learning objective.
\begin{align}
    \label{eq: score_matching_loss}
    \cL(\theta):= \E_t\bigg\{ \lambda(t) \E_{X, Z_0} \E_{Z_t|Z_0}[\| s_\theta(z; x, t) - (\nabla \log p_{\mu_{0t}(.|Z_0, X)}(Z_t)) + \nabla g(Z_t)) \|_2^2]\bigg\},
\end{align}
where $\lambda(t)$ is a weighting function that adjusts the importance of different time steps for the score-matching loss. Recent works explore how to learn the conditional score functions $s_\theta(Z; x, t)$ efficiently. %\textcolor{red}{(Yousef do you know citations)}. 
%It would be interesting to see how these methods can be plugged in and applied to learn the new target. 
For simplicity, we train based on the standard formulation based on Tweedie's formula. 

\paragraph{Why this Works?} 
Since $Z$ is constructed based on Assumption \ref{assumption: cevt}, at an event $\{ X^\star = x\}$, with $x$ large, the initial density of $\mu_0(Z|X^\star = x)$ will already be (approximately) at equilibria: 
\begin{align*}
    p_{\mu_0(\cdot|X^\star = x)} \approx e^{-g}
\end{align*}
Thus, at these extreme values of $X^\star$, the sequence of maps $ \{\nabla g +  \nabla \log p_{\bar{\mu}_t(.|x^\star)}\}_{t=0}^T \approx 0$, making them much easier to estimate. We display an example of this in Figure \ref{fig:cevtsec3}.

\begin{figure}
    \centering
    \includegraphics[width=\linewidth]{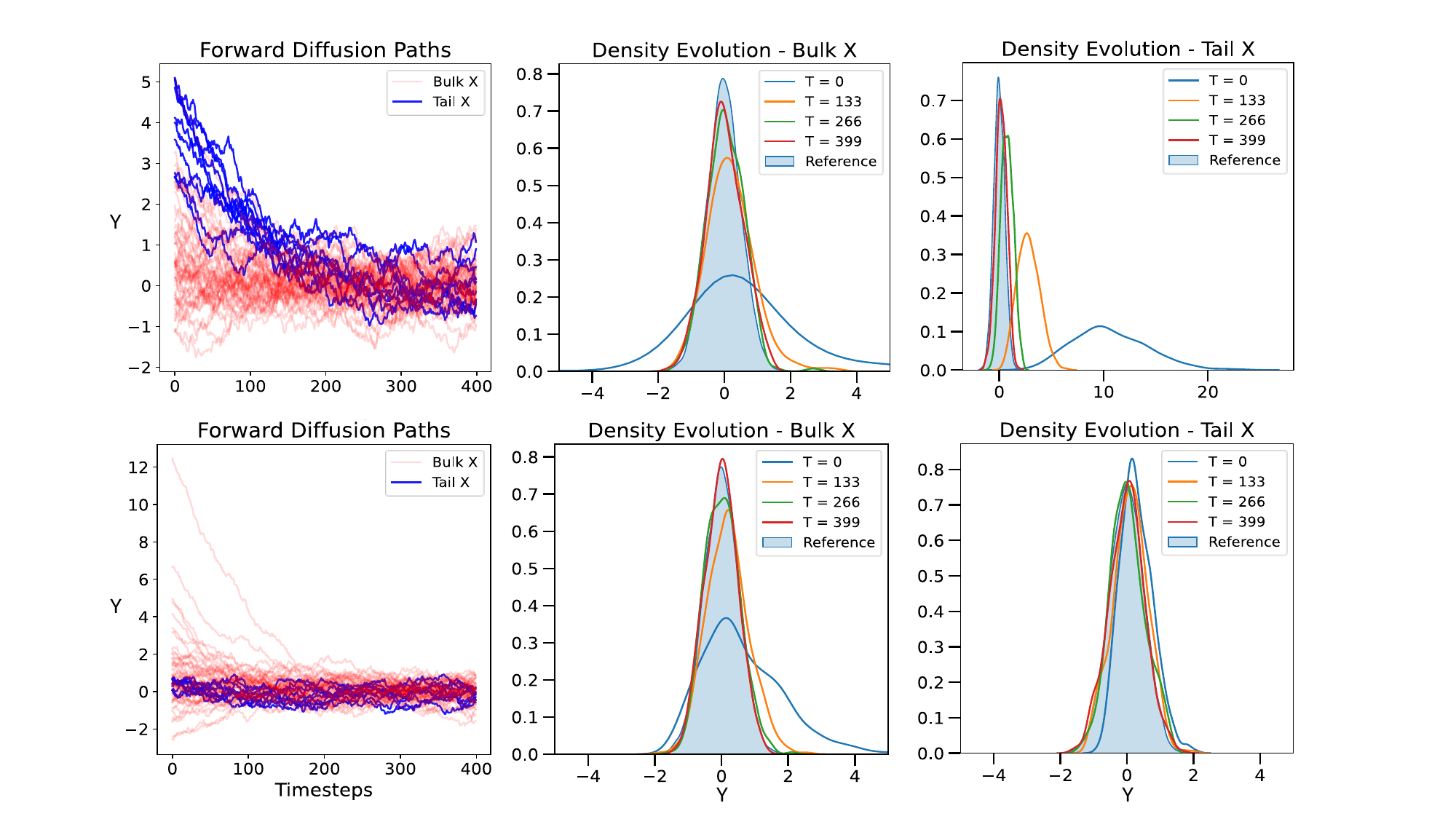}
    \caption{We visualize a forward diffusion before and after the transformation outlined in Section 3.2.  Before transformation, the Langevin diffusion induces quite dramatic changes in the conditional density  at tail events ($\{X=x\}, \ x$ very large). This can be seen by looking at the blue particle paths (top left) or the evolving density, $p_{\mu_t(\cdot|x)}(y)$, visualized in the top right plot. After taking the steps outlined in Section 3.2, the tail conditional density does not change dramatically in the forward diffusion. Compare the new particle paths in blue (bottom left plot), or the new conditional densities at time $t$ (bottom right plot). For tail, low-probability conditions, after transformation, the conditional density is already (nearly) at stationarity. Details can be found in Appendix A.2 }
    \label{fig:cevtsec3}
\end{figure}

\subsection{Step 3 -- Sampling}
For a desired value of $X$, we prompt the learned diffusion model to retrieve a sample $Z$ from $P(Z|X=x)$. Sampling is implemented via time-reversal as in \eqref{eqn:rev-langevin}  and substituting in the learned estimator, $s_\theta(z;x, t)$. 
\begin{align}
\label{eqn:score-reverse}
    &\dd \bar{Z}_t = -(2 s_\theta(Z_t;x, t) - \nabla f(\bar{Z_t})) \dd t + \sqrt{2\beta^{-1}}\dd \bar{B}_t %\quad \text{or}, \\
    % &\dd \bar{Z}_t = -s_\theta(Z_t;x, t)  \dd t
\end{align}
In practice, we use a simple Euler-Maryama discretization of the above to sample. Once we have a sample $Z \sim \hat{P}(Z|X)$, we convert it to a sample from our desired distribution by inverting the sequence of transformations,
\begin{align*}
     &Y^\star = a(X^\star) + b(X^\star) \cdot Z \\
     &Y = \hat{F}_Y^{-1} (F_{Lap}(Y^\star))
\end{align*}
Algorithm \ref{alg:reverse_diffusion_detail} in Appendix \ref{app: algorithms} provides pseudocode for our sampling procedure. 

\subsection{Generalization} 
As previously mentioned, there is potential to adapt {\bf Step 1} of the process to more general circumstances (e.g., if the CEVT assumption is not appropriate for the data of interest). %We reiterate that the skeleton of the proposed methodology is generic and can be adapted to any dataset to enhance performance for conditional generation where $P(X=x)$ is small:
% \begin{enumerate}
%     \item Find an invertible transformation, $(X,Y) \stackrel{T}{\to} (X^\star, Z)$ so that when $P(X=x)$ is small, the conditional distribution, $P(Z|X^\star=x)$, is ``similar'' to a log-concave distribution with density $e^{-g}$.
%     \item Devise a new forward process where $e^{-g}$ is the stationary density. Train a score-based diffusion to sample from $P(Z|X^\star=x)$
%     \item Sample, $Z \sim P(Z|X^\star=x)$ given input condition. Inverse transform $Z \stackrel{T^{-1}}{\to} Y$.
% \end{enumerate}
The challenge of adopting the  methodology is the finding the appropriate transformation of the data using a data-driven approach, perhaps using an approach similar to \cite{hu2023complexity}. We leave this for future work.

\section{Experiments}
In this section, we evaluate our proposed approach on two synthetic data examples and a real data example.
%The proposed is a valid scheme for sampling regardless of whether Assumption \ref{eqn:CEVT} holds. Indeed, the denoising diffusion scheme outlined above reflects the general SDE based approach outlined in \cite{song2020score}. On the other hand, this method will show a marked improvement over traditional schemes when \ref{eqn:simple-phen} holds.
%Simply put, CEVT is simply trying to map the following condition to the problem: 
% \begin{align*}
%     \{\bB_t(x)\}_{t=0}^T \ \text{is simple when  } P(X=x) \text{ small.}
% \end{align*}
% The first synthetic example will not appeal to CEVT, but instead a simple conditional relationship where the above holds trivially. 
For baselines, we consider two schemes for denoising. In the standard scheme, we sample $Y_T \sim \cN(0, 1)$, and provided a condition $X = x$, we apply the maps, 
\begin{align*}
    \{\bB^{Gauss}_t(y;x)\}_{t=0}^T
\end{align*}
In the new scheme, we first transform our data, $Y \stackrel{T}{\to} Z$. Sample $Z_T \sim e^{-g}$ and apply the maps, 
\begin{align*}
    \{\bB^{g}_t(z;x)\}_{t=0}^T
\end{align*}
Finally, invert the transform, $Z_0 \stackrel{T^{-1}}{\to} Y_0$. To fairly compare our new scheme to the standard scheme we make the following considerations. 
\paragraph{Neural Net Parametrization:} Fundamentally, we want to track how well $\bB^{Gauss}_t, \bB^g_t$ are learned. As a proxy we will look at sample quality. To enable a fair comparison, we deploy the same neural network architectures to learn each score network, which are standard feedforward neural networks. 

\paragraph{Forward Chain Length:} It is important to recognize that although the sequence of standard denoising maps, $\{\bB^{Gauss}_t(y;x)\}_{t=0}^T$, may have high sample complexity, due to the fast convergence of the forward OU process, the number of noise-steps necessary, $T$, may be smaller. This, in turn, may be beneficial for learning. For example, if one were to train a separate network for each noise-scale $t \in [T]$. We broach this gap by considering smoothed versions of $\nabla g$ for the generic scheme. This directly impacts speed of forward convergence, and is detailed in Appendix C. By choosing the smoothing parameters and step-size $\eta$ appropriately, we are able to use the same number of noise steps for each model. This compromise, between complexity of $\bB_t^g(y;x)$, $\eta$ and size of $T$ needs to be explored more rigorously. We leave this to future work. 

%\paragraph{Hyper parameter tuning} Yousef

\subsection{Synthetic Data Examples}
We consider two synthetic data experiments: a mean-shifted Laplace distribution; and correlated Gaussian distribution. We provide detailed plots with additional results for both synthetic examples can be found in Appendix \ref{app: additional_details_syn}.

\paragraph{Mean-Shifted Laplace Target.}
We consider the following data generating process:
\begin{align}
    X \sim \text{Pareto}(1), \quad Y \sim \frac{10}{X} + \text{Laplace}(0, 1)
\end{align}
Without appealing to CEVT, we see that as $X \to \infty$, $Y \sim \text{Laplace}(0, 1)$. This suggests we target standard Laplace as the equilibrium distribution of the forward process, without applying any transformation to the data. We run,
\begin{align*}
    Y_{t+\eta} = Y_t - \eta \cdot \nabla g_{Lap*}(Y_t) + \sqrt{2\eta}\cdot \cN(0, 1).
\end{align*}
Refer to Appendix C.1 to see form and justification for $\nabla g_{Lap*}$. We plot a comparison of the new method and standard Gaussian diffusion in Figure \ref{fig:synthex}a. The results of the figure demonstrate that in 90\% percentile, a standard diffusion model with Gaussian base distribution does not estimate the target distribution well, while the proposed approach without the CEVT transformation and an appropriately chosen Laplace base distribution more accurately capture the target. 

\paragraph{Correlated Gaussian Target.}
We consider the following data generating process:
\begin{align}
\begin{bmatrix}X\\Y\end{bmatrix} \sim \cN\left(\begin{bmatrix} 0\\0 \end{bmatrix}, \begin{bmatrix} 1 & \rho\\  \rho & 1 \end{bmatrix}\right),
\end{align}
where we set $\rho=0.4$. First we transform $(X,Y) \to (X^\star, Z)$ as per Algorithm 2. As detailed in the Appendix A.2, we know after this transform, $G \sim \cN(0, 2\rho^2(1-\rho^2))$. However, to mimic the data-driven procedure in practice, we instead gauge a form for $e^{-g}$ using tail samples. Based on this, we suggest targeting $\text{Gumbel}(0, 0.4)$ and run, 
\begin{align*}
    Z_{t+\eta} = Z_t - \eta \cdot \nabla g_{Gumb*}(Y_t) + \sqrt{2\eta}\cdot \cN(0, 1).
\end{align*}
Refer to Appendix C.1 to see form and justification for $\nabla g_{Gumb*}$. Once we sample from $P(Z|X^\star = x)$ via the new score-based diffusion, we transform back to the appropriate distribution via inverse CDF. We compare these samples to a traditional (linear) diffusion model that targets sampling from $P(Y|X=x)$. We plot this comparison in Figure \ref{fig:synthex}. From the figure, we can observe that the standard diffusion model fails to capture the target distribution at the tail of the condition, while the proposed method with the Gumbel base distribution almost perfectly captures it.

\begin{figure}[htbp]
    \centering
    \begin{subfigure}[b]{0.48\textwidth}
        \includegraphics[width=\textwidth]{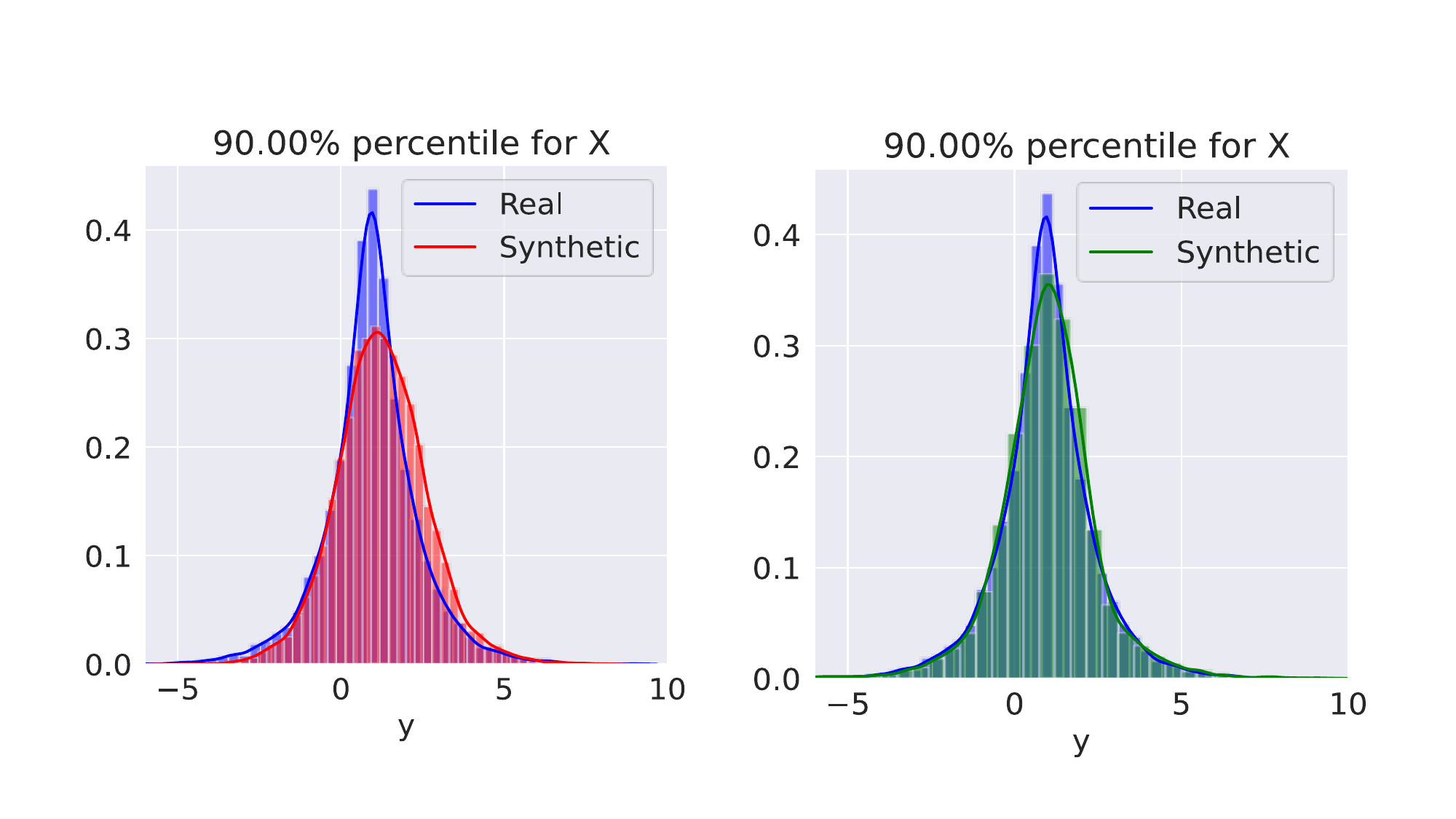}
        \caption{Mean-Shifted Laplace Target.}
    \end{subfigure}
    \hfill
    \begin{subfigure}[b]{0.48\textwidth}
        \includegraphics[width=\textwidth]{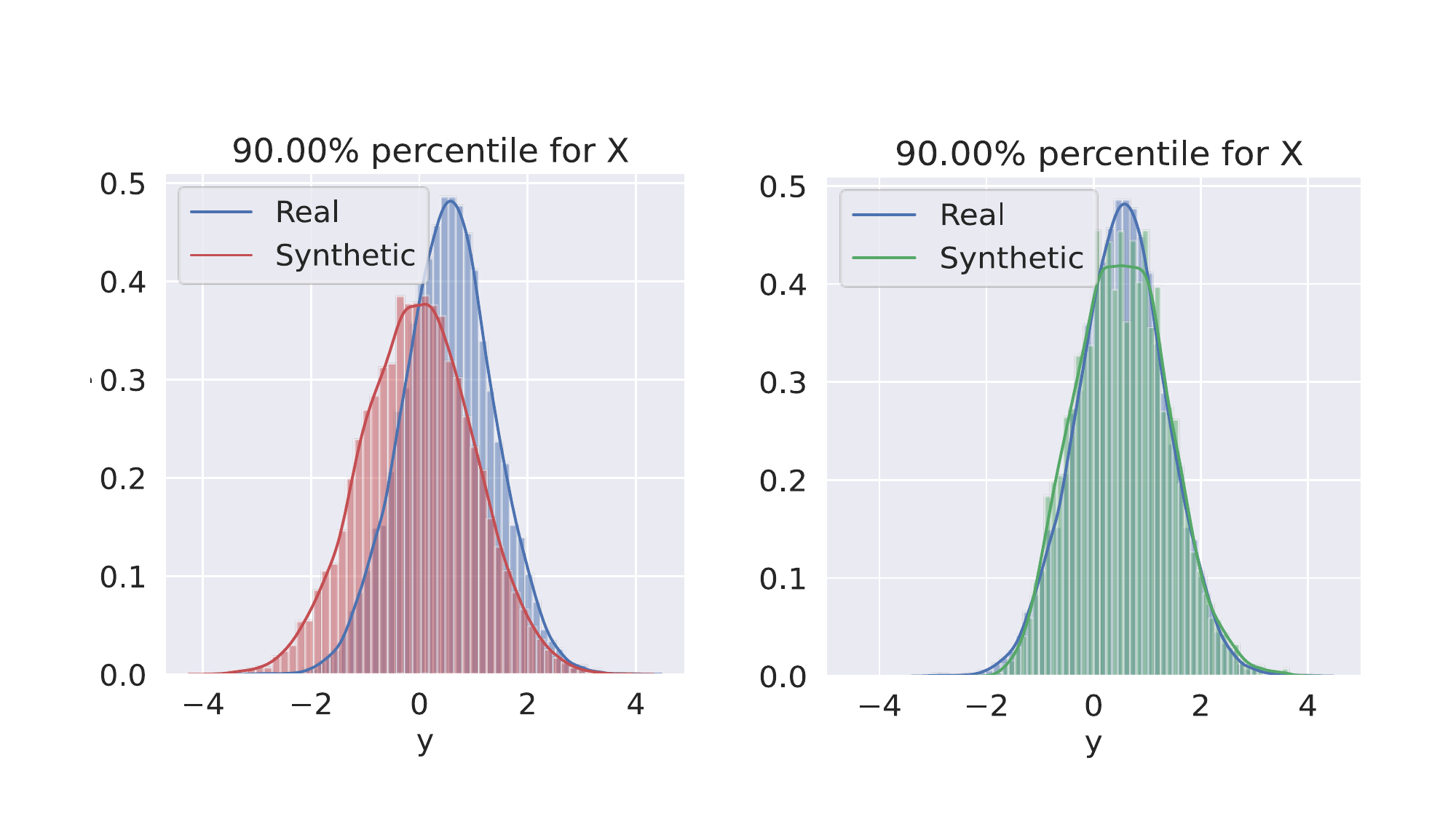}
        \caption{Correlated Gaussian Target.}
    \end{subfigure}
    \caption{In each subfigure, the left plot shows the standard diffusion with Gaussian base distribution, and the right plot shows our proposed method with a standard Laplace base distribution for the mean-shift example (no transformation) and a Gumbel base distribution for the multivariate Gaussian example (with learned CEVT transformation).}
    \label{fig:synthex}
\end{figure}
\subsection{Stock Returns Conditioned on Volatility Index}
The VIX Index%, often referred to as the 'fear gauge,' 
is a time-series that measures market expectations of near-term volatility conveyed by S\&P 500 stock index option prices. A high VIX index typically signals a period of financial stress, as observed during major economic disruptions such as the Global Financial Crisis (GFC) in 2008 and the COVID-19 pandemic in 2020, when the VIX reached elevated levels. In this study, we apply our methodology to real-world data to model the returns of selected financial assets during periods of heightened market volatility. Our objective is to evaluate the proposed method by modeling the returns of financial assets conditioned on a measure of market risk. Specifically, we assess the performance of our approach in generating the marginal returns of a mix of technology and financial stocks during stressed market regimes, using the volatility index VIX as a conditioning factor. The stocks analyzed include AAPL, MSFT, GOOGL, NVDA, AMZN, JPM, WFC, and GS. We focus on two significant periods: the 2008 Global Financial Crisis and the 2020 COVID-19 pandemic. For each period, we establish distinct training and testing phases to evaluate generative performance::
\begin{itemize}
    \item {\bf GFC}: we use training data from 01/01/2005-12/31/2007 and evaluate on the testing data from 01/01/2008-12/31/2009. 
    \item {\bf COVID}: we use training data from 01/01/2017-12/31/2019 and evaluate on the testing data from 01/01/2020-12/31/2021. 
\end{itemize}
For baselines, we compare a standard linear diffusion (Gaussian base) and our proposed methodology with CEVT-based transformation and a Laplace base distribution. We provide more information on the VIX and plots of it during both periods for both the training and test data in the Appendix \ref{app: additional_details_stocks}, which demonstrate the prevalence of more extreme conditions in the testing dataset for both periods. 

\begin{figure}[htbp]
    \centering
    \begin{subfigure}[b]{0.24\textwidth}
        \includegraphics[width=\textwidth, trim=0cm 13cm 28cm 0cm, clip]{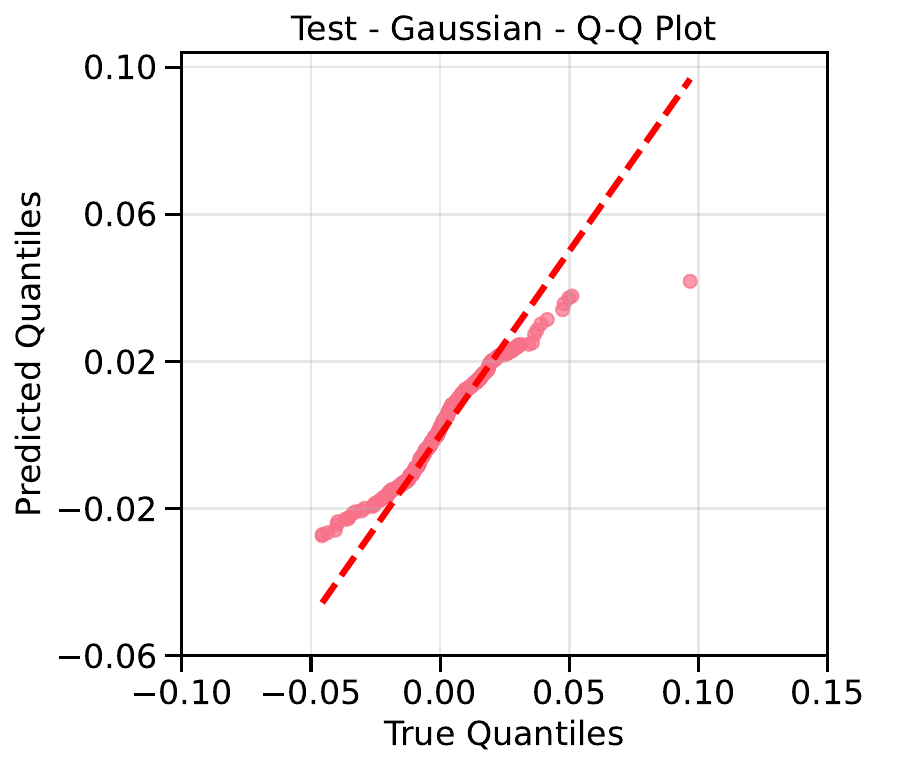}
        \caption{AAPL - Gaussian}
    \end{subfigure}
    \hfill
    \begin{subfigure}[b]{0.24\textwidth}
        \includegraphics[width=\textwidth, trim=0cm 13cm 28cm 0cm, clip]{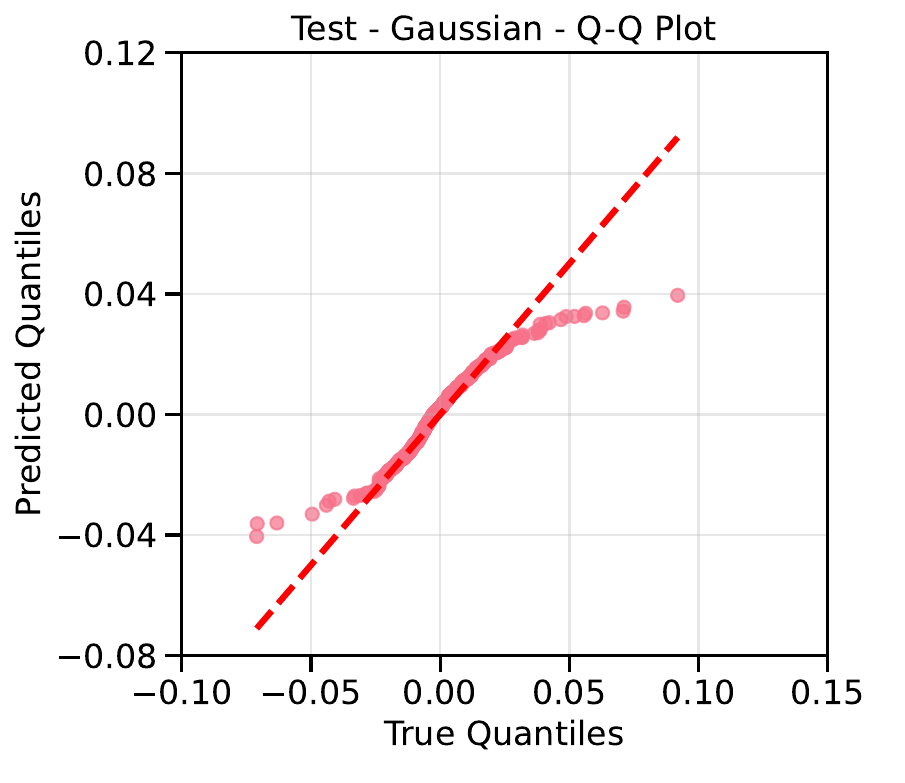}
        \caption{AMZN - Gaussian}
    \end{subfigure}
    \hfill
    \begin{subfigure}[b]{0.24\textwidth}
        \includegraphics[width=\textwidth, trim=0cm 13cm 28cm 0cm, clip]{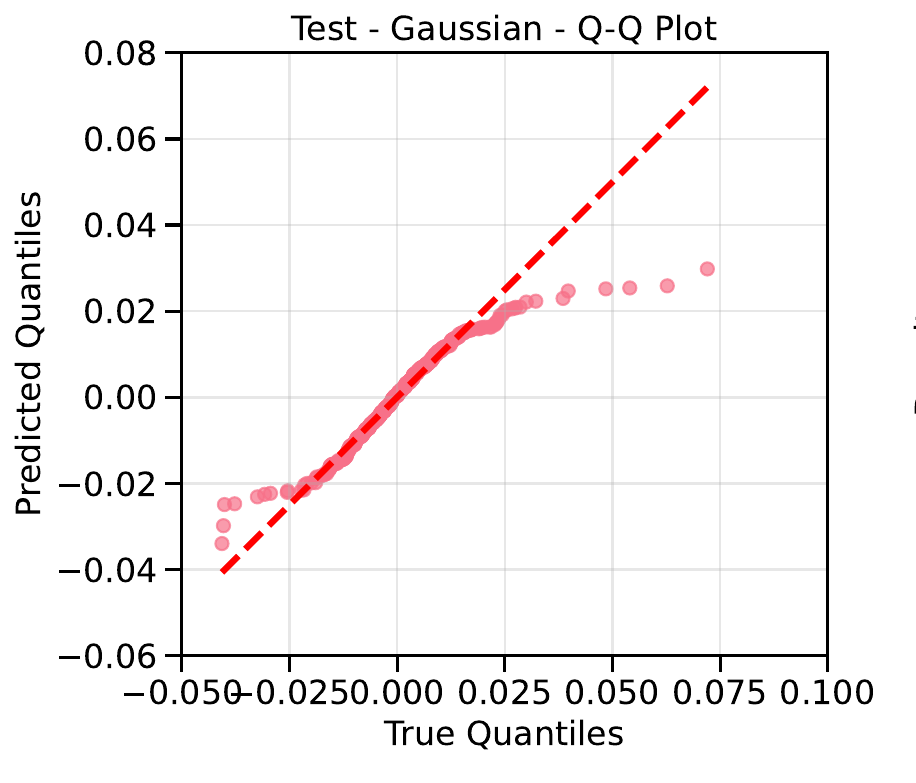}
        \caption{GOOGL - Gaussian.}
    \end{subfigure}
    \hfill
    \begin{subfigure}[b]{0.24\textwidth}
        \includegraphics[width=\textwidth, trim=0cm 13cm 28cm 0cm, clip]{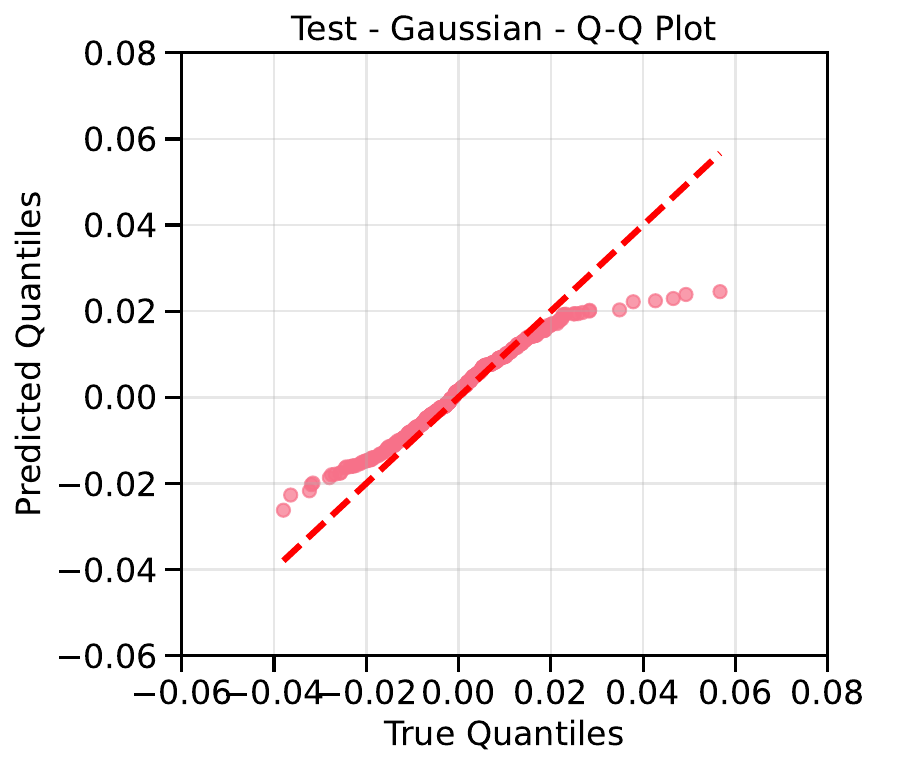}
        \caption{MSFT - Gaussian}
    \end{subfigure}
    
    \vspace{0.5em}
    
    \begin{subfigure}[b]{0.24\textwidth}
        \includegraphics[width=\textwidth, trim=0cm 13cm 28cm 0cm, clip]{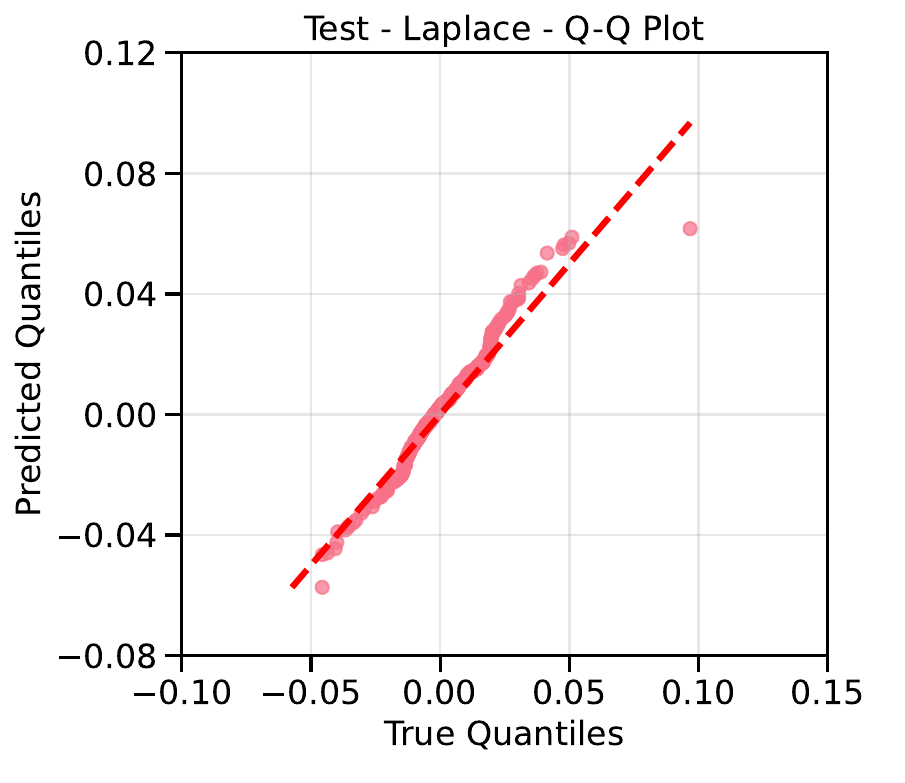}
        \caption{AAPL - Laplace}
    \end{subfigure}
    \hfill
    \begin{subfigure}[b]{0.24\textwidth}
        \includegraphics[width=\textwidth, trim=0cm 13cm 28cm 0cm, clip]{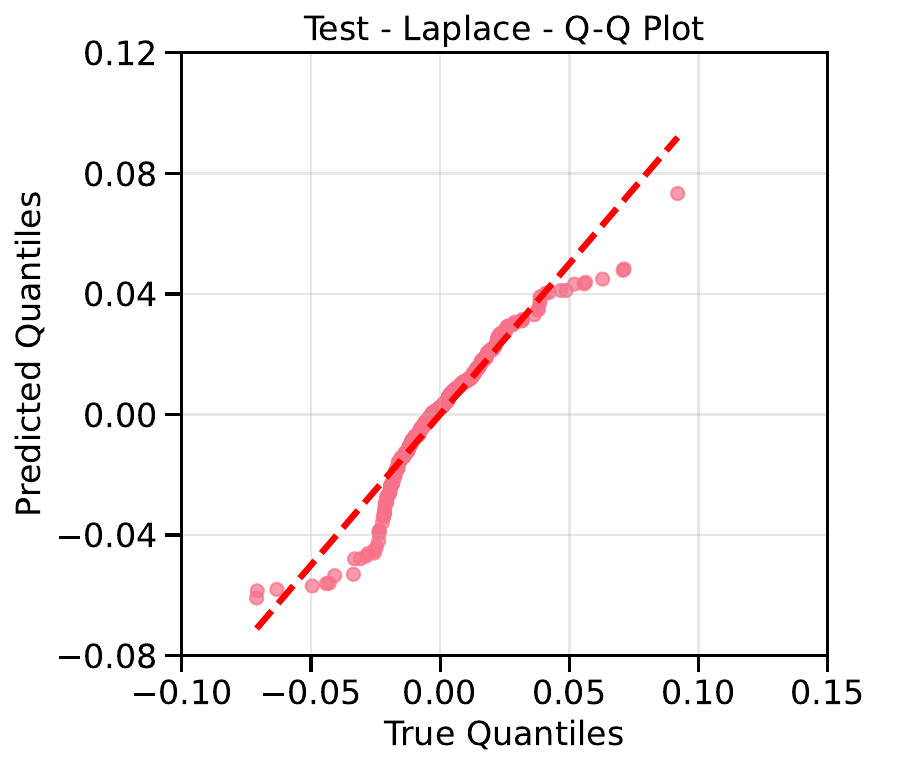}
        \caption{AMZN - Laplace}
    \end{subfigure}
    \hfill
    \begin{subfigure}[b]{0.24\textwidth}
        \includegraphics[width=\textwidth, trim=0cm 13cm 28cm 0cm, clip]{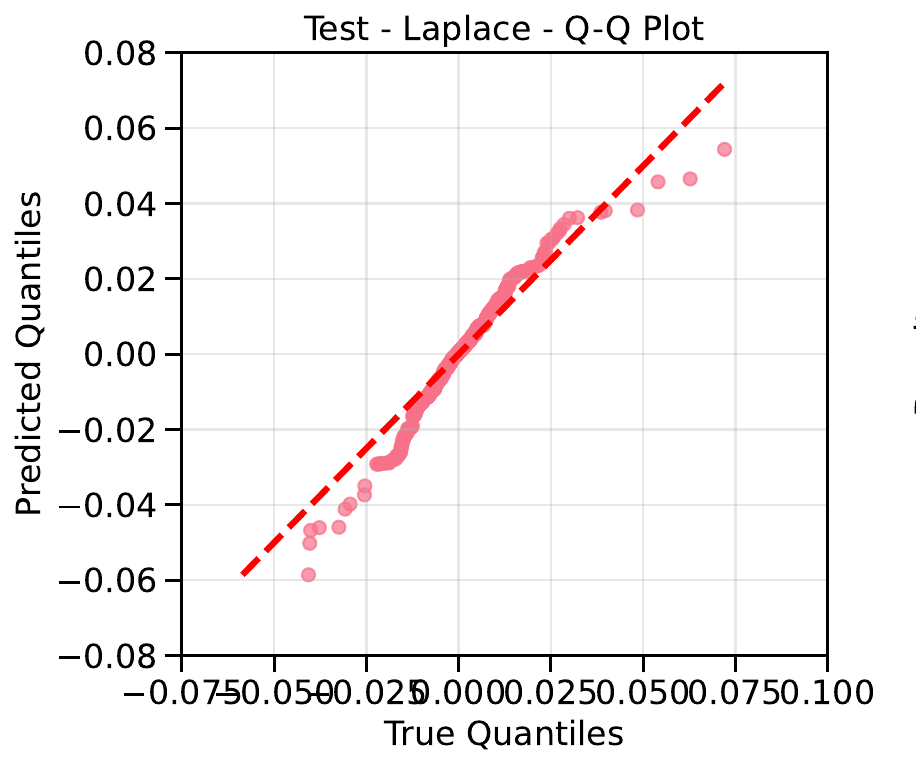}
        \caption{GOOGL - Laplace}
    \end{subfigure}
    \hfill
    \begin{subfigure}[b]{0.24\textwidth}
        \includegraphics[width=\textwidth, trim=0cm 13cm 28cm 0cm, clip]{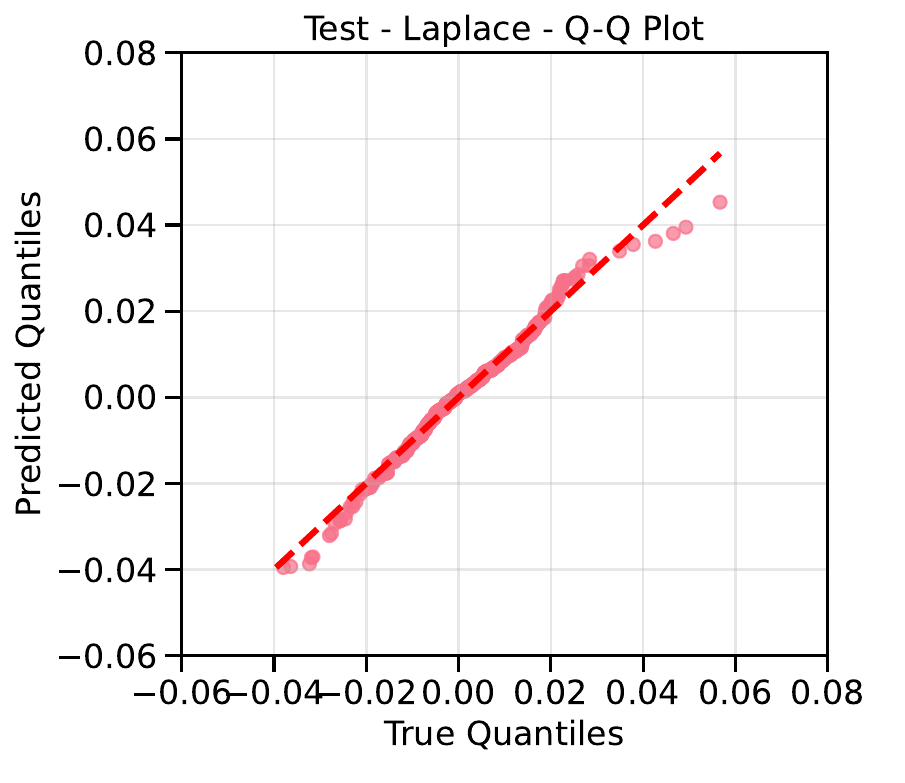}
        \caption{MSFT - Laplace}
    \end{subfigure} 
    \caption{QQ plots on test datasets for COVID period for various technology stocks.}
    \label{fig: qq_plots_covid_period}
\end{figure}

\begin{figure}[htbp]
    \centering
    \begin{subfigure}[b]{0.48\textwidth}
        \includegraphics[width=\linewidth]{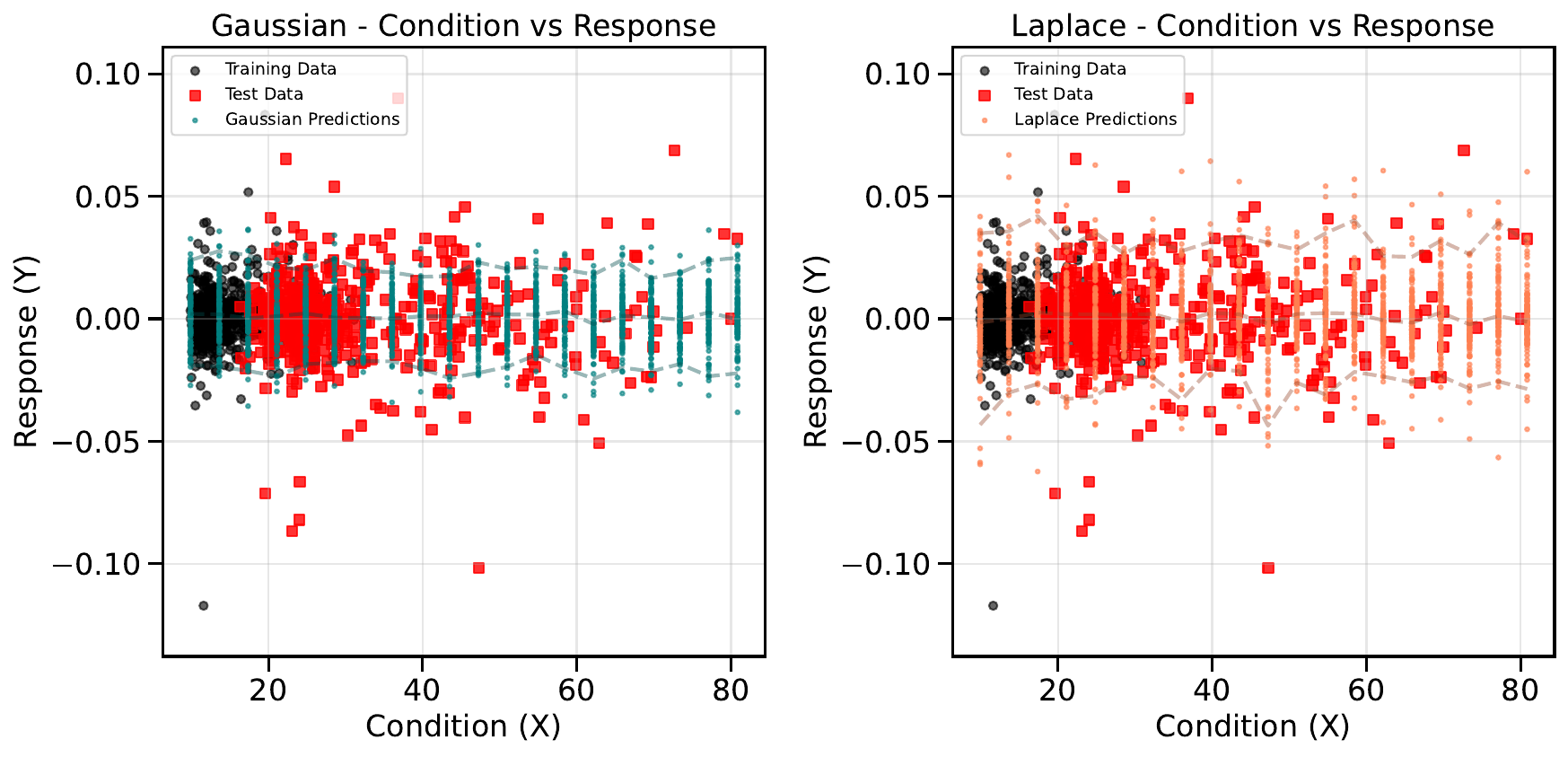}
        \caption{MSFT}
    \end{subfigure}
    \hfill
    \begin{subfigure}[b]{0.48\textwidth}
        \includegraphics[width=\linewidth]{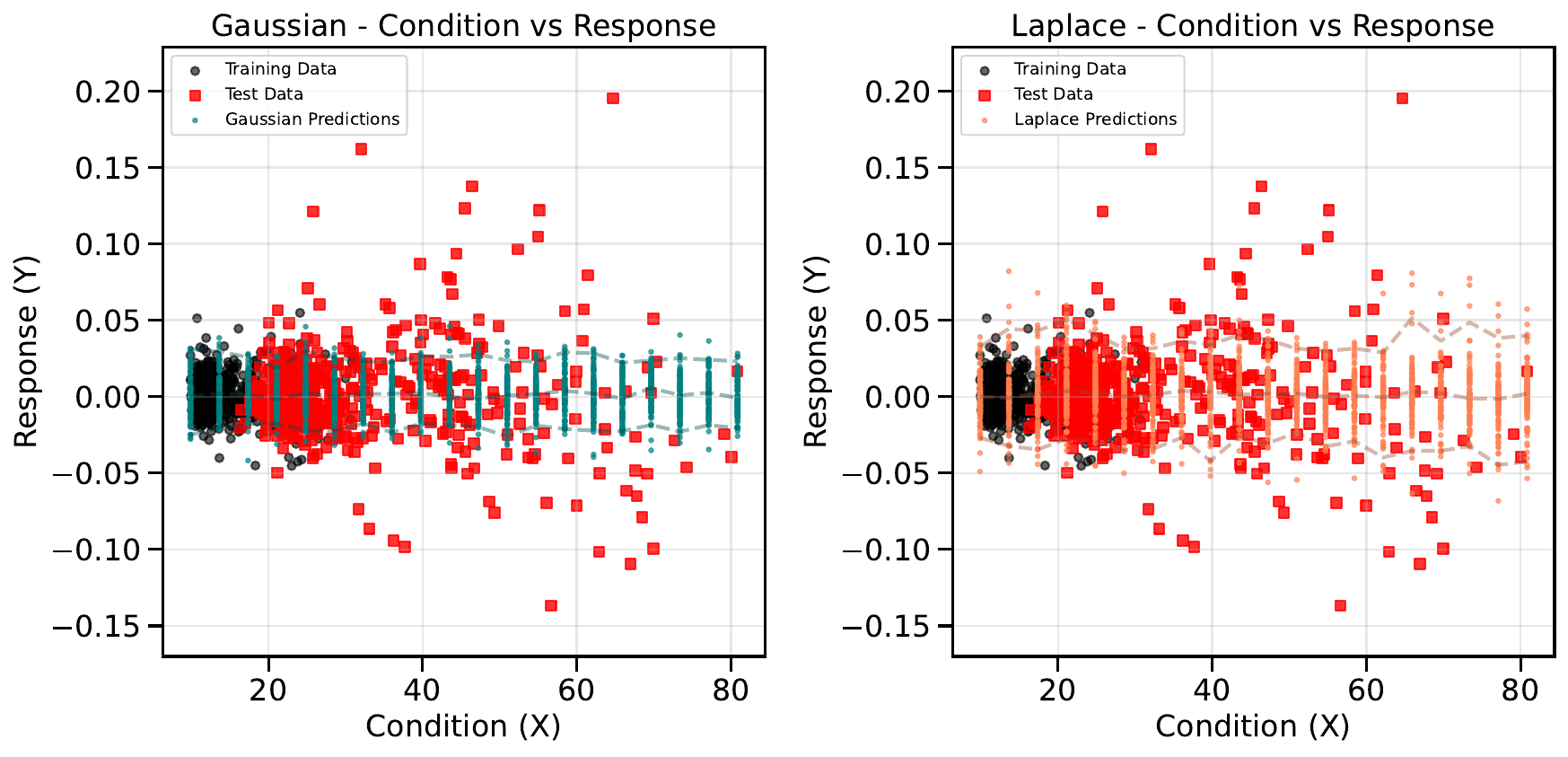}
        \caption{GS}
    \end{subfigure}
    \caption{Performance comparison of Gaussian versus Laplace base distributions based on different values of VIX level for the GFC regime. }
    \label{fig: conditional_performance_MSFT_gauss_vs_laplace}
\end{figure}

Our results demonstrate that, in this example, employing a nonlinear diffusion model offers a clear advantage in capturing the unconditional heavy-tailed behavior of stock returns, while also enhancing the modeling of conditionals for high VIX levels. For instance, as illustrated in the QQ plots in Figure \ref{fig: qq_plots_covid_period}, we observe that when capturing the marginal distribution of returns for various technology stocks during the COVID period, utilizing a Laplace base distribution outperforms its Gaussian counterpart in the tails, while maintaining good calibration in the bulk of the distribution. Regarding performance on the conditionals, we observe that during the GFC period, selecting a Laplace base distribution more effectively captures tail behavior as VIX values increase, despite these high VIX levels not being present during training. We offer more detailed plots analyzing the results for each stock across both periods in the Appendix \ref{app: additional_details_stocks}.

\section{Conclusions and Future Work}
In this work, we propose a methodology for improving rare event sampling in conditional generative modeling based on nonlinear score-based diffusion models. Motivated by conditional extreme value theory, we show that under some transformation of the data, we can choose the equilibrium distribution of the Langevin diffusion that is more advantageous from a sample complexity perspective for our learning problem. We provide numerical simulations on two toy experiments and a practical application of risk modeling for financial returns and demonstrate we can better capture the response distribution for extreme tails of the condition variable. From a practical perspective, challenges pertaining to our work include incorporating data-driven learning of the feature transformation process, extension to high-dimensional conditional variables, and a comprehensive performance comparison across multiple generative models on a larger pool of datasets. 

\begin{ack}
\textbf{Disclaimer:} This paper was prepared for informational purposes by the Artificial Intelligence Research group and the Quantitative Research Markets Capital Group of JPMorgan Chase \& Co. and its affiliates ("JP Morgan'') and is not a product of the Research Department of JP Morgan. JP Morgan makes no representation and warranty whatsoever and disclaims all liability, for the completeness, accuracy or reliability of the information contained herein. This document is not intended as investment research or investment advice, or a recommendation, offer or solicitation for the purchase or sale of any security, financial instrument, financial product or service, or to be used in any way for evaluating the merits of participating in any transaction, and shall not constitute a solicitation under any jurisdiction or to any person, if such solicitation under such jurisdiction or to such person would be unlawful.
\end{ack}

\bibliographystyle{plain}
\bibliography{reference.bib}

\begin{thebibliography}{10}

\bibitem{ANDERSON1982313}
Brian~D.O. Anderson.
\newblock Reverse-time diffusion equation models.
\newblock {\em Stochastic Processes and their Applications}, 12(3):313--326, 1982.

\bibitem{chewi2023log}
Sinho Chewi.
\newblock Log-concave sampling.
\newblock {\em Book draft available at https://chewisinho. github. io}, 9:17--18, 2023.

\bibitem{dalalyan2017further}
Arnak Dalalyan.
\newblock Further and stronger analogy between sampling and optimization: Langevin monte carlo and gradient descent.
\newblock In {\em Conference on Learning Theory}, pages 678--689. PMLR, 2017.

\bibitem{dalalyan2012sparse}
Arnak~S Dalalyan and Alexandre~B Tsybakov.
\newblock Sparse regression learning by aggregation and langevin monte-carlo.
\newblock {\em Journal of Computer and System Sciences}, 78(5):1423--1443, 2012.

\bibitem{drees2017conditional}
Holger Drees and Anja Jan$\ss$en.
\newblock Conditional extreme value models: fallacies and pitfalls.
\newblock {\em Extremes}, 20(4):777--805, 2017.

\bibitem{durmus2017nonasymptotic}
Alain Durmus and Eric Moulines.
\newblock Nonasymptotic convergence analysis for the unadjusted langevin algorithm.
\newblock 2017.

\bibitem{farrell2021deep}
Max~H Farrell, Tengyuan Liang, and Sanjog Misra.
\newblock Deep neural networks for estimation and inference.
\newblock {\em Econometrica}, 89(1):181--213, 2021.

\bibitem{heffernan2007limit}
Janet~E Heffernan and Sidney~I Resnick.
\newblock Limit laws for random vectors with an extreme component.
\newblock 2007.

\bibitem{hefftawn2004}
Janet~E. Heffernan and Jonathan~A. Tawn.
\newblock A conditional approach for multivariate extreme values (with discussion).
\newblock {\em Journal of the Royal Statistical Society Series B: Statistical Methodology}, 66(3):497--546, 07 2004.

\bibitem{ho2020denoising}
Jonathan Ho, Ajay Jain, and Pieter Abbeel.
\newblock Denoising diffusion probabilistic models.
\newblock {\em Advances in neural information processing systems}, 33:6840--6851, 2020.

\bibitem{hu2023complexity}
Tianyang Hu, Fei Chen, Haonan Wang, Jiawei Li, Wenjia Wang, Jiacheng Sun, and Zhenguo Li.
\newblock Complexity matters: Rethinking the latent space for generative modeling.
\newblock {\em Advances in Neural Information Processing Systems}, 36:29558--29579, 2023.

\bibitem{joe1997multivariate}
H.~Joe.
\newblock {\em Multivariate Models and Multivariate Dependence Concepts}.
\newblock Chapman \& Hall/CRC Monographs on Statistics \& Applied Probability. Taylor \& Francis, 1997.

\bibitem{KEEF2013396}
Caroline Keef, Ioannis Papastathopoulos, and Jonathan~A. Tawn.
\newblock Estimation of the conditional distribution of a multivariate variable given that one of its components is large: Additional constraints for the heffernan and tawn model.
\newblock {\em Journal of Multivariate Analysis}, 115:396--404, 2013.

\bibitem{liang2024denoising}
Tengyuan Liang, Kulunu Dharmakeerthi, and Takuya Koriyama.
\newblock Denoising diffusions with optimal transport: Localization, curvature, and multi-scale complexity.
\newblock {\em arXiv preprint arXiv:2411.01629}, 2024.

\bibitem{nelsen2006introduction}
Roger~B Nelsen.
\newblock {\em An introduction to copulas}.
\newblock Springer, 2006.

\bibitem{oksendal2013stochastic}
Bernt Oksendal.
\newblock {\em Stochastic differential equations: an introduction with applications}.
\newblock Springer Science \& Business Media, 2013.

\bibitem{resnick2014transition}
Sidney~I Resnick and David Zeber.
\newblock Transition kernels and the conditional extreme value model.
\newblock {\em Extremes}, 17(2):263--287, 2014.

\bibitem{singhal2024s}
Raghav Singhal, Mark Goldstein, and Rajesh Ranganath.
\newblock What's the score? automated denoising score matching for nonlinear diffusions.
\newblock {\em arXiv preprint arXiv:2407.07998}, 2024.

\bibitem{song2021maximum}
Yang Song, Conor Durkan, Iain Murray, and Stefano Ermon.
\newblock Maximum likelihood training of score-based diffusion models.
\newblock {\em Advances in neural information processing systems}, 34:1415--1428, 2021.

\end{thebibliography}

%%%%%%%%%%%%%%%%%%%%%%%%%%%%%%%%%%%%%%%%%%%%%%%%%%%%%%%%%%%%
\appendix
\section{CEVT details}
\label{app: cevt_details}
We restate the CEVT modeling assumption for convenience.
\begin{assumption}[CEVT \cite{hefftawn2004, KEEF2013396}] Suppose the marginals of $X$ and $Y$ are standard Laplace. Then,  as $X= x \to \infty$, we assume $X, Y$ admit the \textbf{asymptotic} dependency,

\begin{align*}
    \lim_{x \to \infty} P\bigg(\tfrac{Y - a(X)}{b(X)} < z |X = x\bigg) = G(z)
\end{align*}
where $G$ is some distribution independent of $X$. In other words, for tail values, $X = x \to \infty$, we model, 
    \begin{align*}
    Y = a(X) + b(X) \cdot Z, \quad Z \sim G, 
\end{align*}
\end{assumption}
In a slightly different formulation, \cite{heffernan2007limit} establish that, so long as the conditioning variable $X$ belongs to the domain of attraction of an extreme value distribution, such an assumption about asymptotic behavior is reasonable. More recently, \cite{resnick2014transition} directly related the Heffernan Tawn model to the more general formulation of \cite{heffernan2007limit} and found parsimony under some mild conditions. We emphasize that this modeling assumption is theoretically grounded. A growing body of applied statistical methods successfully apply this model, further strengthening its relevance in practice. 
  
  Importantly, \cite{hefftawn2004, KEEF2013396} found that for all standard copula forms of dependence outlined in \cite{joe1997multivariate, nelsen2006introduction}, the functions $a(X), b(X)$ admit simple parametric forms, thus, the limiting form $G$ can be assessed with a relatively small amount of samples. This insight motivates an approach to extrapolating to the tail in conditional score-based diffusion models.

\subsection{Normalizing Functions}
\label{app: normalizing_fuctions}
 For a variety of relationships between $X$ and $Y$, $G$ has a log-concave density and the normalizing functions $a$ and $b$ admit simple forms \cite{hefftawn2004, KEEF2013396}. 

 Suppose $X$ and $Y$ are marginally Laplace. Then for some suitably high threshold, $x \in \R$ the conditional relationship at the tail values,  $X > x$, approximately satisfy, 
\begin{align*}
    Y = a\cdot X + {X}^b \cdot Z, \quad Z \sim G, \quad a \in [-1, 1], \ b \in (-\infty, 1).
\end{align*}
For a detailed examination of this relationship and clear delineation of when this simple form arises a reader should refer to the original work \cite{hefftawn2004} or the follow-up \cite{KEEF2013396}. In particular, Table 1 in \cite{hefftawn2004}. For failure cases a reader can refer to \cite{drees2017conditional}. We assume for our examples that $a(x)$ and $b(x)$ admit this simple structure. 

In practice, the scalars $a$ and $b$ are estimated. It is possible to learn these parameters via constrained optimization. The simplest approach , which we implemented, is to assume $Z \sim \cN(0, 1)$ and implement maximum likelihood with tail data $\{(X_i,Y_i): X_i>x \}$.

\subsection{Toy Example}
As an example, suppose, 
\begin{align*}
    (X, Y) \sim \cN\left(\begin{bmatrix}
        0 \\ 0
    \end{bmatrix}, \begin{bmatrix}
        1 & \rho\\
        \rho & 1
    \end{bmatrix}\right)
\end{align*}
First, transform the variables to have Laplace marginals, $(X,Y) \to (X^\star, Y^\star)$ (e.g., Inverse CDF Transform). For this example, the normalizing functions admit an explicit form, 
\begin{align*}
    Z = \frac{Y^\star - a(X^\star)}{b(X^\star)}, \quad a(x) = \text{sign}(\rho)\cdot \rho^2 \cdot x, \ b(x) = x^{1/2}
\end{align*}
In this regime, it is well understood \cite{KEEF2013396} that, 
\begin{align*}
    \bP(Z|X^\star = x^\star) \to \cN(0, 2\rho^2(1-\rho^2)), \quad \text{as } x^\star \to \infty.
\end{align*}
So, setting $g(x)= \tfrac{1}{2}x^2, \ \beta = (2\rho^2(1-\rho^2))^{-1/2} $, our new forward diffusion is a scaled OU process that admits $ G = \cN(0, 2\rho^2(1-\rho^2))$ as equilibrium. 

We visualize the diffusion process, before and after transformation, in Figure \ref{fig:cevtappendix}. Comparing the plots in the left column,it is clear that the path evolution of particles $Y_t$ that correspond to large, tail values in $X$ (bottom right, depicted in blue) are much more regular after the transformation. We also plot the conditional densities $\bar{\mu}_t(.|x)$,  for a collection of timesteps and both bulk and tail events $\{X=x\}$. Before the transformation, $\bar{\mu}_t(y|x)$ changes quite drastically across the forward chain. However, after transformation (see bottom right figure), $\bar{\mu}_t(z|x^\star) \propto G$ for tail values $X^\star = x^\star$. Indeed, we see that at the tail values of the condition, $x^\star \to \infty$, the forward process is already at stationarity. In other words,  
\begin{align*}
    \nabla g(y) + \beta^{-1} \nabla \log \rho_{\mu_t(.|x^\star)}(y) = y - y = 0, \quad \forall t, \qquad \text{(easy to learn)}
\end{align*}
 And so, where we have few samples, we have a sequence of functions that may be estimated with few samples.

%for all timesteps, $t$. In contrast (Figure b), for bulk values of $x^\star$, the forward diffusion induces quite dramatic changes in the density of particles as time progresses.
 \begin{figure}[H]
     \centering
     \includegraphics[width=0.98\linewidth]{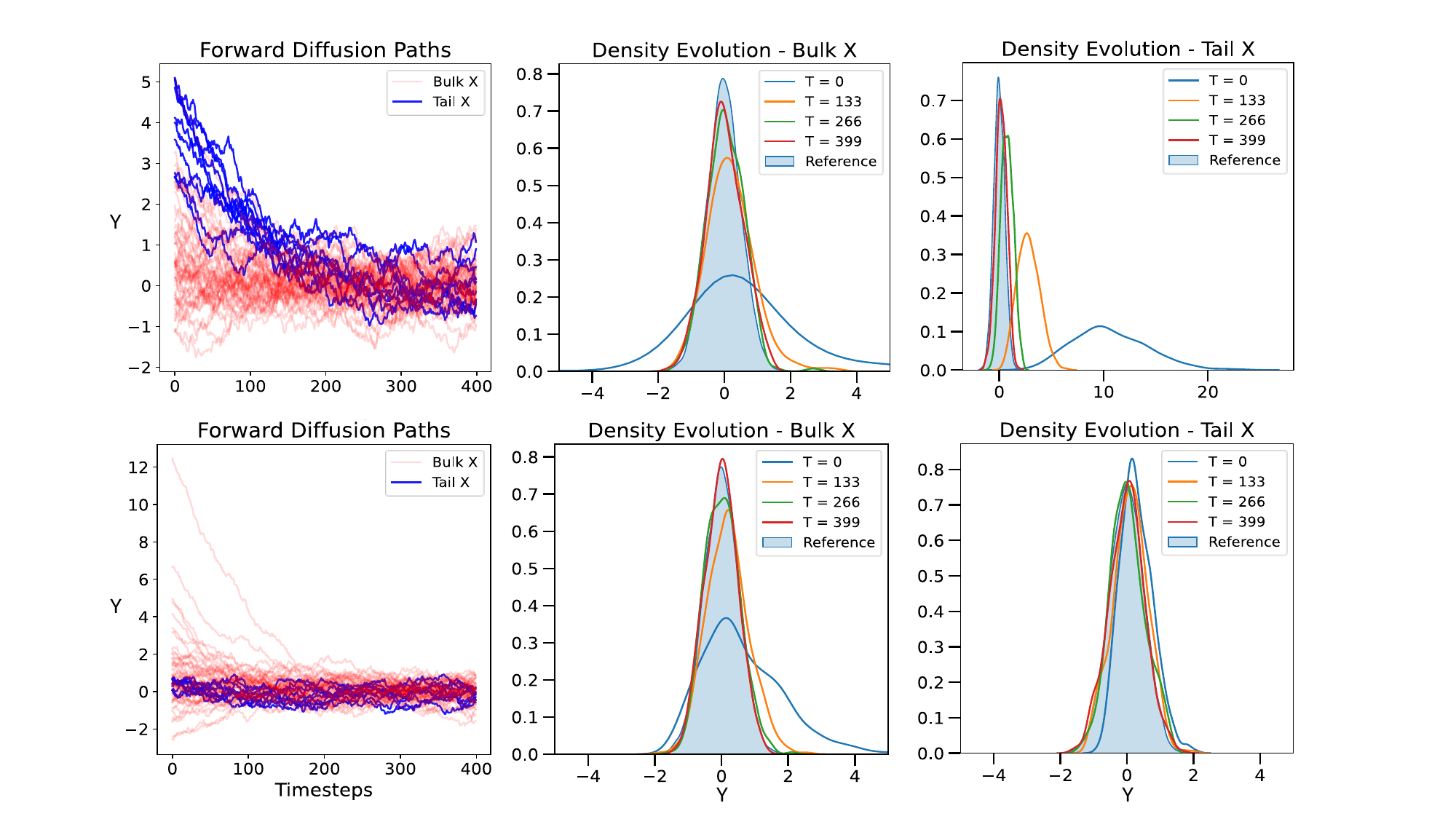}
     \caption{Top row: Before transformation. Bottom row: After transformation.}
     \label{fig:cevtappendix}
 \end{figure}
%This dynamic behavior for high-probability conditions is not so problematic. In this region, many samples are available, and estimation of complex denoising plans should be feasible.

\section{Theory}
\label{app: theory}
A simple change to Theorem 1 in \cite{song2021maximum} will reflect our change in target for estimation. For completeness we include the theorem below and detail the small modification to the proof. We state the result for unconditional densities, but the result follows for conditional densities without loss of generality. 

\begin{theorem} Denote by $p(y)$ the target density. Let $\{Y_t\}_{t \in [0,T]}$ be the stochastic process defined by the SDE in \ref{eqn:langevin}, where $Y_0 \sim p$ and $Y_t \sim p_t$. Suppose $\pi(y)$ is the stationary density of this SDE as $T \to \infty$. Let $\hat Y^\leftarrow_0 \sim p_{\theta}(y)$ be the result of the approximate reverse-time SDE where we substitute our score model, $s_\theta(y, t)$. 
\begin{align}
    \dd {\hat Y}^\leftarrow_t = -(2 s_\theta( \hat Y^\leftarrow_t, t) - \nabla f(\hat Y^\leftarrow_t)) \dd t + \sqrt{2\beta^{-1}}\dd \bar{B}_t, \quad \hat Y^\leftarrow_T \sim \pi
\end{align}

 Under some regularity conditions (see Appendix A \cite{song2021maximum}), 
\begin{align*}
    KL(p||p_\theta) \leq  \int_0^T \E_{p_{t}(y)}[\|\left(\nabla f(y) + \beta^{-1} \nabla \log p_{t}(y)\right) - s_\theta(y, t)\|^2]dt + KL(p_T||\pi).
\end{align*}
    
\end{theorem}
\begin{proof}
Denote the path measure of $\{ Y_t\}_{t\in [0,T]}$ and  $\{ \hat Y^\leftarrow_t\}_{t \in [0,T]}$ by $\mu$ and $\nu$. Recall $Y_0 \sim p$ and $Y_T \sim p_T$, whereas $\hat Y^\leftarrow_0 \sim p_\theta$ and $\hat Y^\leftarrow_T \sim \pi$. Following the line of argumentation in \cite{song2021maximum}, we establish by data-processing inequality (1), and chain rule (2),
\begin{align*}
    KL(p||p_\theta) \stackrel{1}{\leq} KL(\mu||\nu) \stackrel{2}{\leq} KL(p_T||\pi) + \E_{z \sim p_T}[KL(\mu(\cdot|Y_t =z) \|| \nu(\cdot|\hat Y^\leftarrow_T = z))].
\end{align*}
What remains is to tackle the second term on RHS. Due to time-reversal, the path measure $\{ Y_t\}_{t\in [0,T]}$ can be equivalently seen as generated by the reverse time SDE, 
   \begin{align}
    \dd { Y}^\leftarrow_t = -(\nabla f( Y^\leftarrow_t) + 2 \beta^{-1} \nabla \log p_t(Y^\leftarrow_t)) \dd t + \sqrt{2\beta^{-1}}\dd \bar{B}_t, \quad  Y^\leftarrow_T \sim \pi
\end{align}
Then, $KL(\mu(\cdot|Y_t =z) \|| \nu(\cdot|\hat Y^\leftarrow_T = z))$ can be calculated by comparing the following reverse-time SDEs initialized at the same point:
   \begin{align}
       &dY^\leftarrow_t = -(\nabla f(Y^\leftarrow_t) + 2\beta^{-1} \nabla \log p_{\mu_t(\cdot|x)}(Y^\leftarrow_t))dt + \sqrt{2\beta^{-1}}d B_t, \quad Y^\leftarrow_T = z\\
       &d\hat Y^\leftarrow_t = -(2s_\theta(\hat Y^\leftarrow_t; x, t) - \nabla f(\hat Y^\leftarrow_t)) dt + \sqrt{2\beta^{-1}}d B_t, \quad \hat Y^\leftarrow_T = z
   \end{align}
   Since these SDES share the same diffusion coefficient and starting point, we can appeal to Girsanov's theorem \cite{oksendal2013stochastic} to see, 
   \begin{align*}
       KL(\mu(\cdot|Y_t =z) \|| \nu(\cdot|\hat Y^\leftarrow_T = z)) \leq \int_0^T \E_{p_{t}(y)}[\|\left(\nabla f(y) + \beta^{-1}\nabla \log p_{t}(y)\right) - s_\theta(y, t)\|^2]dt
   \end{align*}
\end{proof}

We adopt the following non-asymptotic bound from \cite{farrell2021deep} with regard to the sample complexity of minimizing the squared error in a multi-layer perceptron neural network. 

\begin{theorem} 
\label{thm: smoothness_nn_rate}
Let $\widehat{f}_{\rm MLP}$ denote a standard multi-layer perceptron. Under the assumption that the target function $f_\star=\nabla f + \beta^{-1} \nabla \log p_{\mu_t(\cdot|x)}$ lies in the Sobolev ball ${\cal W}^{S,\infty}([-1, 1]^d)$ with smoothness parameter $S\in\mathbb{N}^+$, then with probability at least $1-\delta$ where $\delta=\exp(-n^{\frac{d}{S+d}}\log^8n)$, for large enough $n$:
    \begin{equation}
        \|\widehat{f}_{\rm MLP}-f_\star\|_{L_2(x)}^2\leq C\left(n^{-\frac{S}{S+d}}\log^8n + \frac{\log\log n}{n}\right)
    \end{equation}
    Intuitively, the "rougher" the function (the smaller the value of $S$) and the higher the input dimension $d$, a larger number of samples are needed to estimate the target function $f_\star$.
\end{theorem}

\section{Methodology Details}

\subsection{Smoothness of $f$} 
\label{app: smoothness_f}

We implement the Euler-Maryama discretization of the forward diffusion, \ref{eqn:langevin}. This amounts to Unadjusted Langvevin Algorithm (ULA). It is well established that the convergence speed of ULA depends on the gradient of our drift term, $\nabla^2 f$ (developed in a sequence of works \cite{dalalyan2012sparse, dalalyan2017further, durmus2017nonasymptotic}). We present a result condensed in \cite{chewi2023log}, and for simplicity, specialized to dimension, $d = 1$.
\begin{theorem}[Convergence of ULA \cite{chewi2023log}] Suppose that $\pi \propto e^{-f}$ is the target distribution and $f$ satisfies $\alpha \leq \nabla^2 f \leq \beta$. Define $\kappa = \beta/\alpha$ as the condition number and $\mu_{t\cdot\eta}$ as the $t-$th measure in the sequence. Then, for any $\epsilon \in [0, 1]$, with step size $\eta \asymp \epsilon^2/\beta \kappa$, we obtain that after, 
\begin{align*}
    T = O\left(\frac{\kappa^2}{\epsilon^2}\log \frac{\alpha W_2^2(\mu_0, \pi)}{\epsilon^2}\right) \quad \text{iterations,}
\end{align*}
\begin{align*}
   \alpha W_2^2(\mu_{T\cdot \eta}, \pi)\leq \epsilon^2
\end{align*}
\end{theorem}
In our methodology, we propose choosing a convex $f$ to target a specific distribution, $e^{-f}$, that reflects the tail characteristics of our target conditional distribution, $P(Y|X=x)$. However, choosing $f$ with poor curvature directly impacts speed of forward process. This in turn impacts how many noising steps, $[T]$, are necessary to diffuse-then-denoise and can detriment computational efficiency. This is particularly relevant when part of our argument concerns out-performing Gaussian diffusions. However, when $e^{-f} \propto e^{-x^2/2}$, $\kappa = 1$ and convergence is fast. 

To overcome this we use appropriately smoothed versions of the new target density, $e^{-f^\star}$. We smooth in such a way that $\kappa$ is bounded, $f^\star$ is continuously differentiable, but $e^{-f^\star} \approx e^{-f}$. In the backward process, we still initiate samples by drawing from $e^{-f}$. We emphasize that this does not impact the quality of the method. 
\begin{itemize}
    \item We show below that by appropriately choosing smoothing parameters, the forward process converges to a distribution very similar to the target, $e^{-f}$.
    \item Small perturbations between the end of the forward process ($e^{-f^\star}$) and start of the backward process ($e^{-f}$) is theoretically negligible \cite{liang2024denoising}. Even with standard schemes, owing to the finite time steps $T < \infty$, the end of the forward proccess will not be exactly Gaussian. 
\end{itemize}

Below are examples relevant to this paper.

\paragraph{Gaussian}
The standard scheme is to target $e^{-f} \propto e^{-x^2/2}$, standard Gaussian density. In this case, $\kappa = 1$. 
\begin{figure}[H]
    \centering
    \begin{subfigure}[t]{0.32\textwidth}
        \centering
        \includegraphics[width=0.98\linewidth]{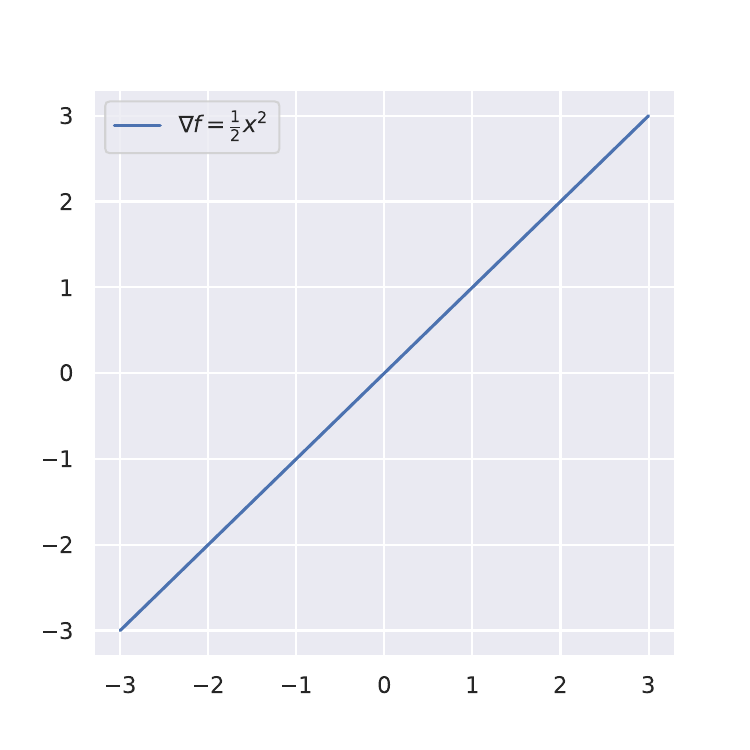}
        \caption{}
    \end{subfigure}
    \begin{subfigure}[t]{0.32\textwidth}
        \centering
        \includegraphics[width=0.98\linewidth]{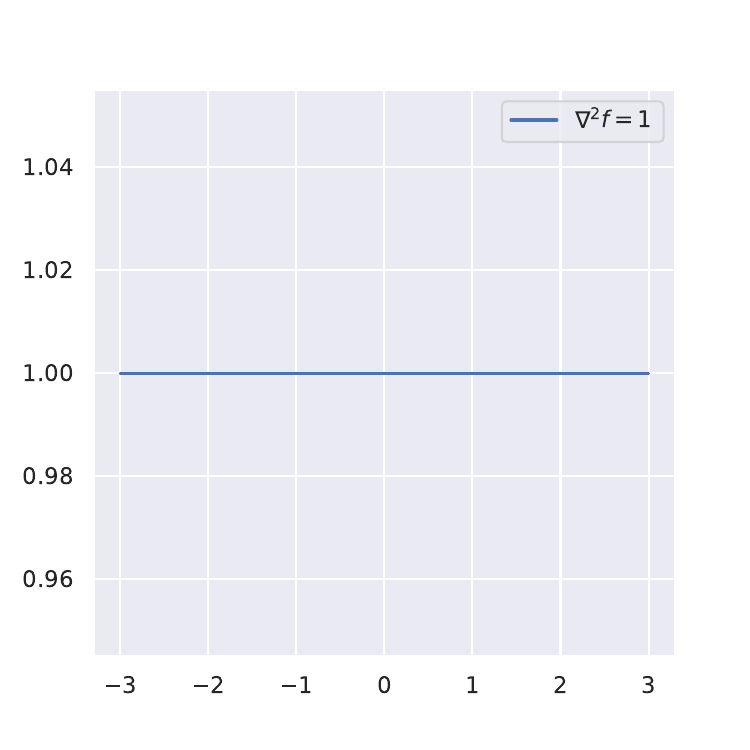}
        \caption{}
    \end{subfigure}
    % \begin{subfigure}[t]{0.32\textwidth}
    %     \centering
    %     \includegraphics[width=0.98\linewidth]{figures/gaussiangradf.pdf}
    %     \caption{}
    % \end{subfigure}
    \caption{(a): Plot of $\nabla f$. (b): Plot of $\nabla^2 f$.}
    \label{fig:placeholder}
\end{figure}

\paragraph{Laplace}
Suppose we want to target $e^{-f} \propto e^{-|x|}$. Then, $\nabla^2 f = 0$ and $f$ is not continuously differentiable (at 0). Convergence theorem for ULA suggests potential problems. Instead, we consider a smooth approximation, 
\begin{align*}
    \nabla f^\star(x, b , c) = \begin{cases}
        \tfrac{1}{b}\cdot x + c\cdot x, \quad \text{if } x \in (-b, b),\\
        \text{sign}(x) + c\cdot x, \quad \text{otherwise.}
    \end{cases}
\end{align*}
Here, $b, \ c \geq 0$ are user specified constants. If $b, c = 0$, then we arrive at $\nabla f$. 
This is simply the gradient of the Huber function with a linear perturbation by $c\cdot x$. With this smoothing, $\kappa =1 + \frac{1}{bc} $.
\begin{figure}[H]
    \centering
    \begin{subfigure}[t]{0.32\textwidth}
        \centering
        \includegraphics[width=0.98\linewidth]{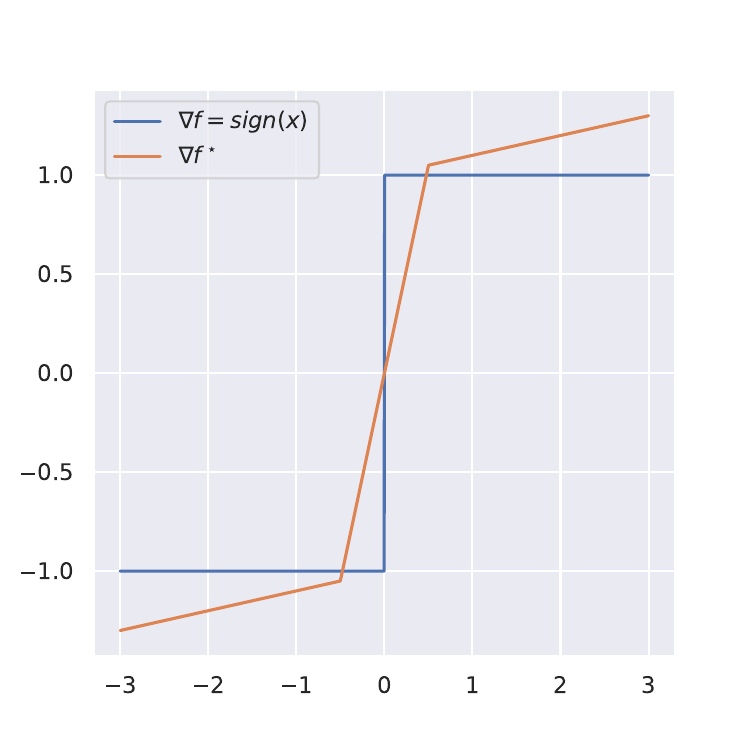}
        \caption{}
    \end{subfigure}
    \begin{subfigure}[t]{0.32\textwidth}
        \centering
        \includegraphics[width=0.98\linewidth]{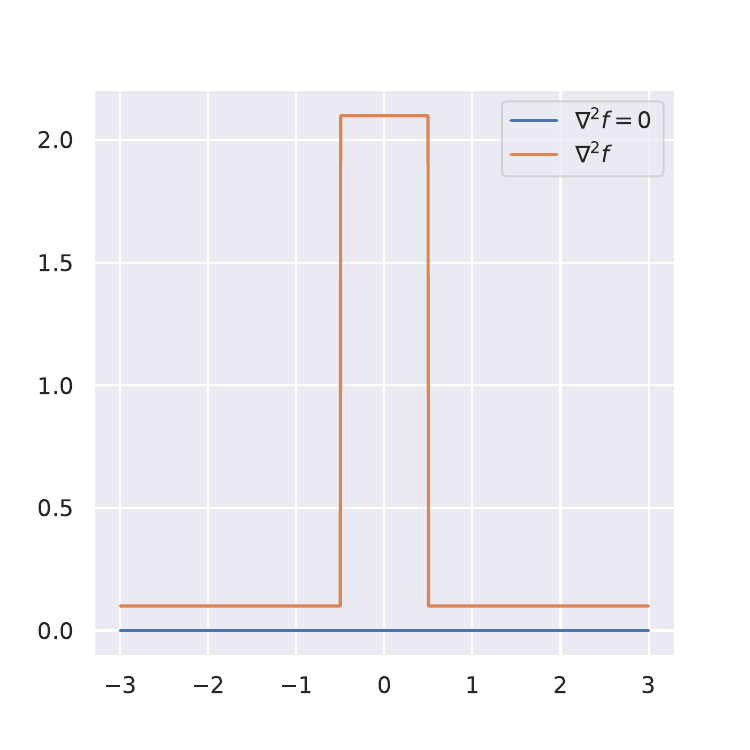}
        \caption{}
    \end{subfigure}
    \begin{subfigure}[t]{0.32\textwidth}
        \centering
        \includegraphics[width=0.98\linewidth]{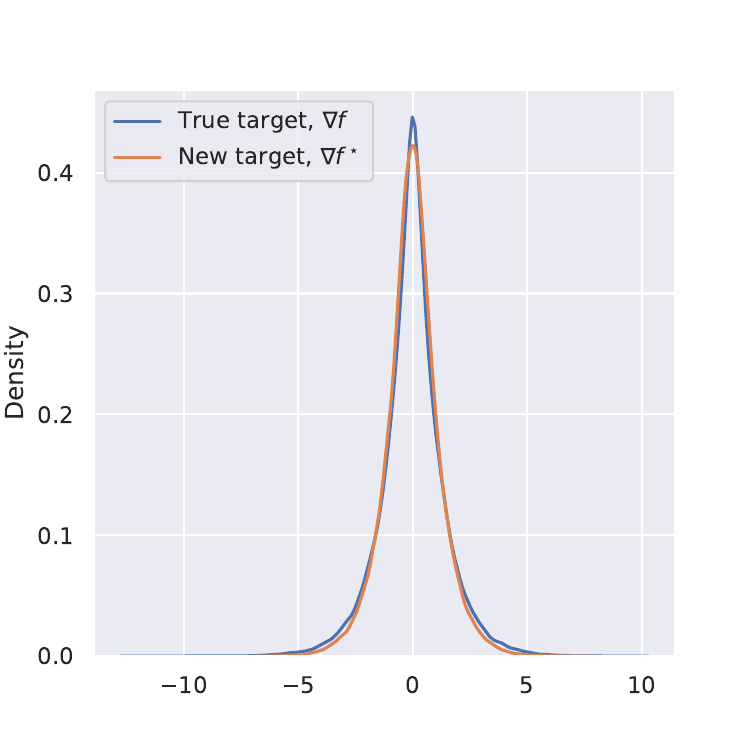}
        \caption{}
    \end{subfigure}
    \caption{Set $b = 0.5, \ c = 0.1$. (a): Plot comparing $\nabla f$ and $\nabla f^\star$. (b): Plot comparing $\nabla^2 f$ and $\nabla^2 f^\star$. (c): Comparing densities after running ULA ($\eta = 0.01, \ T = 1000$ with $\nabla f$ and $\nabla f^\star$.}
    \label{fig:placeholder}
\end{figure}

\paragraph{Huber}
Suppose we want to target $e^{-f} \propto e^{-(x+e^{-x})}$. Then, $\nabla^2 f = e^{-x}$ which is not bounded above, and approaches $\to 0$ as $x \to \infty$. We consider a smooth approximation, 
\begin{align*}
    \nabla f^\star(x, b , c) = \begin{cases}
        e^{b}, \quad \text{if } x \leq -b,\\
        e^{-x}, \quad \text{if } -b<x<c ,\\
        e^{-c}, \quad \text{if } x \geq c.
    \end{cases}
\end{align*}
Here, $b, \ c \geq 0$ are user specified constants. If $b, c = 0$, then we arrive at $\nabla f$. With this smoothing, $\kappa = e^{b+c} $.

\begin{figure}[H]
    \centering
    \begin{subfigure}[t]{0.32\textwidth}
        \centering
        \includegraphics[width=0.98\linewidth]{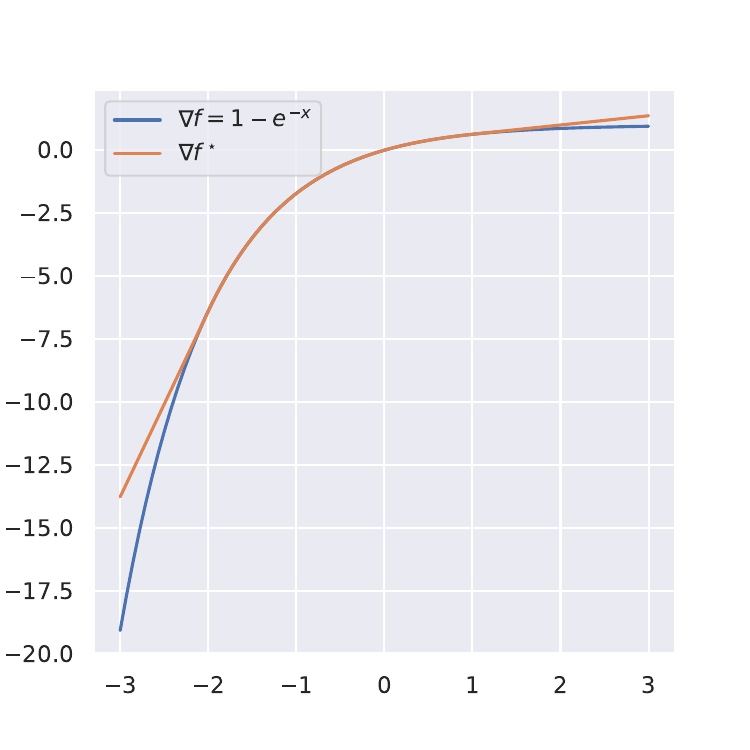}
        \caption{}
    \end{subfigure}
    \begin{subfigure}[t]{0.32\textwidth}
        \centering
        \includegraphics[width=0.98\linewidth]{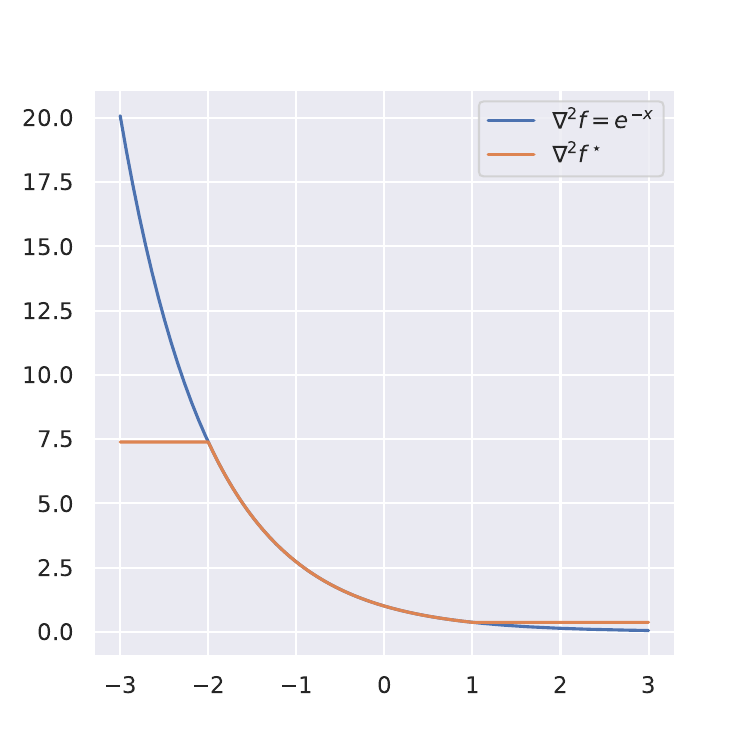}
        \caption{}
    \end{subfigure}
    \begin{subfigure}[t]{0.32\textwidth}
        \centering
        \includegraphics[width=0.98\linewidth]{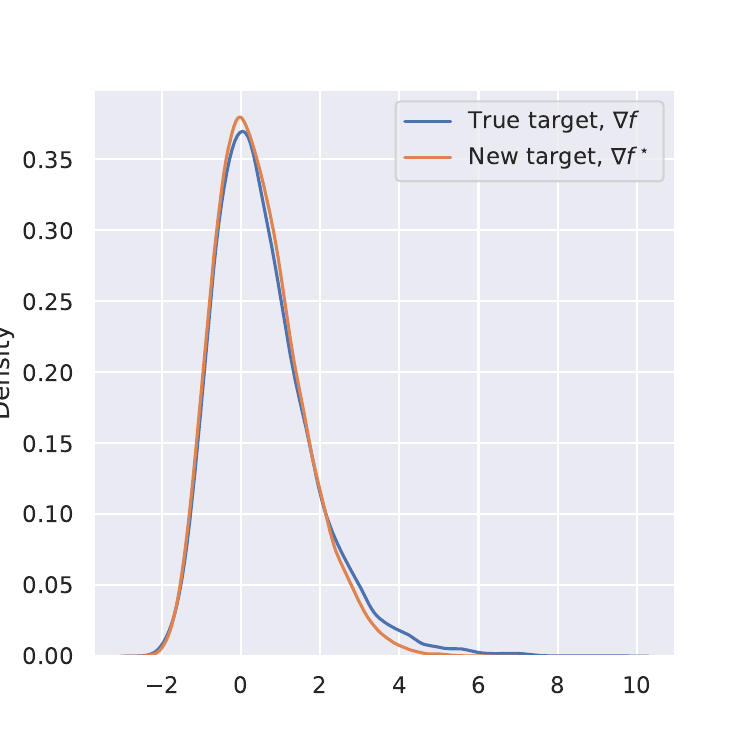}
        \caption{}
    \end{subfigure}
    \caption{Set $b = 2, \ c = 1$. (a): Plot comparing $\nabla f$ and $\nabla f^\star$. (b): Plot comparing $\nabla^2 f$ and $\nabla^2 f^\star$. (c): Comparing densities after running ULA ($\eta = 0.01, \ T = 1000$ with $\nabla f$ and $\nabla f^\star$.}
    \label{fig:placeholder}
\end{figure}
% So...\\

% Gaussian (ok)

% Laplace to Huber

% Gumbel to Gumbelmod. 
% \VP{Can we obtain the bound for standard distributions like Gaussian, Laplace, Gumbel? For instance, $\kappa$ is $1$ for Gaussian but what is it for Laplace? Answer: no thats why we need to do smoothing. Planning to add details here}

\subsection{Taylor Accelerated Forward Diffusion}
\label{app: taylor_forward}
An important practical consideration for our training algorithm is the efficiency in the estimation of the score function. In our work, we utilize the Euler-Maruyama approximation in order to sample $Z_t$ given $Z_0$. In practice, this can be inefficient, since it requires ${\cal O}(t)$ sampling steps to sample. For a given time $t_\star\in \{1, \ldots, T\}$, direct score estimation based on Euler–Maruyama is given by: 
\begin{align}
    \label{eq: Euler-Maruyama_forward}
    &Z_0 \sim {\cal D} \\
    &Z_t = Z_{t-1}-\eta\nabla f (Z_{t-1}) +\sqrt{2\eta}\cdot {\cal N}(0, 1), \quad t=1,\ldots,t_\star 
\end{align}
A linear SDE can be solved more easily and allows for ancestral sampling, where $Z_t|Z_0$ can be sampled in a single step. As an example, consider the Ornstein–Uhlenbeck (OU) process:
\begin{equation}
    \label{eq: ou_process}
    dZ_t = \theta(\mu-Z_t)dt +\sigma dW_t
\end{equation}
and its Euler-Maruyama discretized counterpart:
\begin{equation}
    \label{eq: euler_discretization_ou_process}
     Z_t = Z_{t-1} + \theta(\mu-Z_{t-1}) + \sigma \epsilon_t,\quad  \epsilon_t\sim \mathcal{N}(0, 1).
\end{equation}
The discretized process can alternatively be parameterized as:
\begin{equation}
    \label{eq: euler_discrete_another_parameterization_ou}
    Z_t=(1-\theta)Z_{t-1} + \theta\mu + \sigma \epsilon_t
\end{equation}
which allows for straightforward derivation of the conditional $p(Z_t|Z_0)$:
\begin{equation}
    \label{eq: ancestral_sampling_ou_process_euler_approximation}
    p(Z_t|Z_0) = {\cal N}\left(Z_t ; \alpha^t Z_0 + (1-\alpha^t) \mu, \sigma^2\left(\frac{1-\alpha^{2(t+1)}}{1-\alpha^2}\right) \right)
\end{equation}
As we can see above, sampling from a linear SDE like the OU process is easy and does not require multiple rounds of a solver. One idea to make sampling from a nonlinear SDE more efficiently is to consider a first-order Taylor of the score. Particular to this paper, consider a Langevin diffusion with score function $s(Z)= -\nabla_Z f(Z)$, which we know converges to $p(Z_\star)\propto e^{-f(Z)}$ at equilibrium. Consider the first-order Taylor approximation to the score centered around $\tilde{Z}$:
\begin{equation}
    \label{eq: score_taylor_approx}
    s(Z) \approx s(\tilde{Z}) + \nabla_{\tilde{Z}} s(\tilde{Z})(Z-\tilde{Z}) 
\end{equation}
We can see that under this approximation, $s(Z)$ is approximately a linear function in $Z$. It is straightforward to see that by plugging in this approximation into the Langevin SDE, we can employ ancestor sampling as in \eqref{eq: ancestral_sampling_ou_process_euler_approximation} to accelerate the forward diffusion for nonlinear SDEs. In particular, we can easily see that under this linear approximation, the Langevin SDE will reduce to an OU process with certain parameterization:
\begin{align}
Z_t &= Z_{t-1} - \eta \nabla_{Z_{t-1}} s(Z_{t-1}) + \sqrt{2\eta} \epsilon_t \\
&\approx Z_{t-1} -\eta(s(\tilde{Z}) + \nabla_{\tilde{Z}} s(\tilde{Z}) (Z_{t-1}-\tilde{Z}))+ \sqrt{2\eta} \epsilon_t \\
&= Z_{t-1} -\eta s(\tilde{Z}) - \eta \nabla_{\tilde{Z}} s(\tilde{Z}) (Z_{t-1}-\tilde{Z})+ \sqrt{2\eta} \epsilon_t \\
&= Z_{t-1} -\eta s(\tilde{Z}) -\eta \nabla_{\tilde{Z}} s(\tilde{Z})Z_{t-1}+ \eta \nabla_{\tilde{Z}} s(\tilde{Z})\tilde{Z}+ \sqrt{2\eta} \epsilon_t \\
&= \left(1-\eta \nabla_{\tilde{Z}} s(\tilde{Z}) \right) Z_{t-1} + \eta\nabla_{\tilde{Z}} s(\tilde{Z}) \left(\tilde{Z}-\left(\nabla_{\tilde{Z}} s(\tilde{Z})\right)^{-1}s(\tilde{Z})\right) + \sqrt{2\eta} \epsilon_t
\end{align} 
We can see that this is an OU process with the following parameters:
\begin{align}
    &\theta=\eta \nabla_{\tilde{Z}}s(\tilde{Z}) \\
    &\mu=\tilde{Z}-\left(\nabla_{\tilde{Z}} s(\tilde{Z})\right)^{-1} s(\tilde{Z}) \\
    &\sigma =2\eta
\end{align}
This means that we can apply ancestral sampling to the Taylor approximation of our Langevin diffusion. We refer the reader to the pseudocode in Algorithm \ref{alg:taylor_forward_sampling} for our specific implementation.

\begin{algorithm}[H]
  \caption{Taylor-Accelerated Forward Sampling}
  \label{alg:taylor_forward_sampling}
  \begin{algorithmic}[1]
    \State \textbf{Input:} Initial residual $Z_0$, conditioning variable $X$, target time $t_\star$, step size $\eta$, Taylor steps function $K(t)$
    \State \textbf{Initialize:} Current state $Z_{curr} = Z_0$, current time $t_{curr} = 0$
    \While{$t_{curr} < t_\star$}
        \State \textbf{Determine number of Taylor horizon:} Set $k = \min(K(t_{curr}), t_\star - t_{curr})$
        \State \textbf{Compute Taylor approximation:} 
        \begin{align*}
            s_{curr} &= s_\theta(Z_{curr}; X, t_{curr}) \\
            \nabla s_{curr} &= \nabla_{Z_{curr}} s_\theta(Z_{curr}; X, t_{curr})
        \end{align*}
        \State \textbf{Set OU parameters:}
        \begin{align*}
            \alpha &= 1 - \eta \nabla s_{curr} \\
            \mu_{eff} &= Z_{curr} - \frac{s_{curr}}{\nabla s_{curr}} \\
            \sigma_{eff} &= \sqrt{2\eta}
        \end{align*}
        \State \textbf{Ancestral sampling:} Sample directly at time $t_{curr} + k$:
        $Z_{curr} \sim \mathcal{N}\left(\alpha^k Z_{curr} + (1-\alpha^k) \mu_{eff}, \sigma_{eff}^2 \frac{1-\alpha^{2k}}{1-\alpha^2}\right)$
        \State \textbf{Update time:} $t_{curr} \gets t_{curr} + k$
    \EndWhile
    \State \textbf{Return:} Final residual $Z_{t_\star} = Z_{curr}$
  \end{algorithmic}
\end{algorithm}

% \section{Error for Taylor Approximation }
% \paragraph{Maybe Appendix: Theoretical analysis:} We would like to theoretically analyze the effect the first-order Taylor linearization has on the forward process. In particular, we would like to understand on the Wasserstein distance. For any probability measures $\mu$ and $\nu$ on $\mathbb{R}$, the Wasserstein-2 distance is defined as:
% \begin{equation}
%     W_2^2(\mu, \nu) = \inf_{\gamma\in\Gamma(\mu, \nu)}\int_{\mathbb{R}\times\mathbb{R}} |x-y|^2 d\gamma(x, y)
% \end{equation}
% We consider two approaches for generating the random variable $X_t$. One by solving a Langevin SDE using the \underline{\bf Euler method}:
% \begin{align}
%     &X_0^{\cal E} \sim \mu_0 \\
%     &X_k^{\cal E} = X_{k-1}^{\cal E} - \eta f'(X_{k-1}^{\cal E})dt + \sqrt{2\eta} \epsilon_k, \quad k=1,\ldots, K 
% \end{align}
% where $t = \eta K$. Consider a subset of $\{k_0, k_1, \ldots, k_L\}\in[K]$, such that $k_0=0$ and $k_L=K$. The \underline{\bf Taylor method} simulates $X_t$ using the following: 
% \begin{align}
%     &X_0^{\cal T} \sim \mu_0 \\
%     &X_{k}^{\cal T} = (1-\eta f''(\tilde{x}_k))X_{k-1}^{\cal T} + \eta f''(\tilde{x}_k)\left(\tilde{x}_k-\frac{f'(\tilde{x}_k)}{f''(\tilde{x}_k)}\right) + \sqrt{2\eta} \epsilon_k, \quad k=1,\ldots, K 
% \end{align}
% where we define $\tilde{x}_k = \sum_{\ell=0}^L \tilde{x}_l \cdot \mathbf{1}(k_{\ell-1}\leq k<k_{\ell})$. 

\subsection{Pseudocode for Methodology}
\label{app: algorithms}
% \begin{algorithm}[H]
%   \caption{CEVT-based Data Preprocessing}
%   \label{alg:iterative_training}
%   \begin{algorithmic}
%     \State {\bf Input}:  $\{(X_i, Y_i)\}_{i=1}^n$
%     \State {\bf Initialization}: Choose threshold quantile, $\alpha >  0$.
%     \State Estimate empirical CDF $\hat{F}_X, \ \hat{F}_Y$
%     \State Transform to Laplace marginals \begin{align*}
%         &X^\star_i \gets -\text{sign}(\hat{F}_X(X_i) - 0.5) \cdot \log(1 - 2|\hat{F}_X(X_i) - 0.5|)  \\&Y^\star_i \gets -\text{sign}(\hat{F}_Y(Y_i) - 0.5) \cdot \log(1 - 2|\hat{F}_Y(Y_i) - 0.5|) 
%     \end{align*}
%     \State Find subset $\{Y^\star_i, X^\star_i\}_{i = 1}^m \text{ where}\ F(X_i) > 1- \alpha$
%     \State Estimate $a, b$ in expression () using tail samples $\{Y^\star_i, X^\star_i\}_{i = 1}^m$
%     \State Set $Z_i = \frac{Y^\star_i - a\cdot X^\star_i}{{X^\star_i}^b}, \ i = 1\ldots n$
    
%     \State {\bf Return:} $\{(X_i^\star, Z_i)\}_{i=1}^n$
%   \end{algorithmic}
% \end{algorithm}

\begin{algorithm}[H]
  \caption{CEVT-based Data Preprocessing}
  \label{alg:cevt_preprocessing}
  \begin{algorithmic}[1]
    \State \textbf{Input:} $\{(X_i, Y_i)\}_{i=1}^n$, threshold quantile $\alpha > 0$
    \State \textbf{Estimate empirical CDFs:} Compute $\hat{F}_X$ and $\hat{F}_Y$ from data
    \State \textbf{Transform to Laplace marginals:} For all samples $i=1,\ldots, n$ {\bf do}
    \begin{align*}
        X^\star_i &\gets -\text{sign}(\hat{F}_X(X_i) - 0.5) \cdot \log(1 - 2|\hat{F}_X(X_i) - 0.5|) \\
        Y^\star_i &\gets -\text{sign}(\hat{F}_Y(Y_i) - 0.5) \cdot \log(1 - 2|\hat{F}_Y(Y_i) - 0.5|)
    \end{align*}
    \State \textbf{Extract tail samples:} Find subset $\{(X^\star_i, Y^\star_i)\}_{i=1}^m$ where $\hat{F}_X(X_i) > 1 - \alpha$
    \State \textbf{Estimate tail parameters:} Compute coefficients $a, b$ using tail samples $\{(X^\star_i, Y^\star_i)\}_{i=1}^m$
    \State \textbf{Compute residuals:} Set $Z_i = \frac{Y^\star_i - a \cdot X^\star_i}{(X^\star_i)^b}$ for $i = 1, \ldots, n$
    \State \textbf{Return:} Preprocessed dataset $\{(X^\star_i, Z_i)\}_{i=1}^n$
  \end{algorithmic}
\end{algorithm}

% \begin{algorithm}[H]
%   \caption{CEVT-Diffusion Training}
%   \label{alg:iterative_training}
%   \begin{algorithmic}
%     \State {\bf Input}:  $\{(X_i^\star, Z_i)\}_{i=1}^n$
%     \State {\bf Initialization}: 
%     \State Step 1
%     \State Step 2
%     \State $\vdots$
%     \State Step K
%     \State {\bf Return:} $\theta_\star$
%   \end{algorithmic}
% \end{algorithm}

\begin{algorithm}[H]
  \caption{Diffusion Model Training}
  \label{alg:cevt_diffusion_training}
  \begin{algorithmic}[1]
    \State \textbf{Input:} Preprocessed dataset $\{(X_i^\star, Z_i)\}_{i=1}^n$, learning rate $\eta$, epochs $E$, weighting function $\lambda(t)$
    \State \textbf{Initialize:} Network parameters $\theta$, noise schedule parameters
    \For{epoch $e = 1, \ldots, E$}
        \For{batch $\{(X_j^\star, Z_j)\}_{j \in \text{batch}}$}
            \State \textbf{Sample timestep:} Sample $t$ uniformly over time-horizon.
            \State \textbf{Generate noisy samples:} Sample $Z_t | Z_0 = Z_j$ according to forward process using 
            \Statex \hspace{2.75em} Euler-Maruyama solver of Taylor-accelerated sampling in Algorithm \ref{alg:taylor_forward_sampling}.
            \State \textbf{Compute score matching loss:} Evaluate ${\cal L}(\theta)$ in \eqref{eq: score_matching_loss} for each sample in the batch.
            \State \textbf{Backward pass:} Compute gradients $\nabla_\theta \mathcal{L}(\theta)$
            \State \textbf{Update parameters:} $\theta \gets \theta - \eta \nabla_\theta \mathcal{L}(\theta)$
        \EndFor
    \EndFor
    \State \textbf{Return:} Trained parameters $\theta$
  \end{algorithmic}
\end{algorithm}

\begin{algorithm}[H]
  \caption{Diffusion Model Sampling}
  \label{alg:reverse_diffusion_detail}
  \begin{algorithmic}[1]
    \State \textbf{Input:} Initial noise $Z_T$, conditioning $X^\star$, trained score $s_{\theta_\star}$, time horizon $T$
    \State \textbf{Initialize:} Current state $Z_{curr} = Z_T$, current time $t_{curr} = T$
    \While{$t_{curr} > 0$}
        \State \textbf{Determine step size:} Set $k = \min(K(t_{curr}), t_{curr})$
        \State \textbf{Compute score:} Evaluate $s_{\theta_\star}(Z_{curr}; X^\star, t_{curr})$
        \State \textbf{Reverse step:} Apply reverse SDE or Euler-Maruyama:
        $$Z_{curr} = Z_{curr} + \eta \cdot s_{\theta_\star}(Z_{curr}; X^\star, t_{curr}) + \sqrt{2\eta} \epsilon$$
        where $\epsilon \sim \mathcal{N}(\mathbf{0}, \mathbf{I})$
        \State \textbf{Update time:} $t_{curr} \gets t_{curr} - k$
    \EndWhile
    \State \textbf{Return:} Denoised residual $Z_0 = Z_{curr}$
  \end{algorithmic}
\end{algorithm}

\section{Additional Experimental Results}
Here, we provide additional plots and metrics for the experiments section of our work. 
\subsection{Synthetic Data}
\label{app: additional_details_syn}
\begin{figure}[H]
    \centering
    \begin{subfigure}[t]{0.98\textwidth}
        \centering
        \includegraphics[width=0.98\linewidth, trim=0cm 0cm 0cm 0cm, clip]{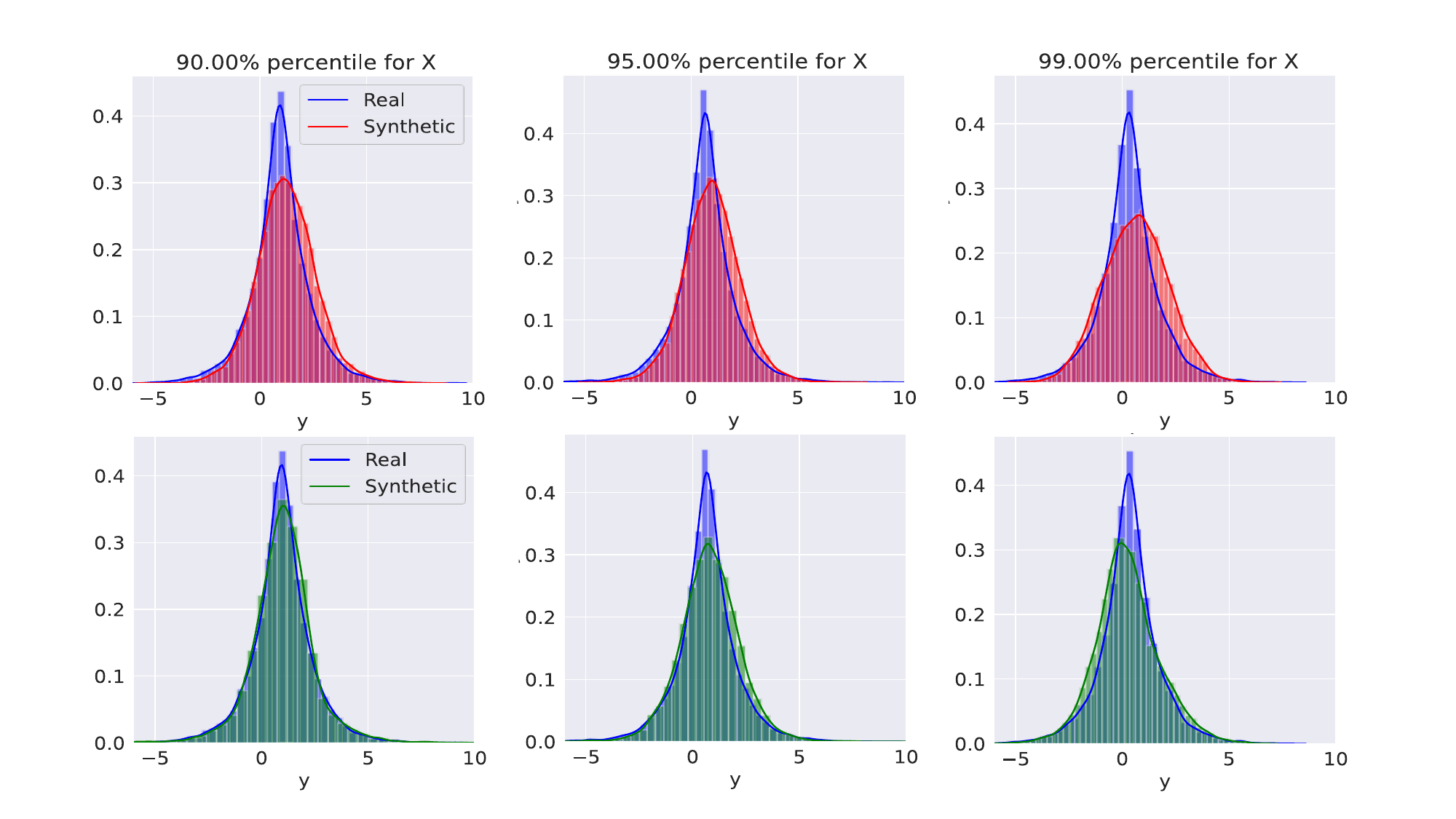}
        \caption{Top Row: Standard method targeting $P(Y|X)$ with linear diffusion. Bottom Row: New method. New method manages to capture the heavy Laplace tails, standard method struggles to do so. }
    \end{subfigure}\\
     \begin{subfigure}[t]{0.98\textwidth}
        \centering
        \includegraphics[width=0.98\linewidth]{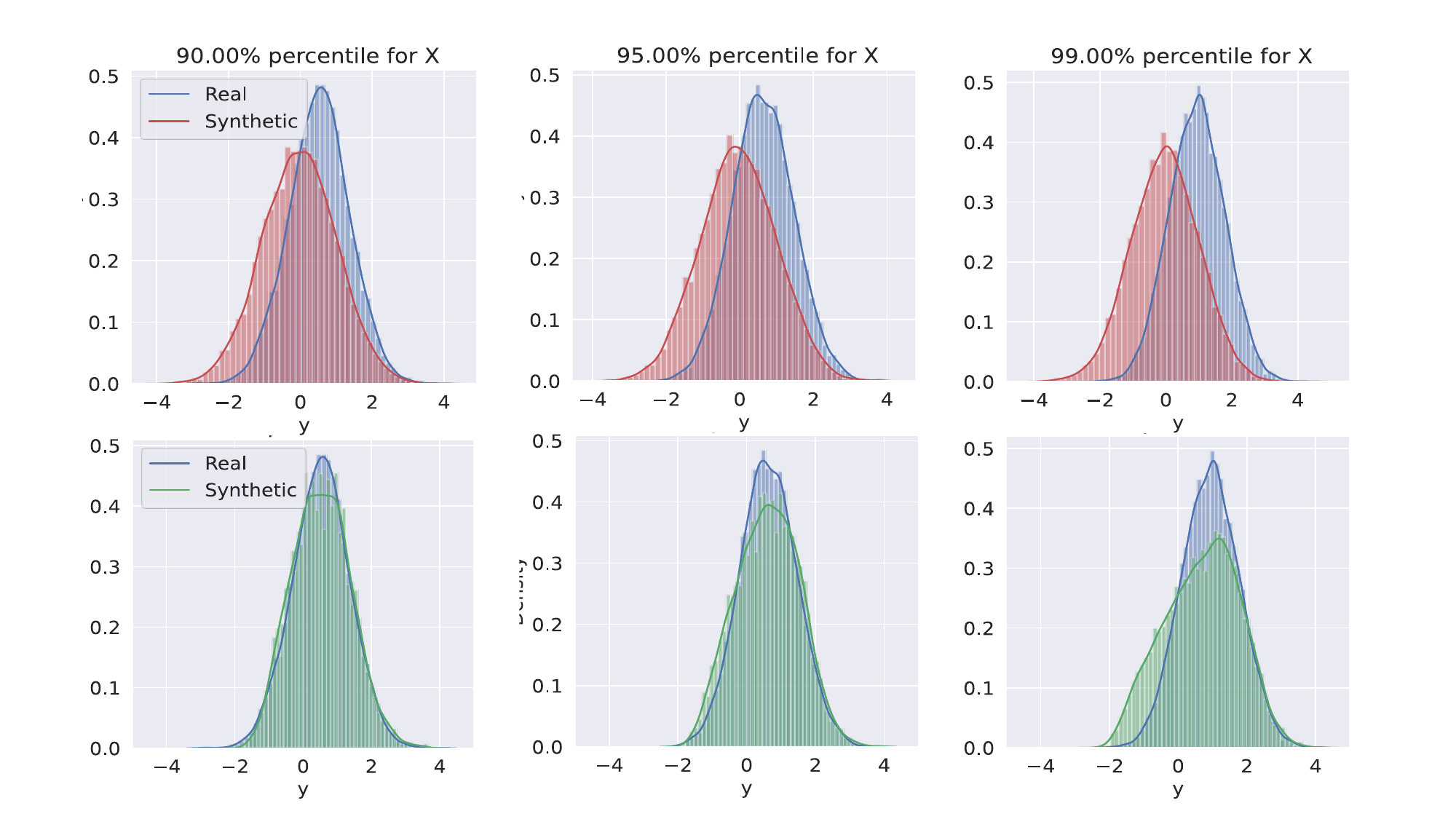}
        \caption{Top Row: Standard method targeting $P(Y|X)$ with linear diffusion. Bottom Row: New method.}
    \end{subfigure}
    \caption{(a) Synthetic Example 1 (b) Synthetic Example 2}
    \label{fig: gaussian_vs_laplace_synthetc_experiment}
\end{figure}
\begin{figure}[H]
    \centering
    \includegraphics[width=\linewidth, trim=0cm 5cm 0cm 5cm, clip]{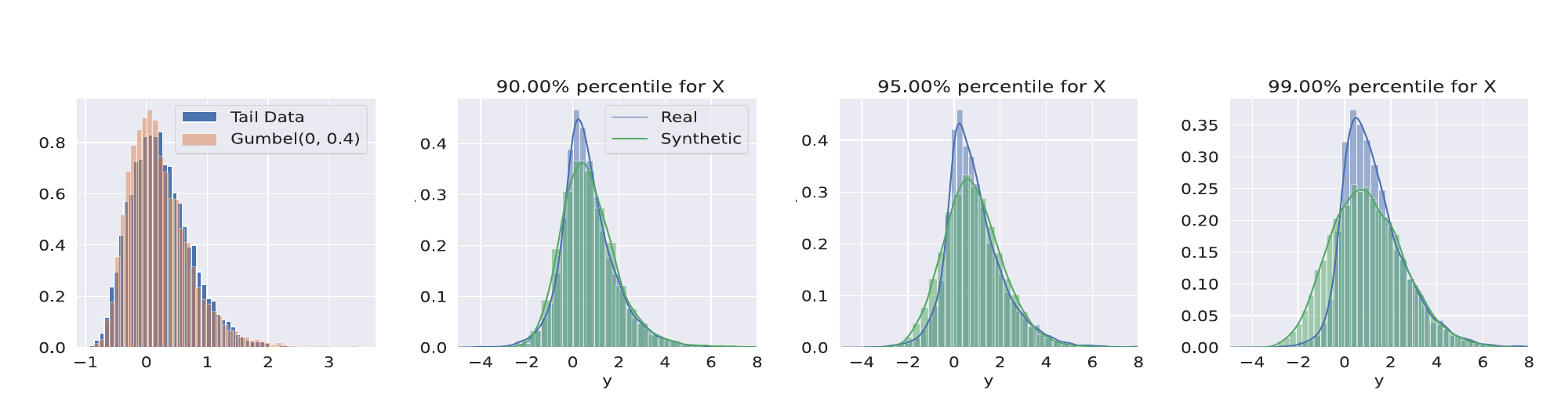}
    \caption{Left Plot: As discussed, we see for extreme (but not infinite) values in the tail, data seem Gumbel distributed. We visualize sampling in the CEVT based representation space ($P(Z|X^\star$) in the subsequent plots. We capture the one-sided tails.}
    \label{fig: choice_of_gumbel}
\end{figure}
\subsection{Financial Returns Conditioned on VIX}
\label{app: additional_details_stocks}
We provide additional and more detailed experimental results for our evaluation on real data. 

\subsubsection{VIX Time Series}
Here, we show a plot of the VIX time series in Figure \ref{fig: vix_time_series}, which serves as the conditional information supplied to the diffusion models for the stock return generation experiment. For both the GFC and COVID periods, the VIX level is relatively lower in the training data (plotted in blue) than the testing data (plotted in orange), indicating that the testing data covers a period of market stress. 

\begin{figure}[htbp]
    \centering
    \begin{subfigure}[b]{0.48\textwidth}
        \includegraphics[width=\linewidth]{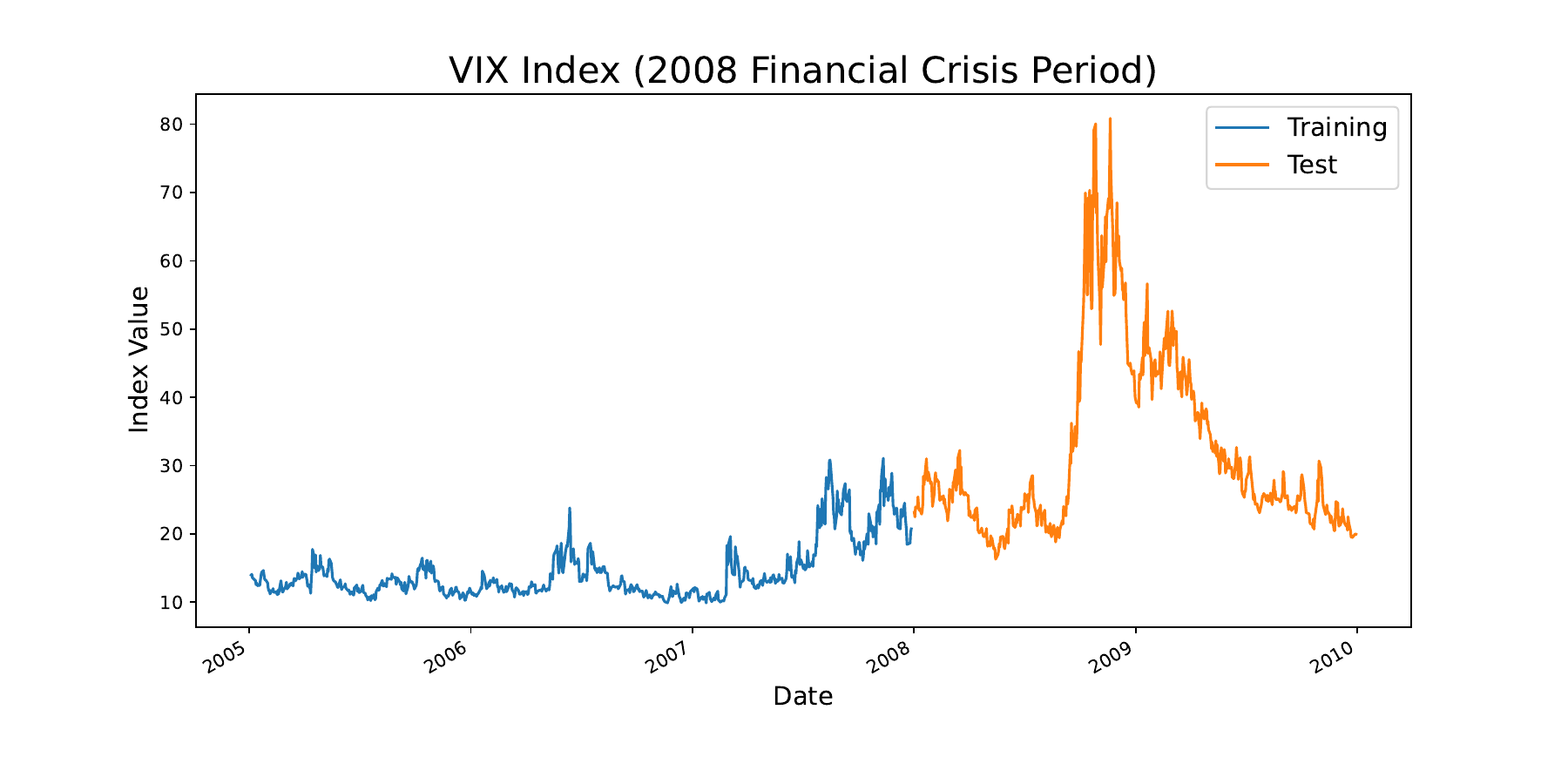}
        \caption{GFC}
    \end{subfigure}
    \hfill
    \begin{subfigure}[b]{0.48\textwidth}
        \includegraphics[width=\linewidth]{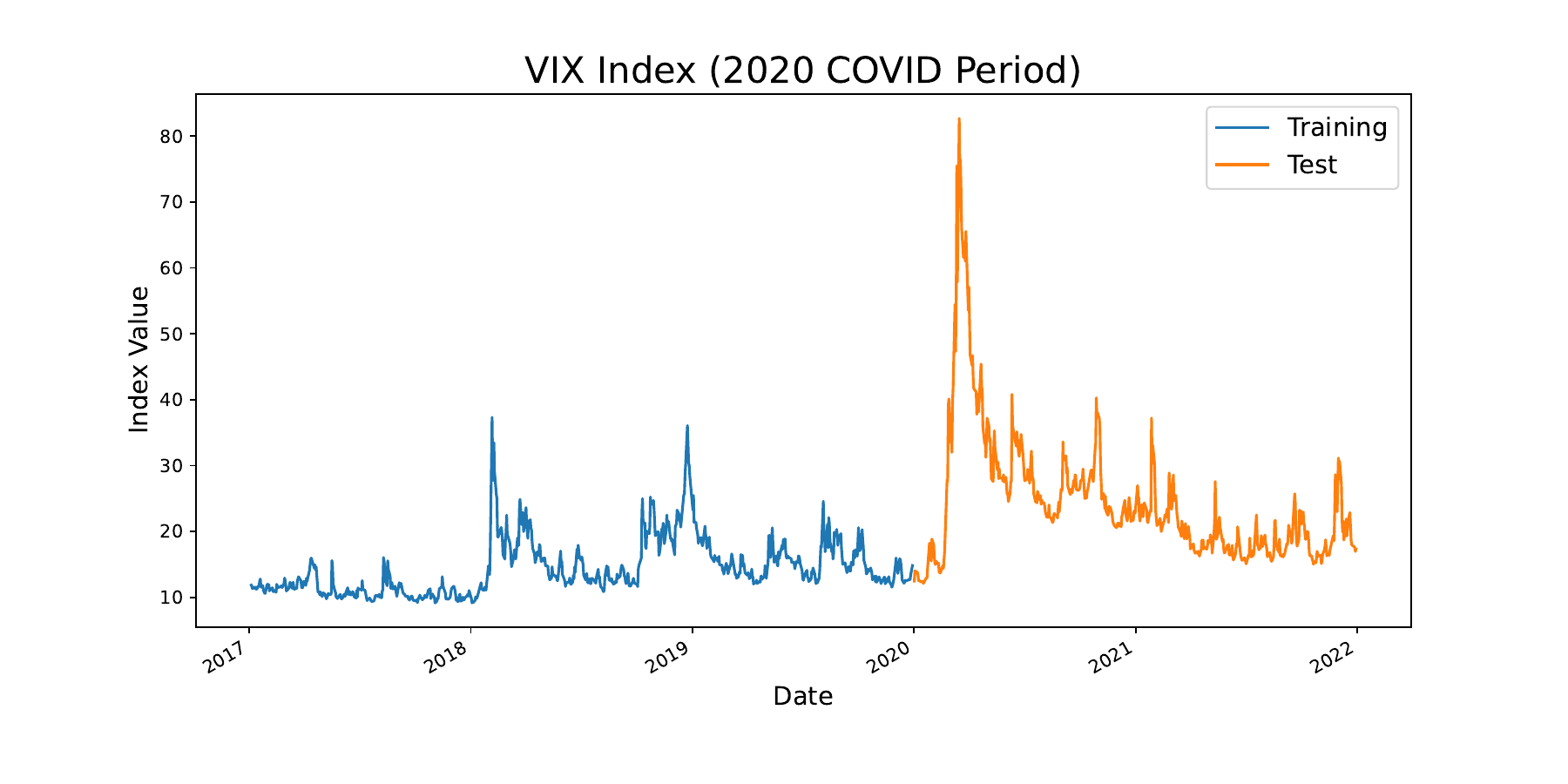}
        \caption{COVID}
    \end{subfigure}
    \caption{VIX level during the analyzed periods of market stress. VIX level in the training datasets (shown in blue) correspond to more stable market periods, while VIX levels in the testing dataset (shown in orange) correspond to a period of market stress.}
    \label{fig: vix_time_series}
\end{figure}

\subsubsection{Evaluation of Calibration via QQ plots (Unconditional Evaluation)}
To evaluate the unconditional generative performance (where we marginalize out the conditions) of the proposed conditional diffusion model, we use QQ plots to check for the calibration of the predicted quantiles versus the true quantiles from the empirical dataset. Figure \ref{fig: qq_plots_gfc_training_base} and \ref{fig: qq_plots_gfc_testing} show the QQ plots for each stock on the training and testing datasets for the GFC period, respectively.  Figure \ref{fig: qq_plots_covid_training_base} and \ref{fig: qq_plots_covid_testing} show the QQ plots for each stock on the training and testing datasets for the GFC period, respectively.  The results indicate that while the use of a Gaussian base distribution generally leads to better calibration in the training dataset and in the bulk of the distribution (10\%-90\% quantiles), the use of a Laplace distribution offers a significant advantage in the tail, specifically for the testing datasets, since the testing dataset considers VIX levels (conditions) much larger than what is seen in the training dataset. This showcases the advantages of considering alternative base distributions in the case of generative modeling for heavy-tailed targets.

\begin{figure}[htbp]
    \centering
    \begin{subfigure}[b]{0.24\textwidth}
        \includegraphics[width=\textwidth, trim=0cm 13cm 28cm 0cm, clip]{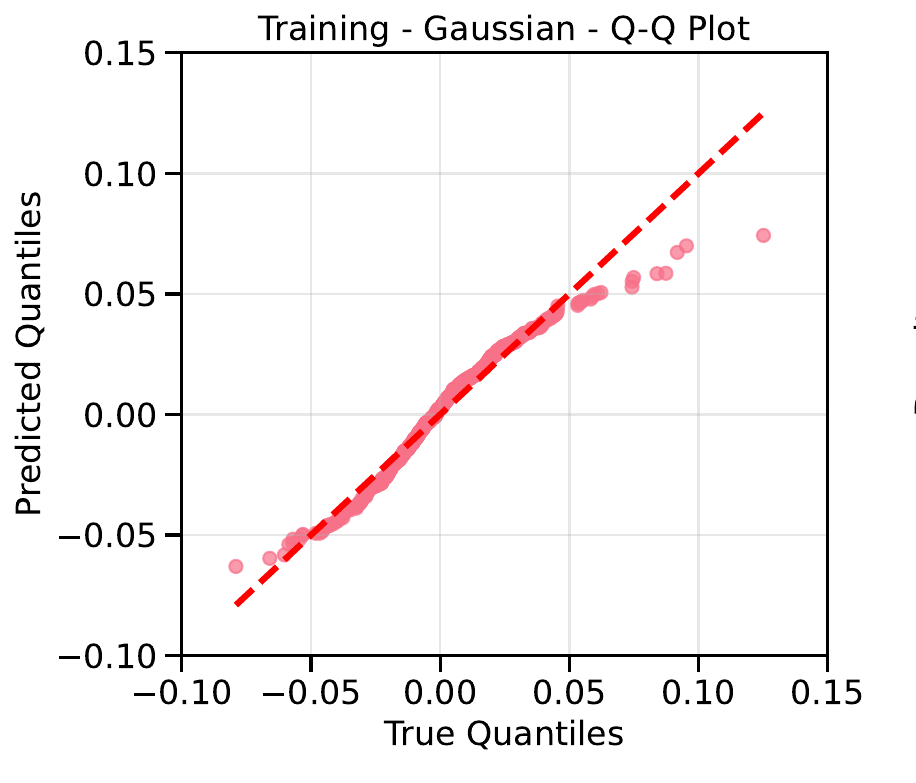}
        \caption{AAPL - Gaussian}
    \end{subfigure}
    \hfill
    \begin{subfigure}[b]{0.24\textwidth}
        \includegraphics[width=\textwidth, trim=0cm 13cm 28cm 0cm, clip]{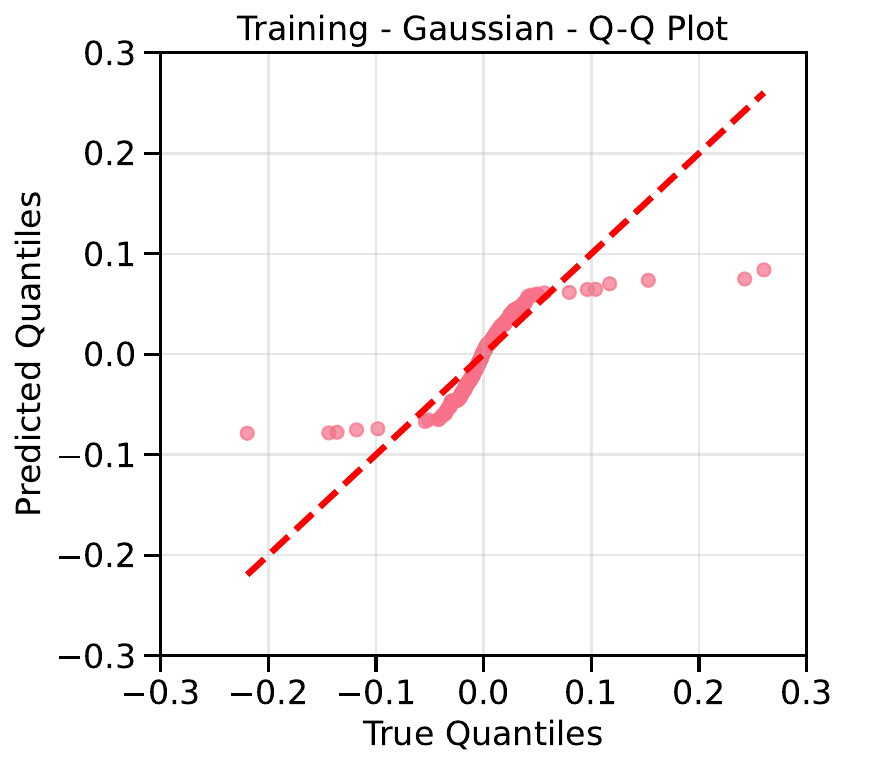}
        \caption{AMZN - Gaussian}
    \end{subfigure}
    \hfill
    \begin{subfigure}[b]{0.24\textwidth}
        \includegraphics[width=\textwidth, trim=0cm 13cm 28cm 0cm, clip]{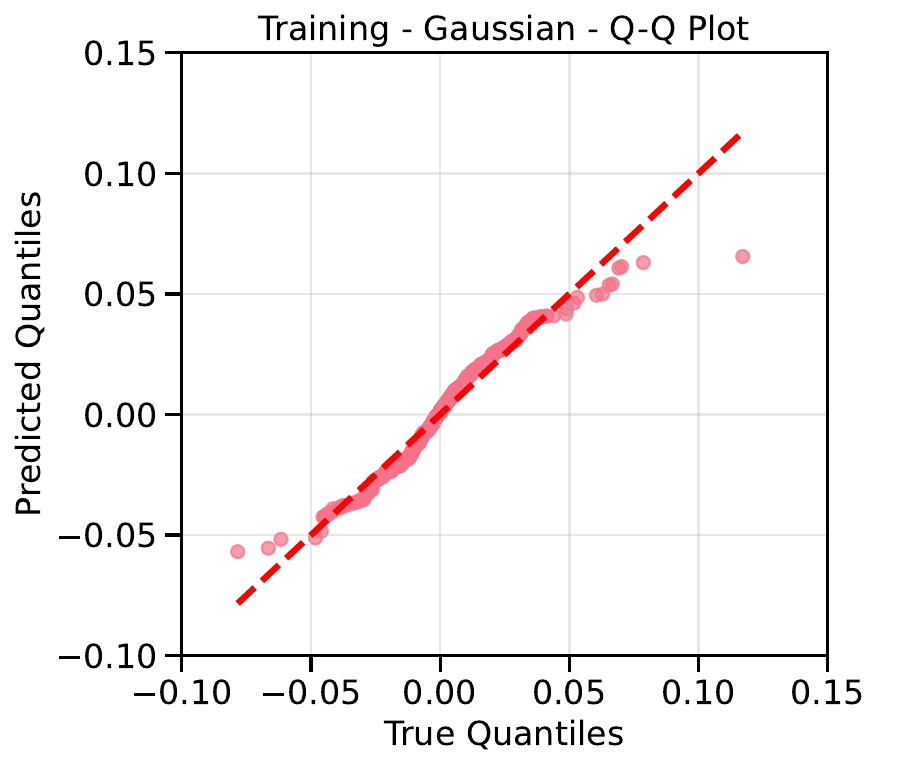}
        \caption{GOOGL - Gaussian.}
    \end{subfigure}
    \hfill
    \begin{subfigure}[b]{0.24\textwidth}
        \includegraphics[width=\textwidth, trim=0cm 13cm 28cm 0cm, clip]{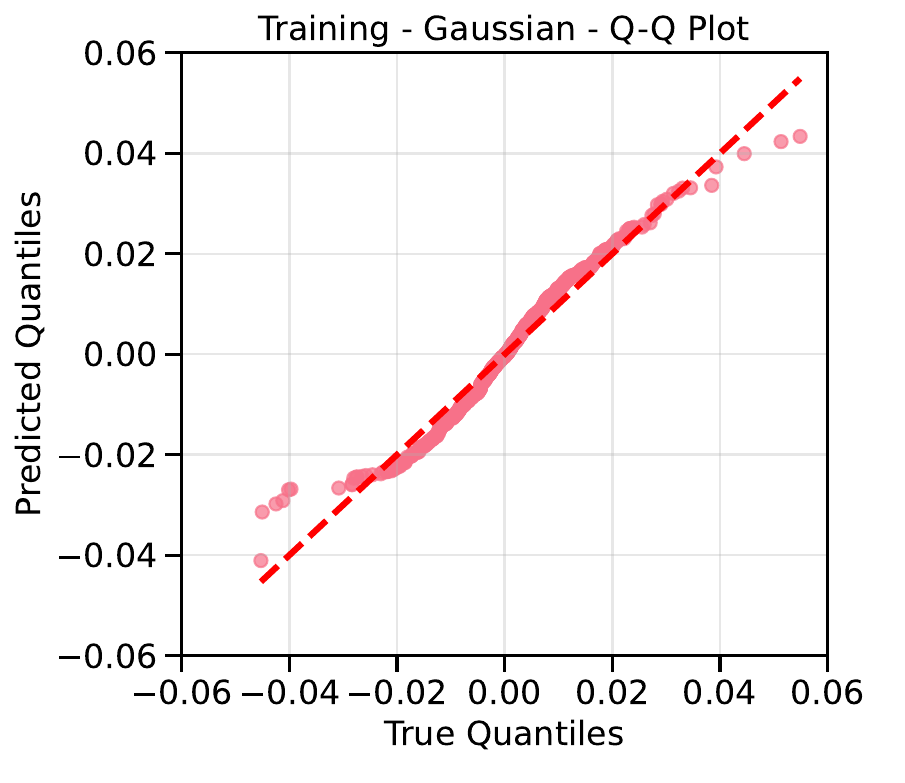}
        \caption{GS - Gaussian}
    \end{subfigure}
    \vspace{0.5em}
    \begin{subfigure}[b]{0.24\textwidth}
        \includegraphics[width=\textwidth, trim=0cm 13cm 28cm 0cm, clip]{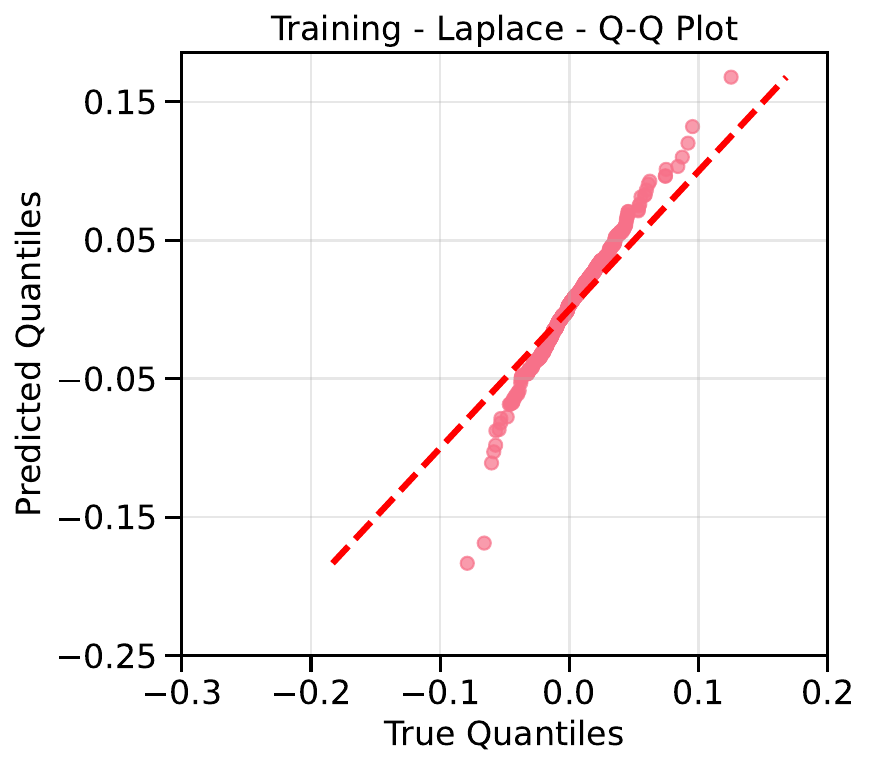}
        \caption{AAPL - Laplace}
    \end{subfigure}
    \hfill
    \begin{subfigure}[b]{0.24\textwidth}
        \includegraphics[width=\textwidth, trim=0cm 13cm 28cm 0cm, clip]{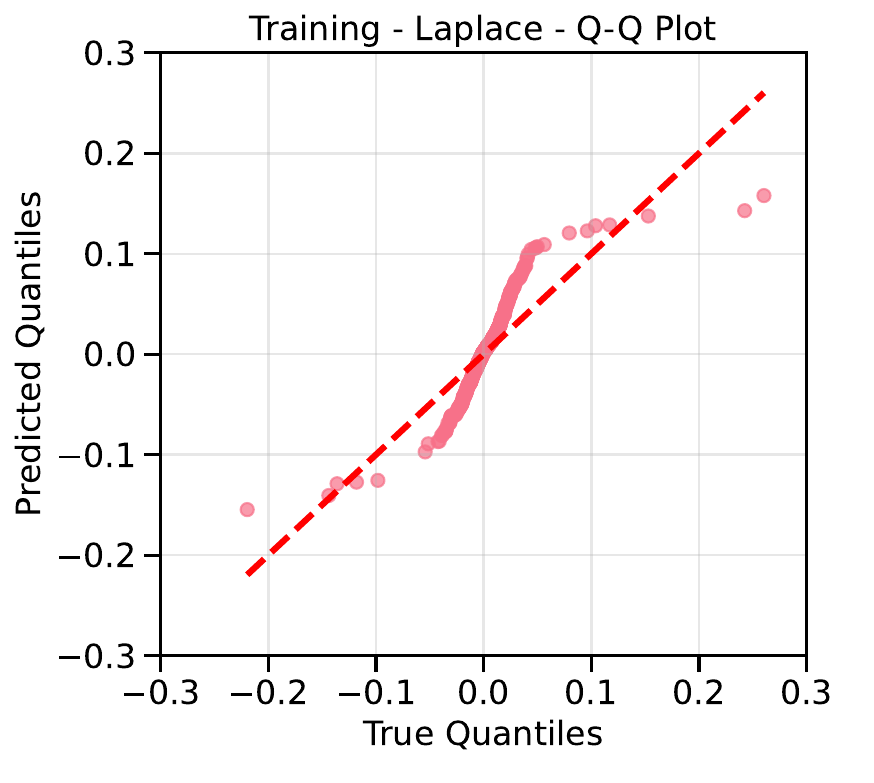}
        \caption{AMZN - Laplace}
    \end{subfigure}
    \hfill
    \begin{subfigure}[b]{0.24\textwidth}
        \includegraphics[width=\textwidth, trim=0cm 13cm 28cm 0cm, clip]{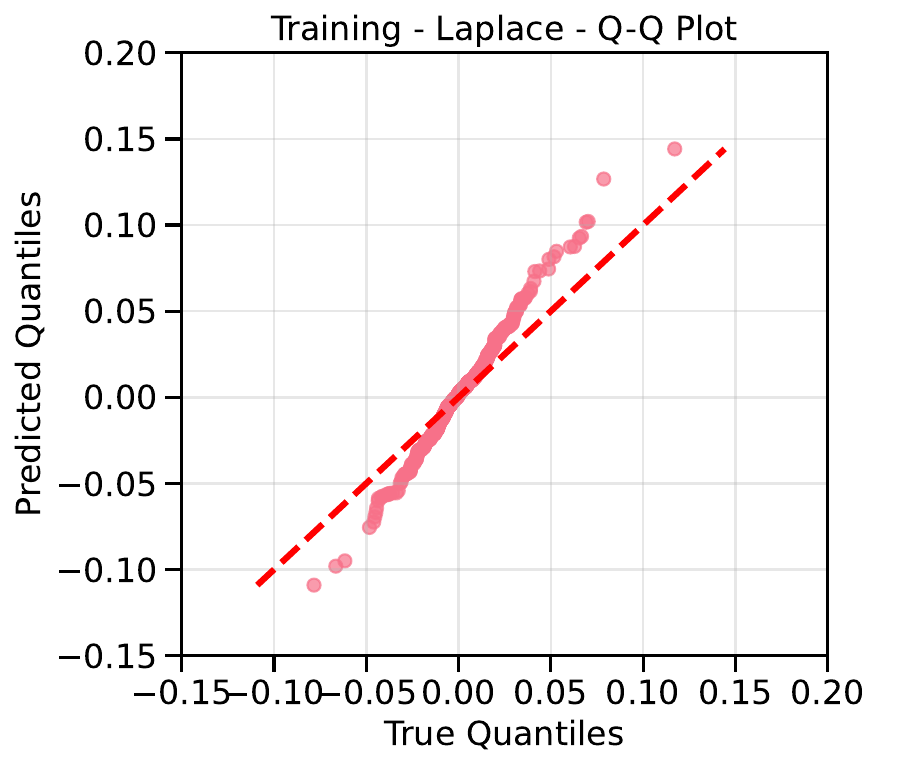}
        \caption{GOOGL - Laplace.}
    \end{subfigure}
    \hfill
    \begin{subfigure}[b]{0.24\textwidth}
        \includegraphics[width=\textwidth, trim=0cm 13cm 28cm 0cm, clip]{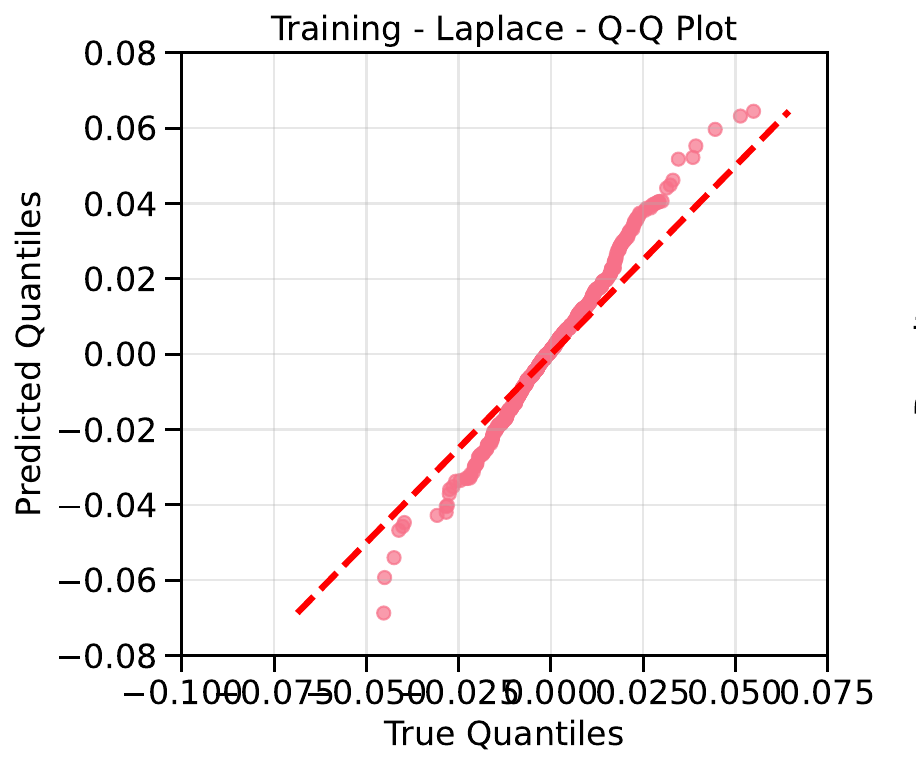}
        \caption{GS - Laplace}
    \end{subfigure}
    \vspace{0.5em}
    
    \begin{subfigure}[b]{0.24\textwidth}
        \includegraphics[width=\textwidth, trim=0cm 13cm 28cm 0cm, clip]{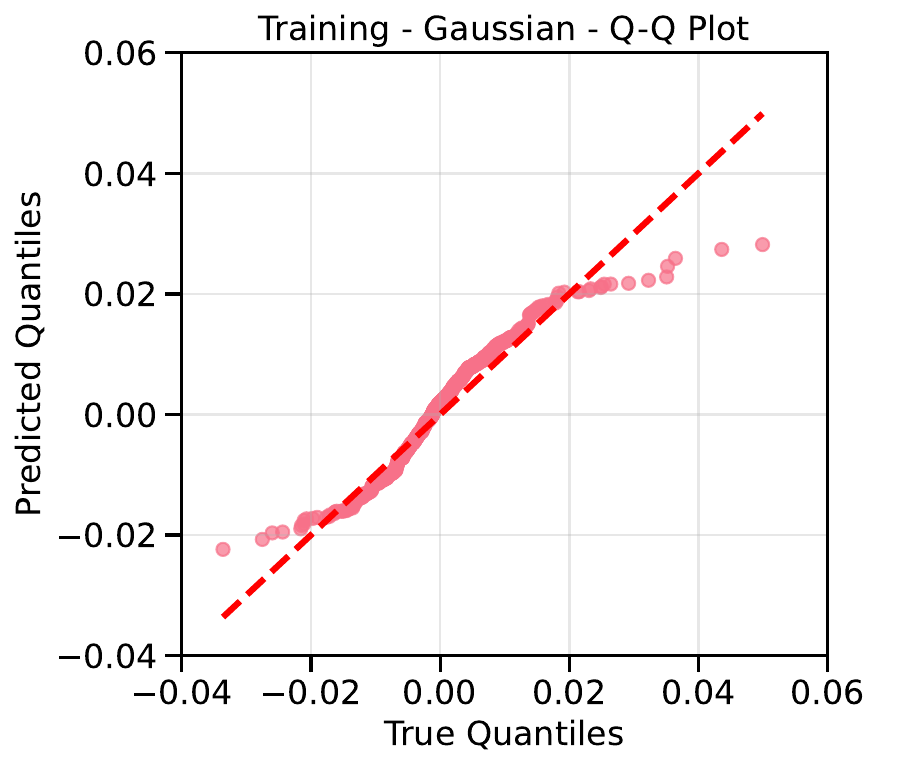}
        \caption{JPM - Gaussian}
    \end{subfigure}
    \hfill
    \begin{subfigure}[b]{0.24\textwidth}
        \includegraphics[width=\textwidth, trim=0cm 13cm 28cm 0cm, clip]{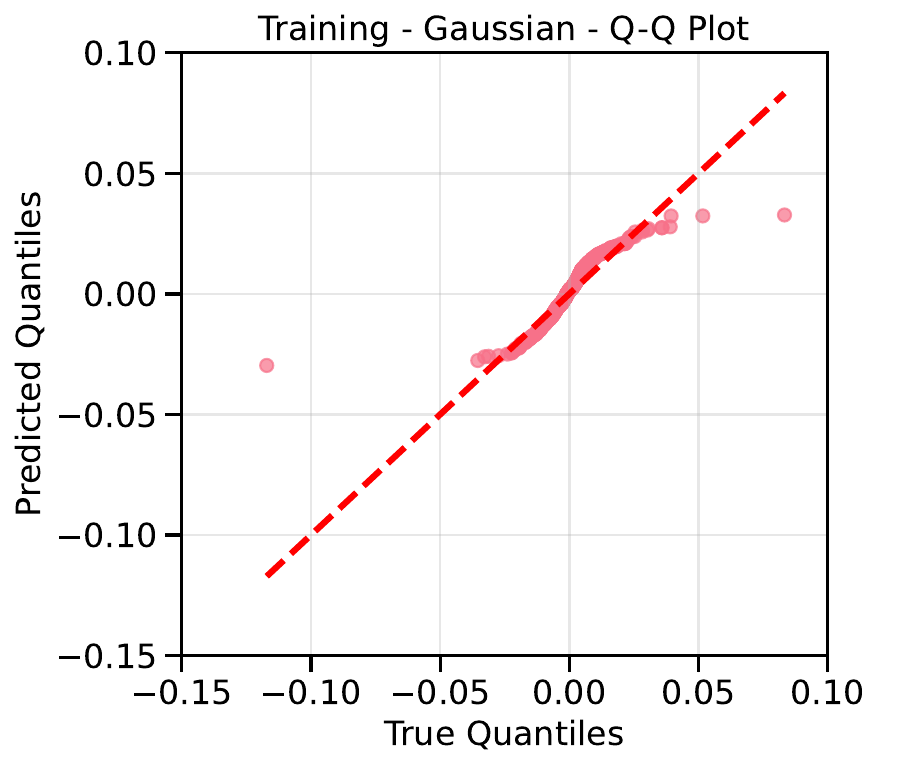}
        \caption{MSFT - Gaussian}
    \end{subfigure}
    \hfill
    \begin{subfigure}[b]{0.24\textwidth}
        \includegraphics[width=\textwidth, trim=0cm 13cm 28cm 0cm, clip]{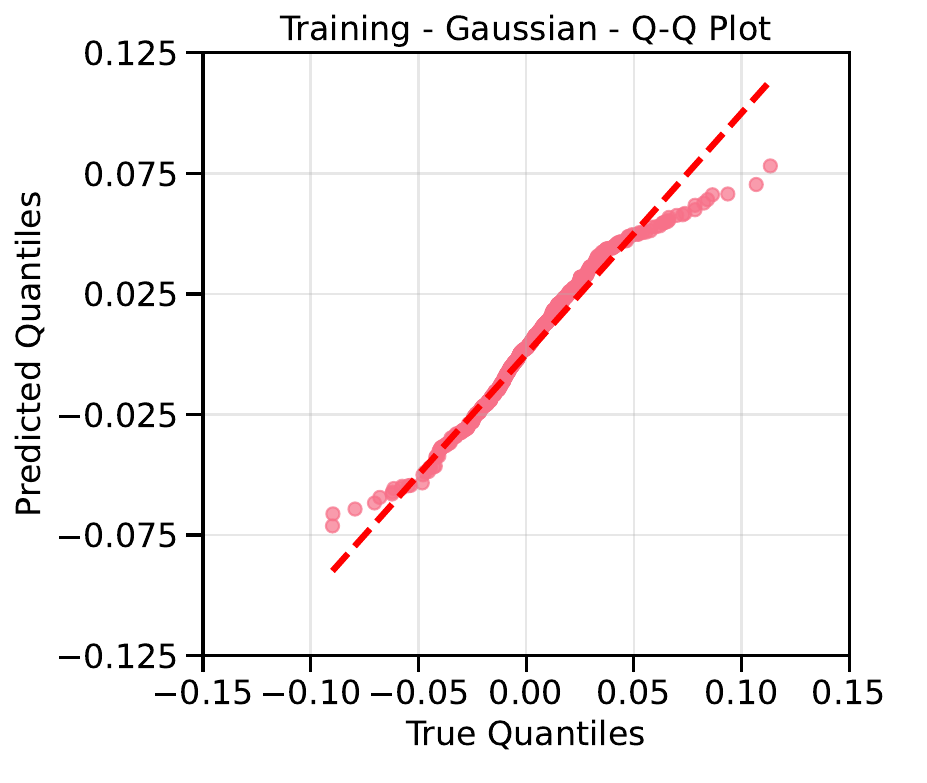}
        \caption{NVDA - Gaussian.}
    \end{subfigure}
    \hfill
    \begin{subfigure}[b]{0.24\textwidth}
        \includegraphics[width=\textwidth, trim=0cm 13cm 28cm 0cm, clip]{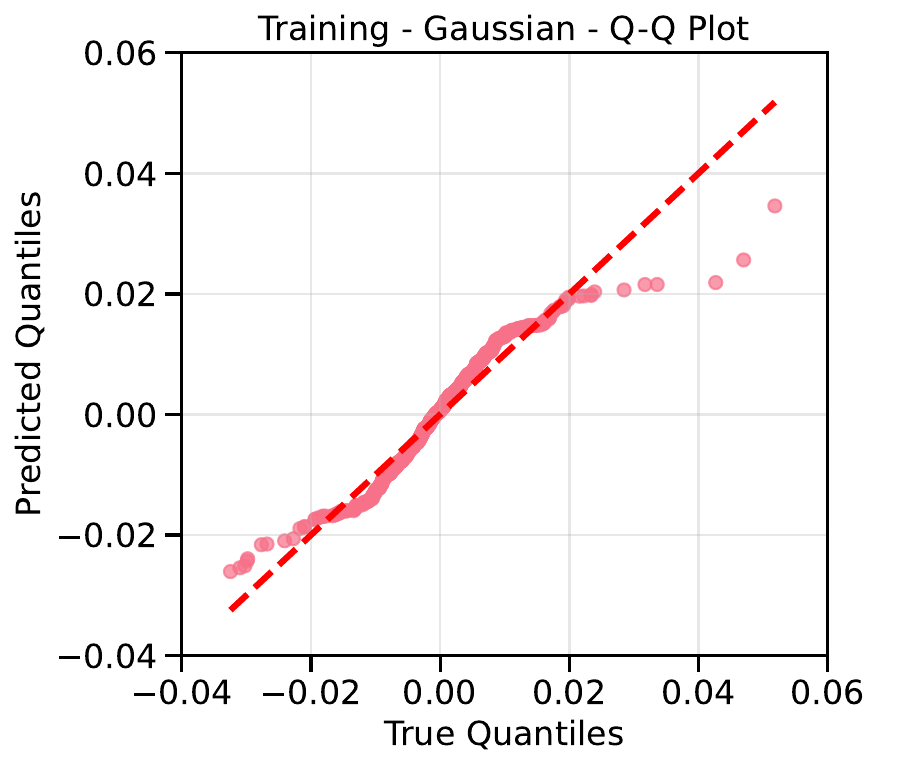}
        \caption{WFC - Gaussian}
    \end{subfigure}
    
    \vspace{0.5em}
    
    \begin{subfigure}[b]{0.24\textwidth}
        \includegraphics[width=\textwidth, trim=0cm 13cm 28cm 0cm, clip]{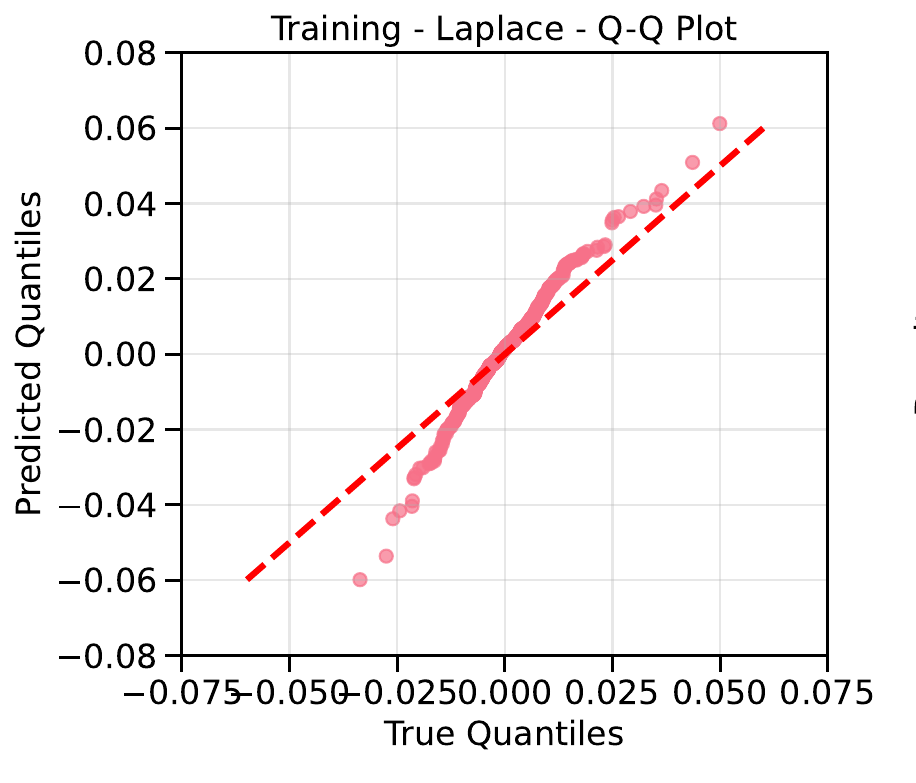}
        \caption{JPM - Laplace}
    \end{subfigure}
    \hfill
    \begin{subfigure}[b]{0.24\textwidth}
        \includegraphics[width=\textwidth, trim=0cm 13cm 28cm 0cm, clip]{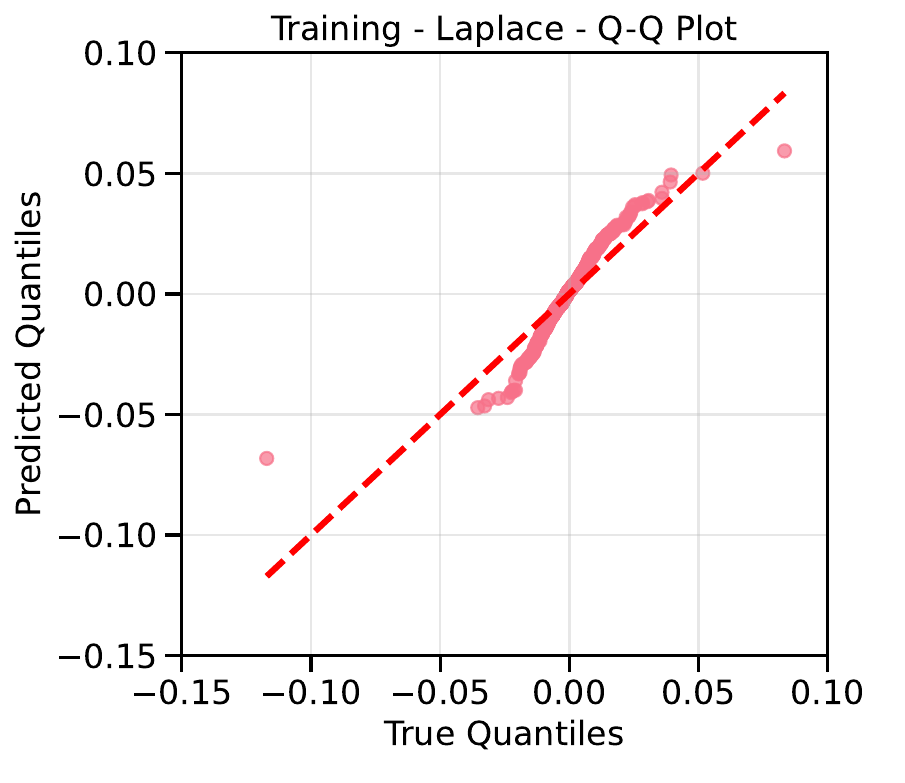}
        \caption{MSFT - Laplace}
    \end{subfigure}
    \hfill
    \begin{subfigure}[b]{0.24\textwidth}
        \includegraphics[width=\textwidth, trim=0cm 13cm 28cm 0cm, clip]{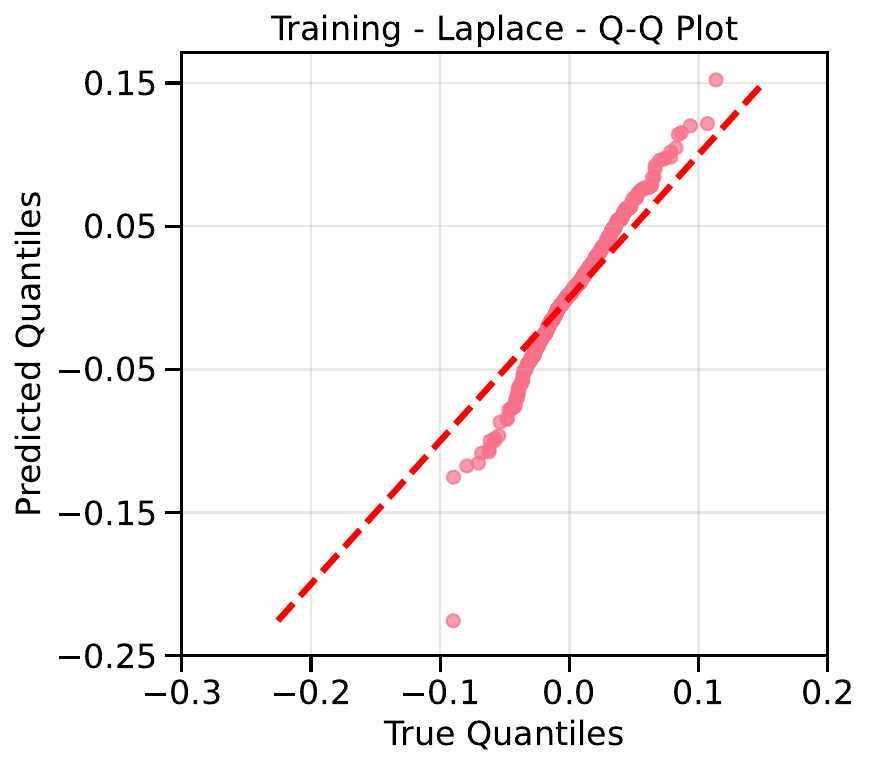}
        \caption{NVDA - Laplace.}
    \end{subfigure}
    \hfill
    \begin{subfigure}[b]{0.24\textwidth}
        \includegraphics[width=\textwidth, trim=0cm 13cm 28cm 0cm, clip]{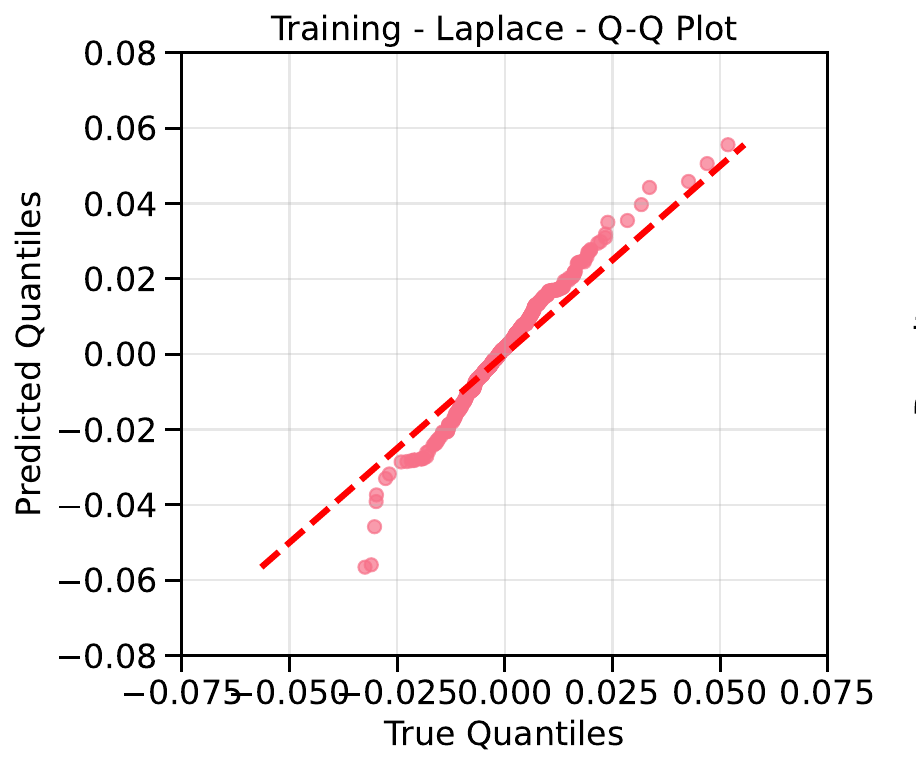}
        \caption{WFC - Laplace}
    \end{subfigure}
    \caption{QQ plots on training datasets for the GFC period across all stocks. When comparing the use of a Gaussian base distribution to a Laplace base distribution, we observe that the Gaussian model exhibits superior calibration, particularly in the central mass of the distribution. We hypothesize that this improved fit in the bulk region is attributable to return distributions more closely approximating Gaussian behavior during this period, which coincides with generally lower VIX (volatility index) levels. Another notable observation is that the Laplace base distribution tends to produce overdispersion in the tails, while the Gaussian base leads to underdispersion. This pattern aligns with theoretical expectations, as the Laplace distribution inherently has heavier tails than the Gaussian distribution, making it prone to overestimating tail probabilities when the true data-generating process is closer to Gaussian in nature. }
    \label{fig: qq_plots_gfc_training_base}
\end{figure}

\begin{figure}[htbp]
    \centering
    \begin{subfigure}[b]{0.24\textwidth}
        \includegraphics[width=\textwidth, trim=0cm 13cm 28cm 0cm, clip]{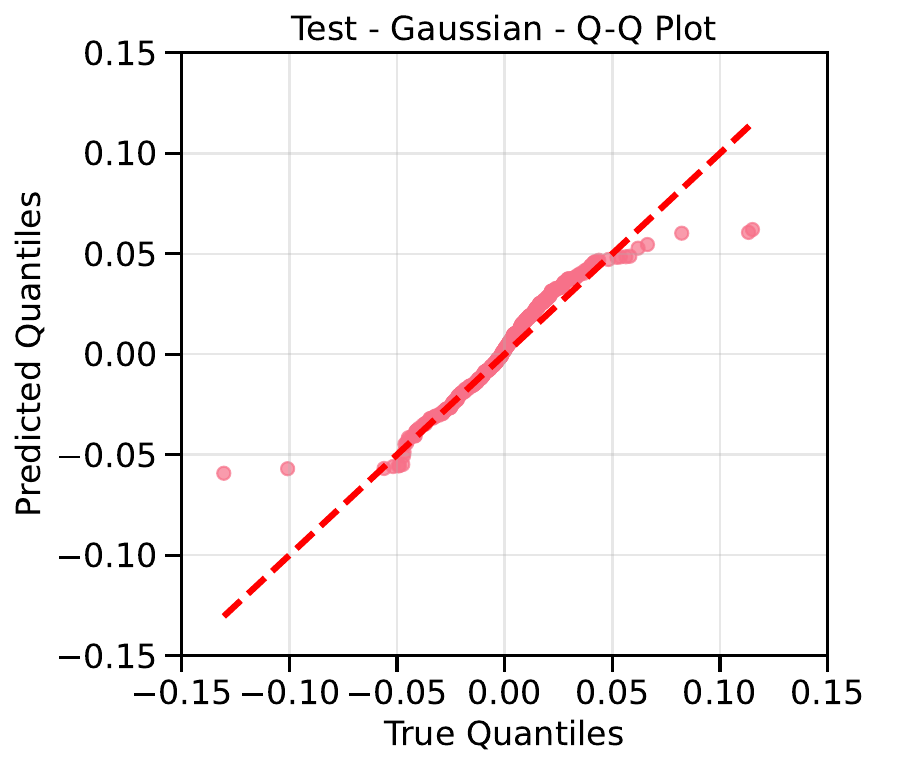}
        \caption{AAPL - Gaussian}
    \end{subfigure}
    \hfill
    \begin{subfigure}[b]{0.24\textwidth}
        \includegraphics[width=\textwidth, trim=0cm 13cm 28cm 0cm, clip]{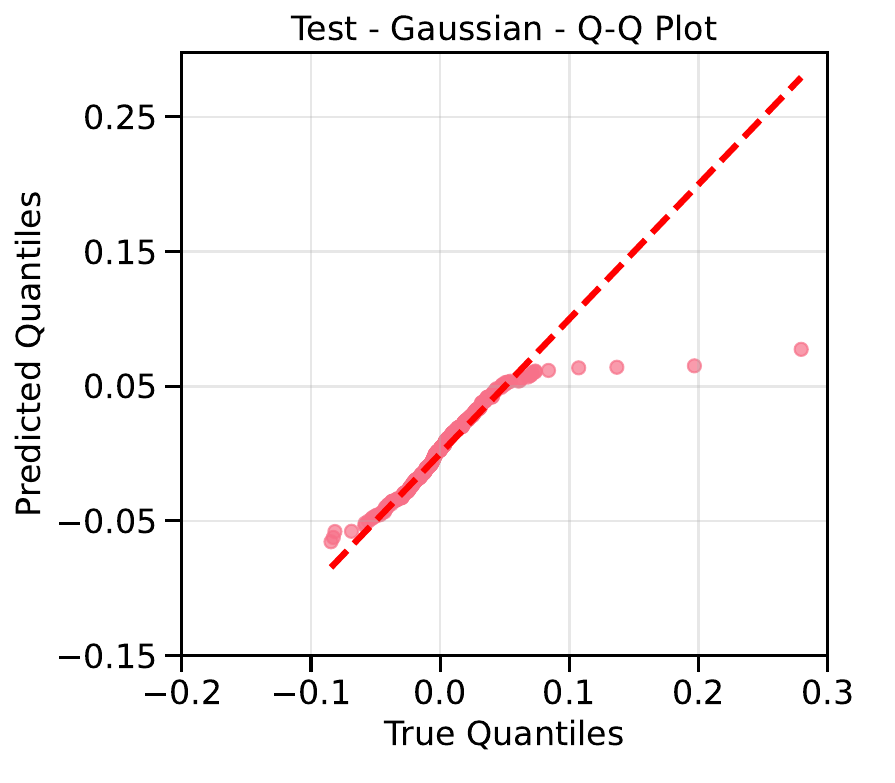}
        \caption{AMZN - Gaussian}
    \end{subfigure}
    \hfill
    \begin{subfigure}[b]{0.24\textwidth}
        \includegraphics[width=\textwidth, trim=0cm 13cm 28cm 0cm, clip]{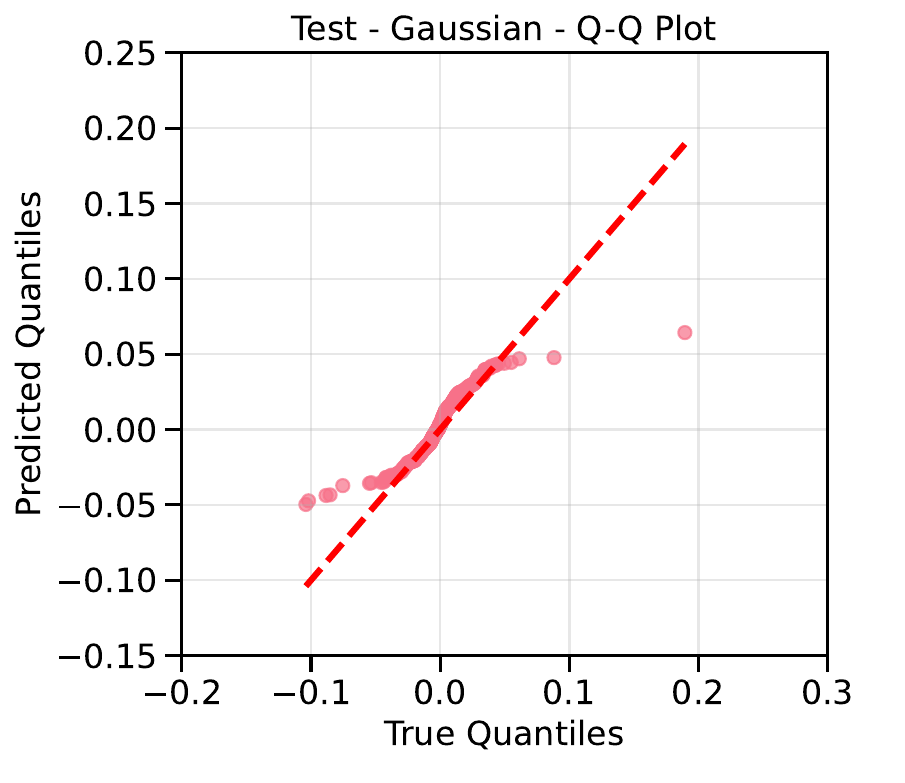}
        \caption{GOOGL - Gaussian.}
    \end{subfigure}
    \hfill
    \begin{subfigure}[b]{0.24\textwidth}
        \includegraphics[width=\textwidth, trim=0cm 13cm 28cm 0cm, clip]{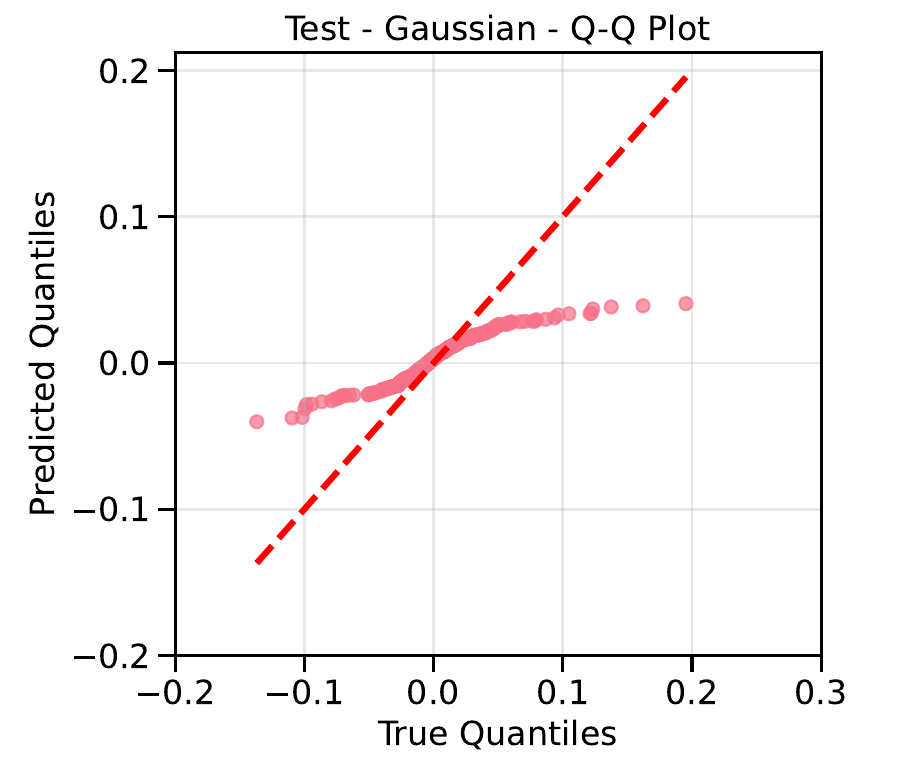}
        \caption{GS - Gaussian}
    \end{subfigure}
    \vspace{0.5em}
    \begin{subfigure}[b]{0.24\textwidth}
        \includegraphics[width=\textwidth, trim=0cm 13cm 28cm 0cm, clip]{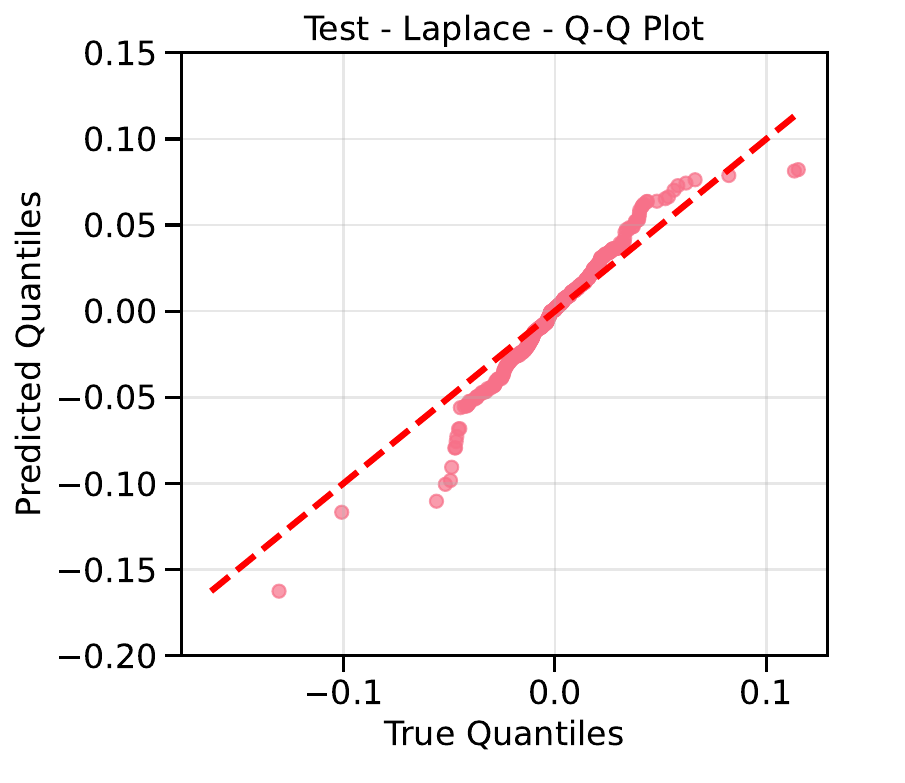}
        \caption{AAPL - Laplace}
    \end{subfigure}
    \hfill
    \begin{subfigure}[b]{0.24\textwidth}
        \includegraphics[width=\textwidth, trim=0cm 13cm 28cm 0cm, clip]{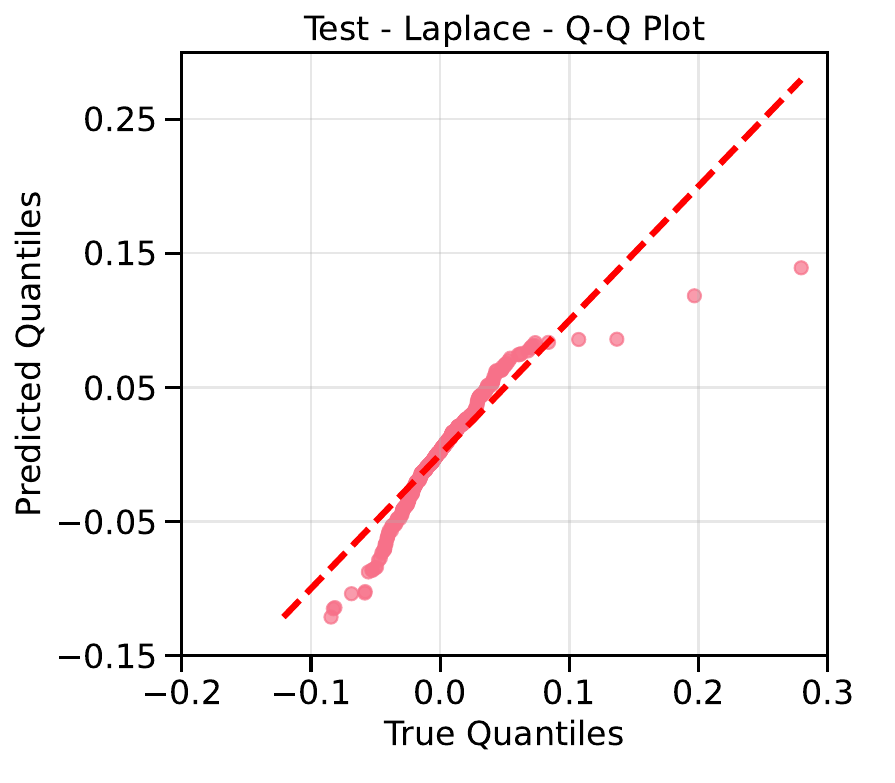}
        \caption{AMZN - Laplace}
    \end{subfigure}
    \hfill
    \begin{subfigure}[b]{0.24\textwidth}
        \includegraphics[width=\textwidth, trim=0cm 13cm 28cm 0cm, clip]{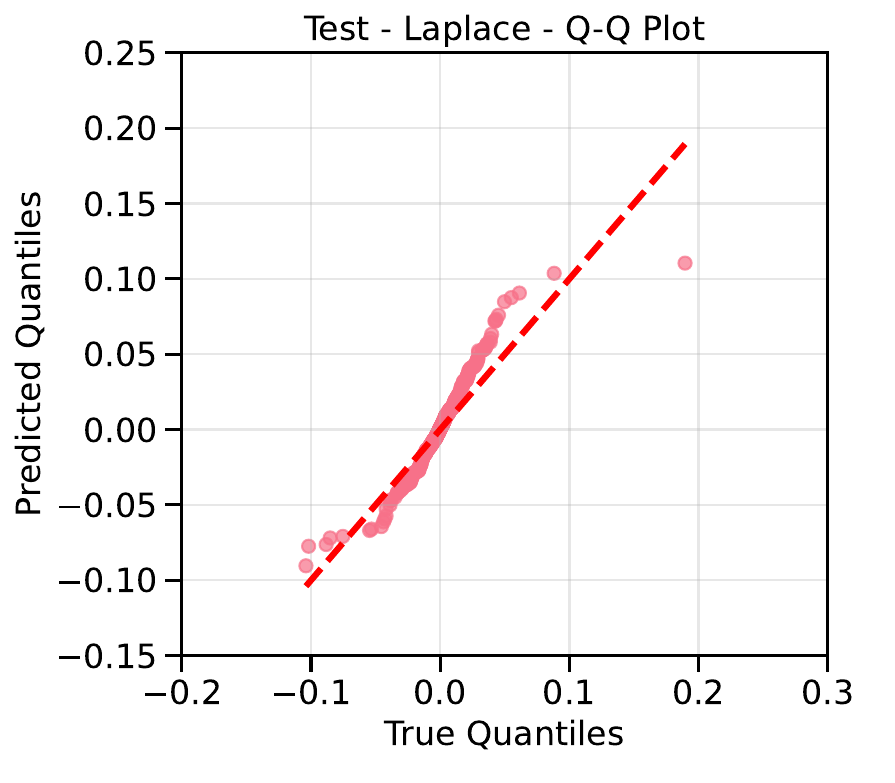}
        \caption{GOOGL - Laplace.}
    \end{subfigure}
    \hfill
    \begin{subfigure}[b]{0.24\textwidth}
        \includegraphics[width=\textwidth, trim=0cm 13cm 28cm 0cm, clip]{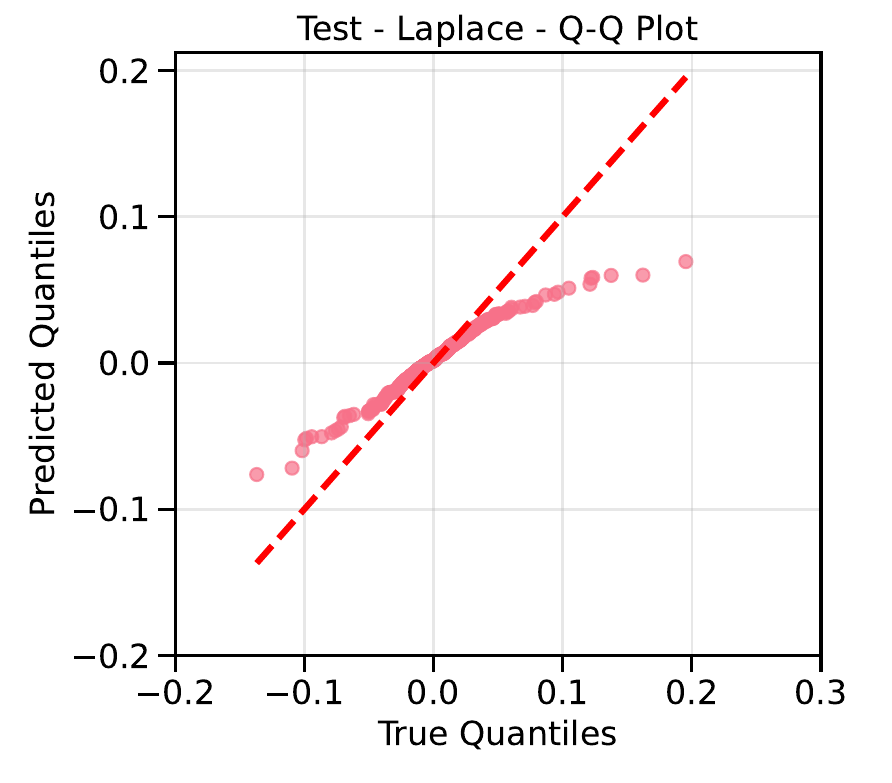}
        \caption{GS - Laplace}
    \end{subfigure}
    \vspace{0.5em}
    
    \begin{subfigure}[b]{0.24\textwidth}
        \includegraphics[width=\textwidth, trim=0cm 13cm 28cm 0cm, clip]{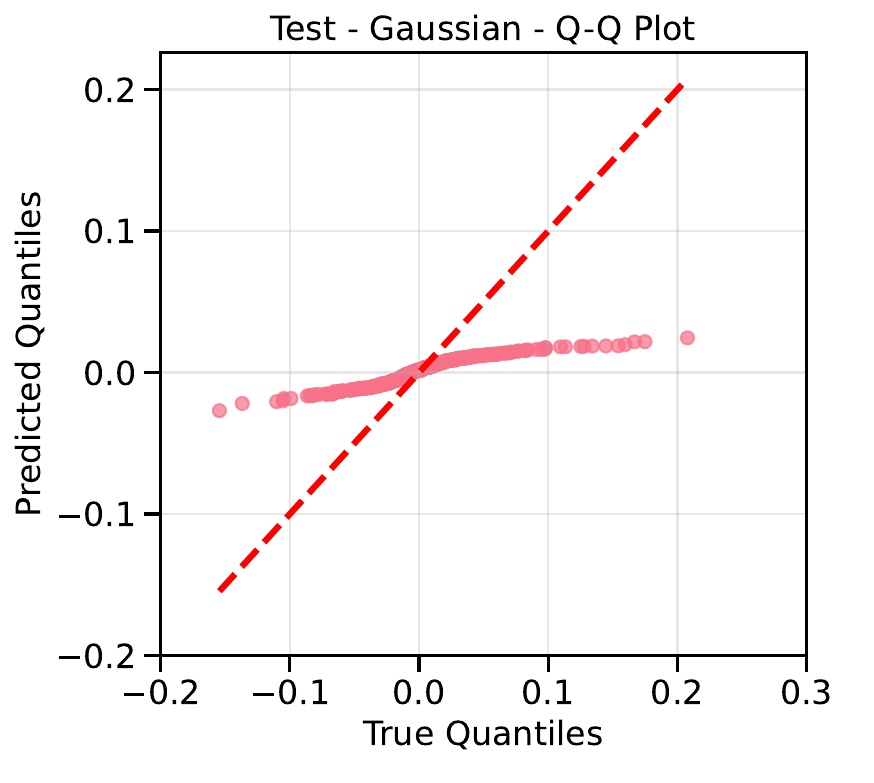}
        \caption{JPM - Gaussian}
    \end{subfigure}
    \hfill
    \begin{subfigure}[b]{0.24\textwidth}
        \includegraphics[width=\textwidth, trim=0cm 13cm 28cm 0cm, clip]{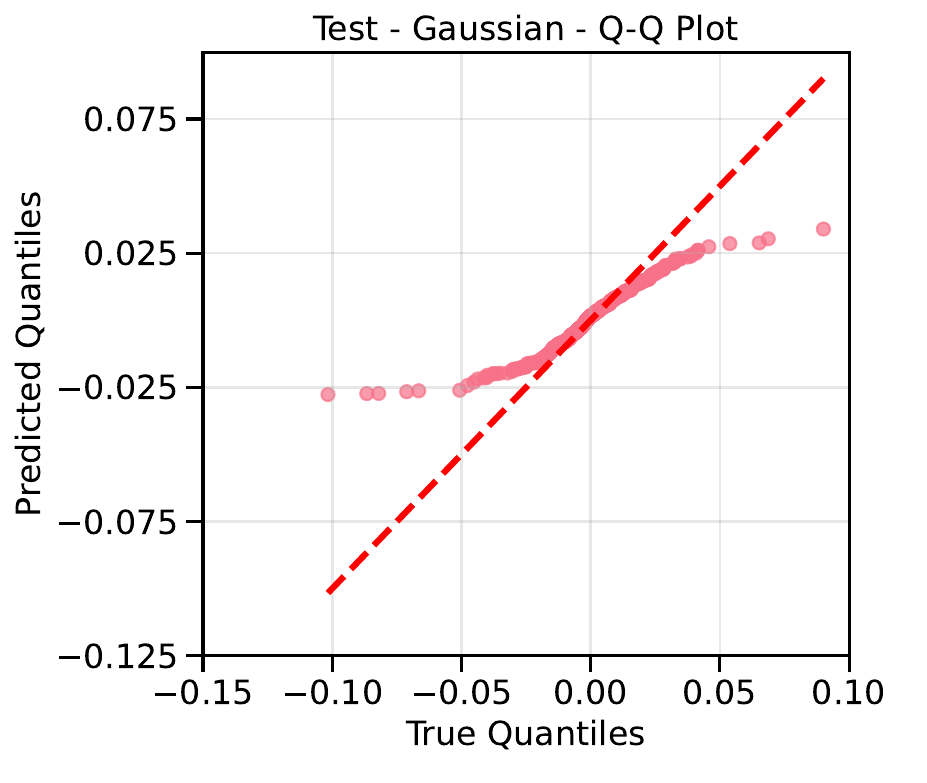}
        \caption{MSFT - Gaussian}
    \end{subfigure}
    \hfill
    \begin{subfigure}[b]{0.24\textwidth}
        \includegraphics[width=\textwidth, trim=0cm 13cm 28cm 0cm, clip]{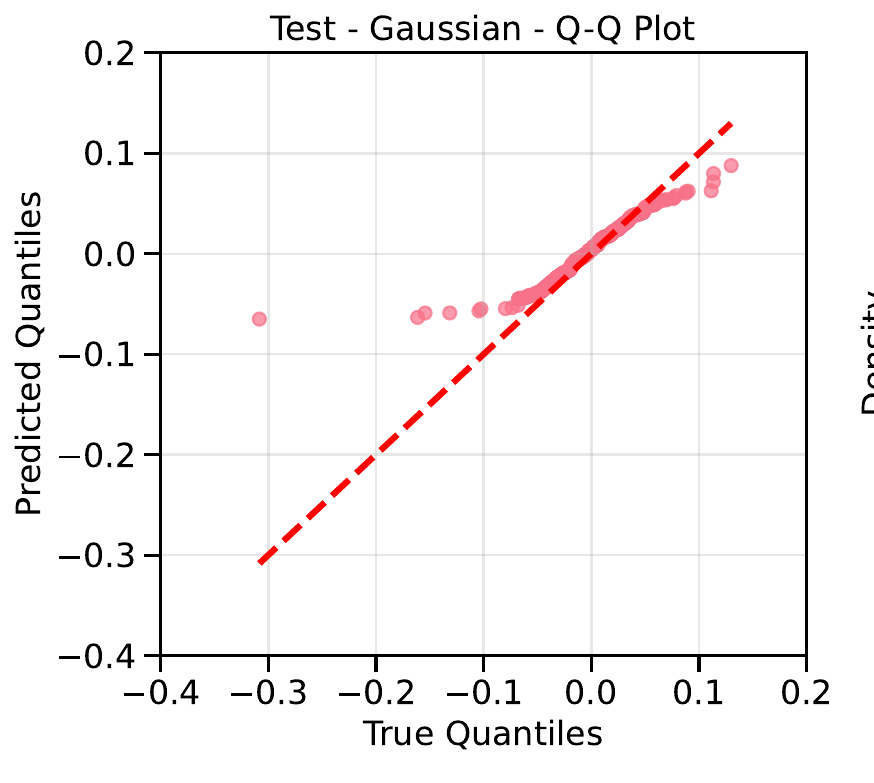}
        \caption{NVDA - Gaussian.}
    \end{subfigure}
    \hfill
    \begin{subfigure}[b]{0.24\textwidth}
        \includegraphics[width=\textwidth, trim=0cm 13cm 28cm 0cm, clip]{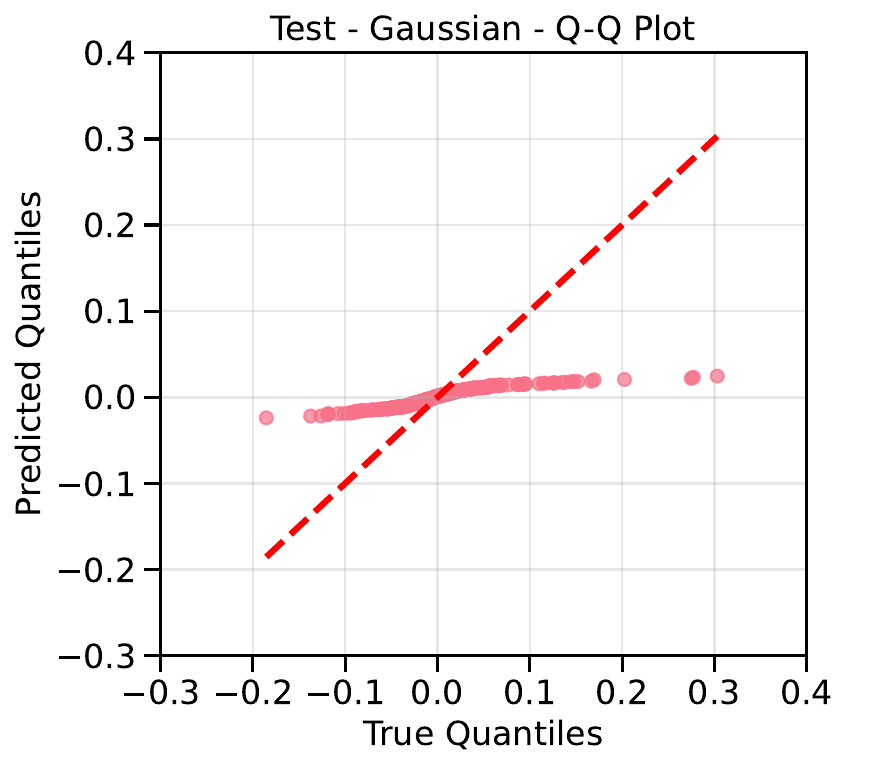}
        \caption{WFC - Gaussian}
    \end{subfigure}
    
    \vspace{0.5em}
    
    \begin{subfigure}[b]{0.24\textwidth}
        \includegraphics[width=\textwidth, trim=0cm 13cm 28cm 0cm, clip]{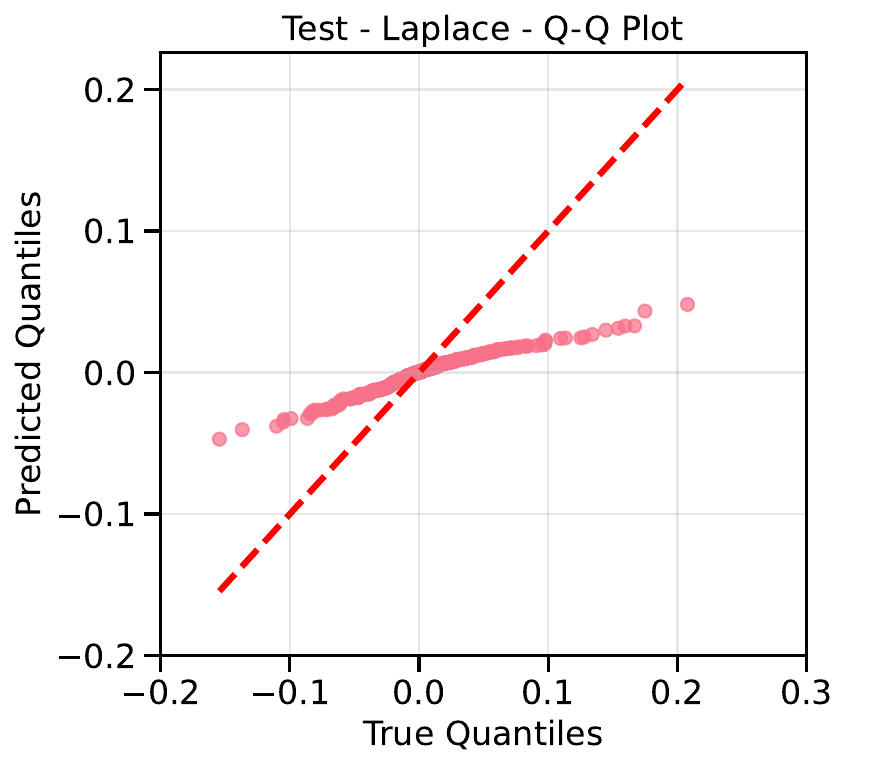}
        \caption{JPM - Laplace}
    \end{subfigure}
    \hfill
    \begin{subfigure}[b]{0.24\textwidth}
        \includegraphics[width=\textwidth, trim=0cm 13cm 28cm 0cm, clip]{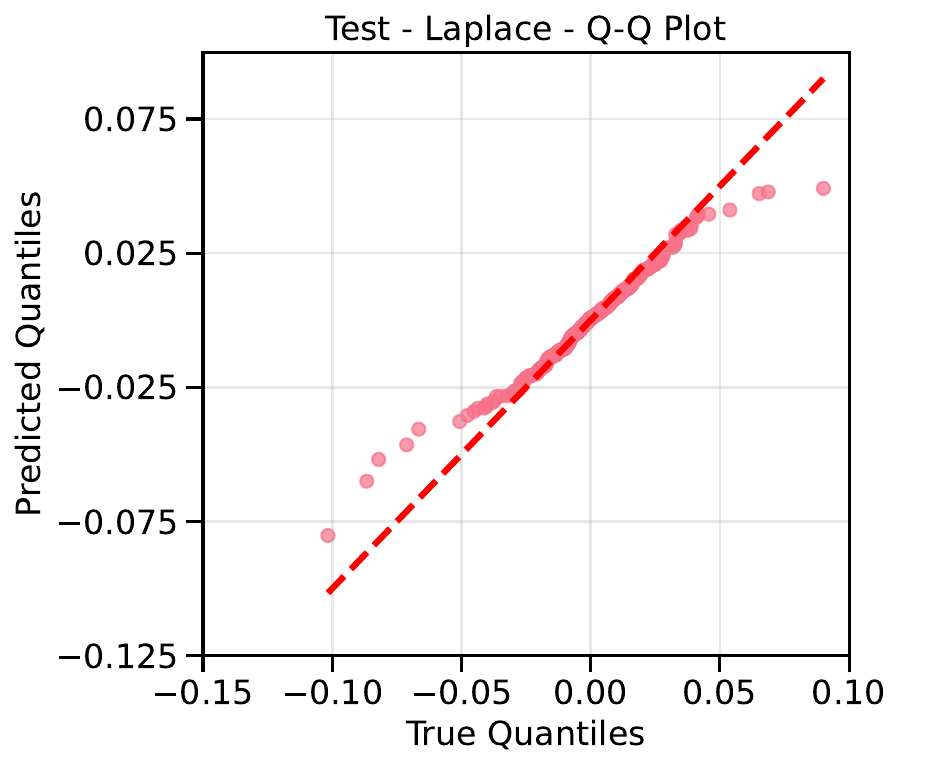}
        \caption{MSFT - Laplace}
    \end{subfigure}
    \hfill
    \begin{subfigure}[b]{0.24\textwidth}
        \includegraphics[width=\textwidth, trim=0cm 13cm 28cm 0cm, clip]{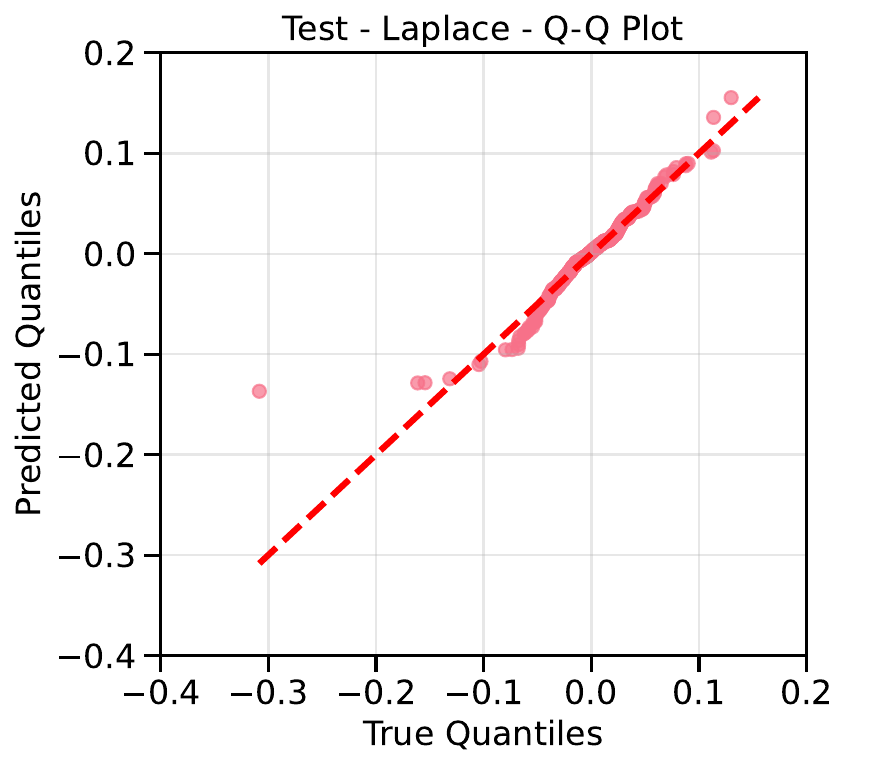}
        \caption{NVDA - Laplace.}
    \end{subfigure}
    \hfill
    \begin{subfigure}[b]{0.24\textwidth}
        \includegraphics[width=\textwidth, trim=0cm 13cm 28cm 0cm, clip]{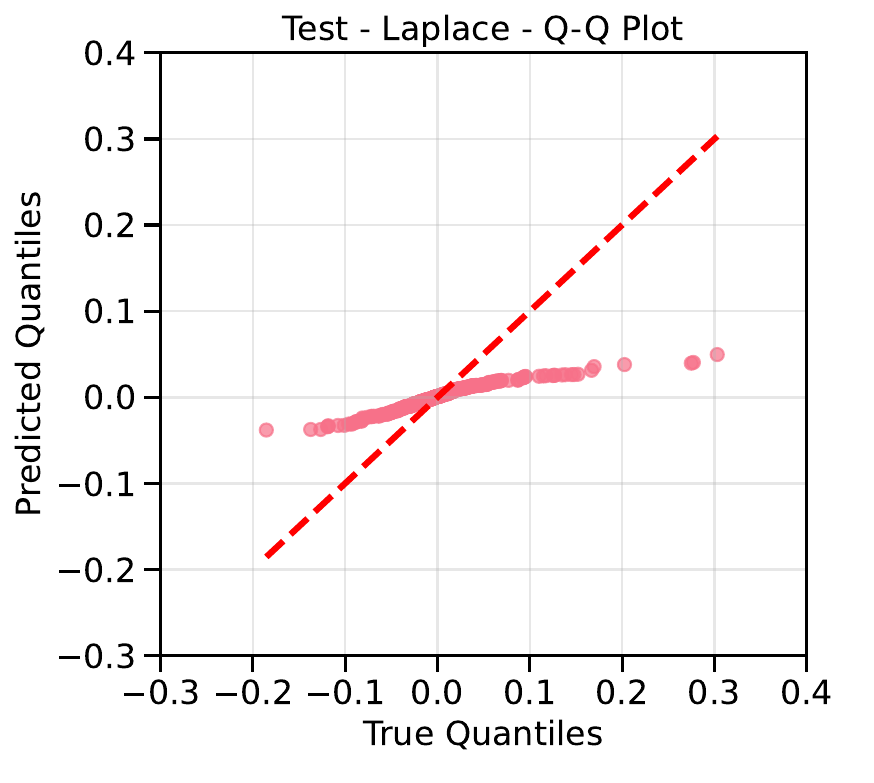}
        \caption{WFC - Laplace}
    \end{subfigure}
    \caption{QQ plots on testing datasets for the GFC period across all stocks. When comparing the use of a Gaussian base distribution to a Laplace base distribution, we observe that the Gaussian model significantly underestimates the tail heaviness of the target distribution (showing extreme underdispersion), while the Laplace distribution provides a much closer approximation to the true tail behavior, particularly in the extreme regions. This pattern is especially pronounced for technology stocks (AAPL, AMZN, GOOGL, NVDA). For financial sector stocks, both distributional models perform inadequately. This can be attributed to the disproportionate impact of the GFC on the financial sector, representing a more comprehensive distribution shift from the training data beyond a covariate shift in market volatility indicators like VIX. Nevertheless, across all stocks, we observe that samples from both base distributions in the conditional generative model exhibit underdispersion relative to the empirical data.}
    \label{fig: qq_plots_gfc_testing}
\end{figure}

\begin{figure}[htbp]
    \centering
    \begin{subfigure}[b]{0.24\textwidth}
        \includegraphics[width=\textwidth, trim=0cm 13cm 28cm 0cm, clip]{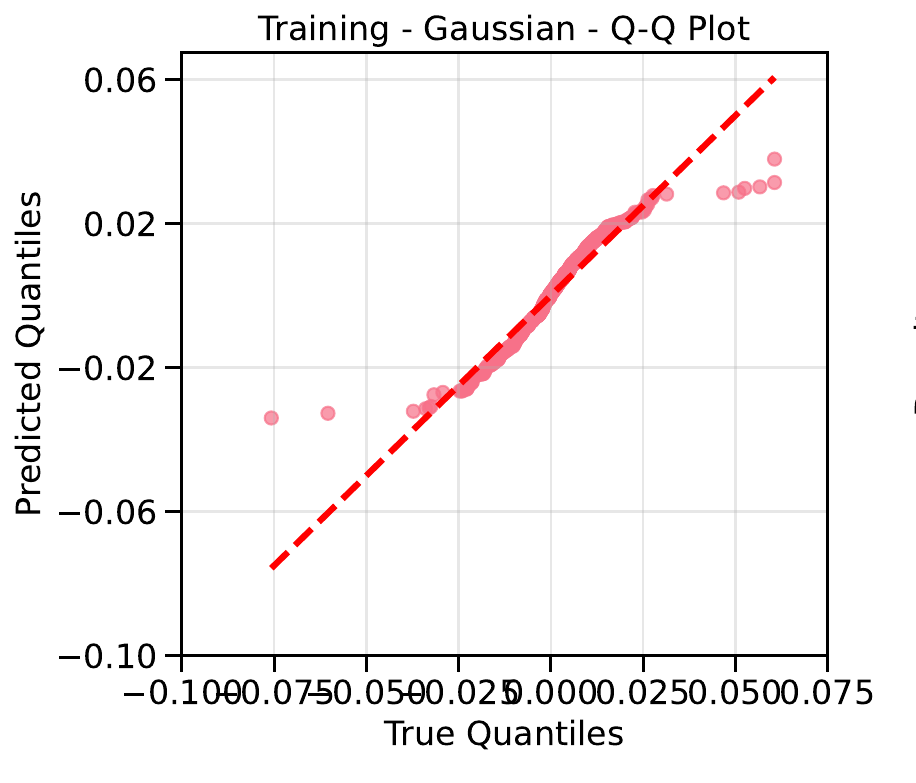}
        \caption{AAPL - Gaussian}
    \end{subfigure}
    \hfill
    \begin{subfigure}[b]{0.24\textwidth}
        \includegraphics[width=\textwidth, trim=0cm 13cm 28cm 0cm, clip]{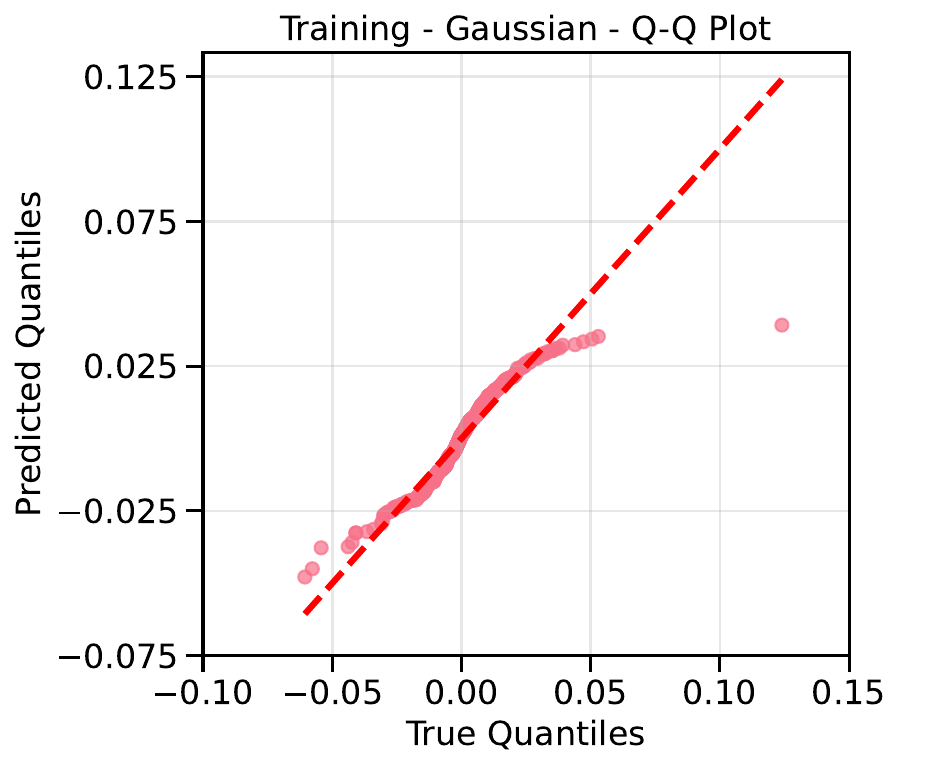}
        \caption{AMZN - Gaussian}
    \end{subfigure}
    \hfill
    \begin{subfigure}[b]{0.24\textwidth}
        \includegraphics[width=\textwidth, trim=0cm 13cm 28cm 0cm, clip]{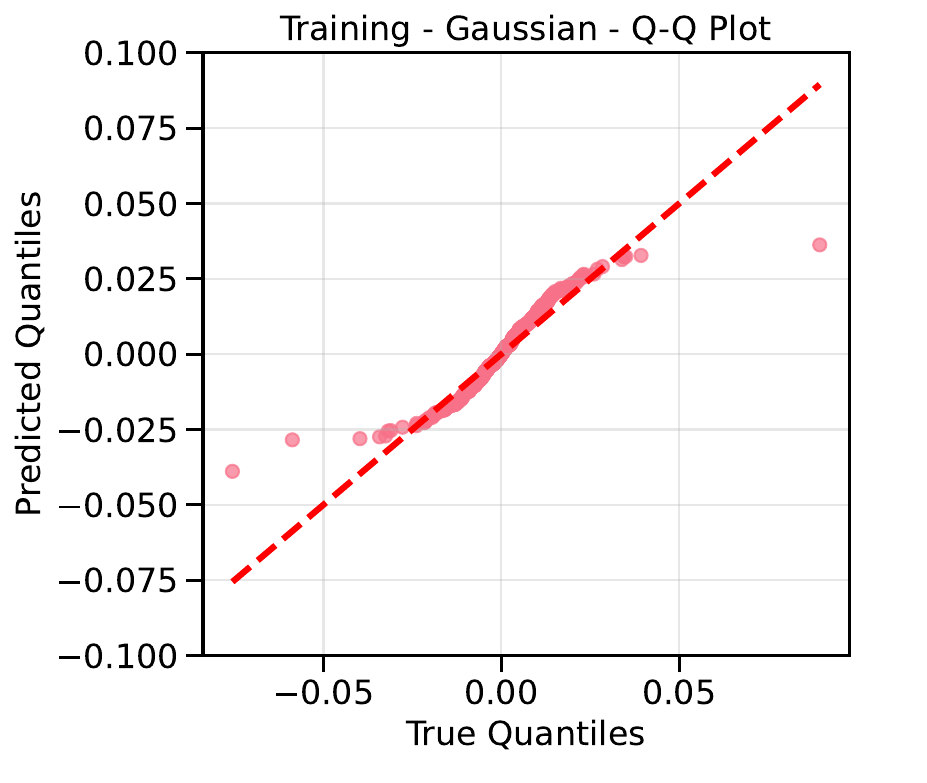}
        \caption{GOOGL - Gaussian.}
    \end{subfigure}
    \hfill
    \begin{subfigure}[b]{0.24\textwidth}
        \includegraphics[width=\textwidth, trim=0cm 13cm 28cm 0cm, clip]{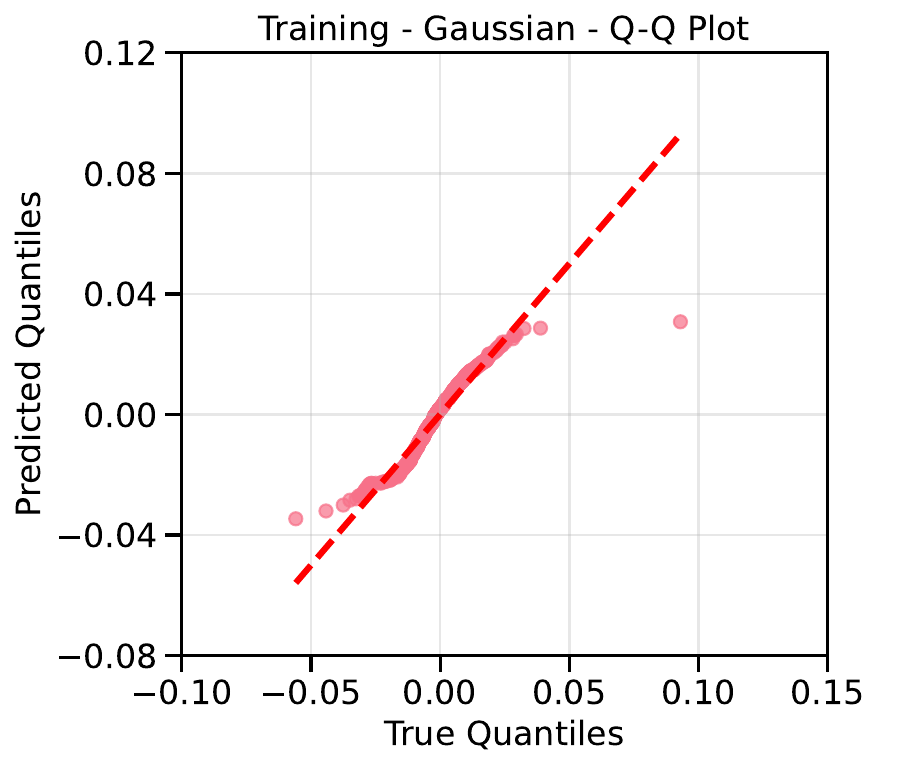}
        \caption{GS - Gaussian}
    \end{subfigure}
    \vspace{0.5em}
    \begin{subfigure}[b]{0.24\textwidth}
        \includegraphics[width=\textwidth, trim=0cm 13cm 28cm 0cm, clip]{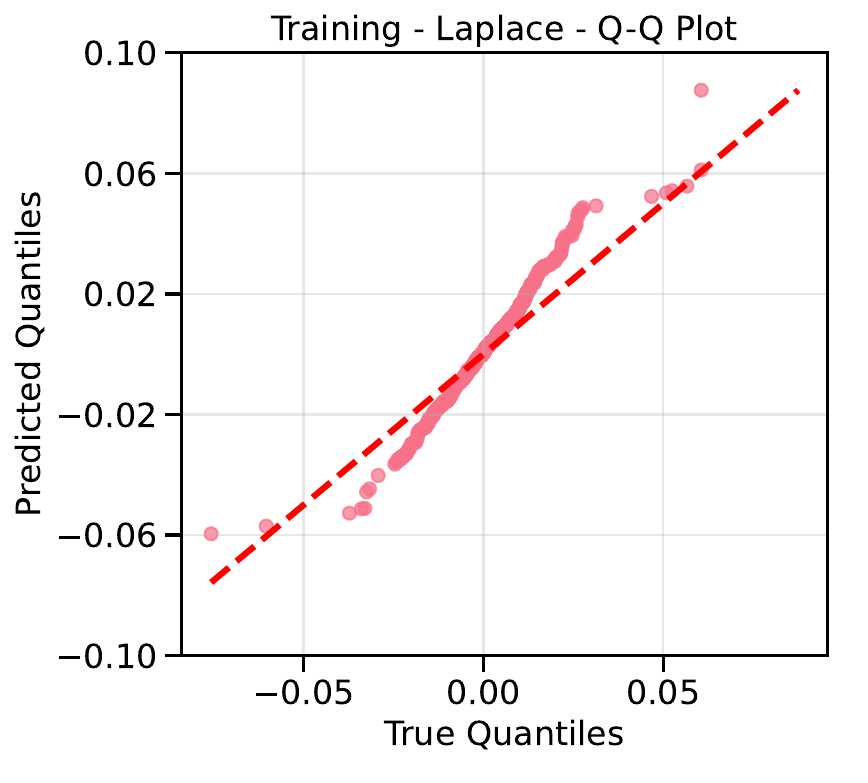}
        \caption{AAPL - Laplace}
    \end{subfigure}
    \hfill
    \begin{subfigure}[b]{0.24\textwidth}
        \includegraphics[width=\textwidth, trim=0cm 13cm 28cm 0cm, clip]{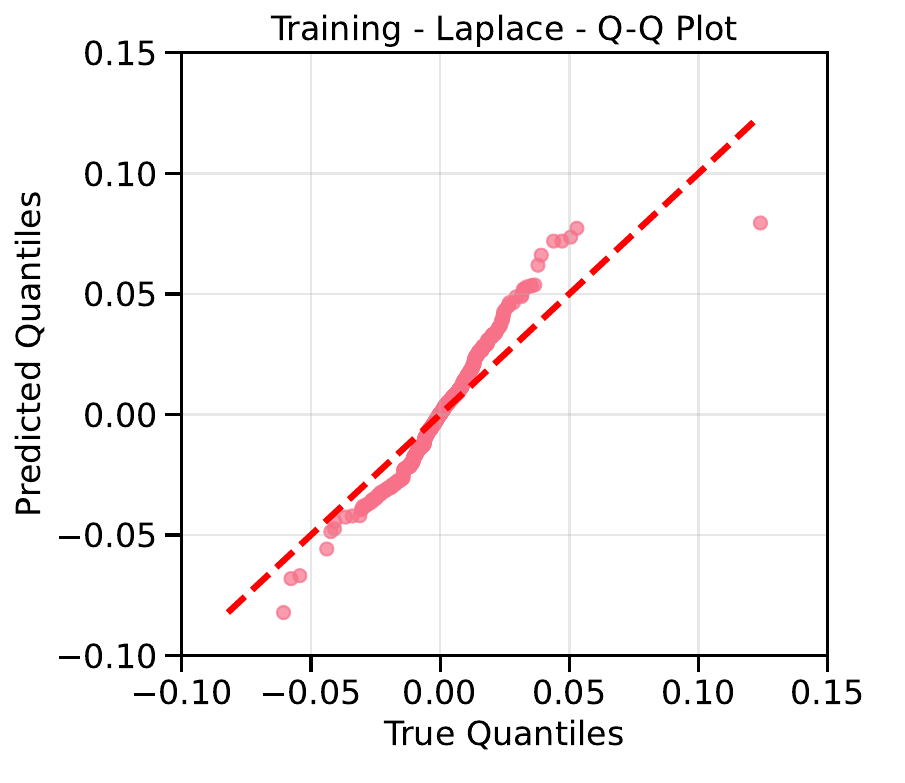}
        \caption{AMZN - Laplace}
    \end{subfigure}
    \hfill
    \begin{subfigure}[b]{0.24\textwidth}
        \includegraphics[width=\textwidth, trim=0cm 13cm 28cm 0cm, clip]{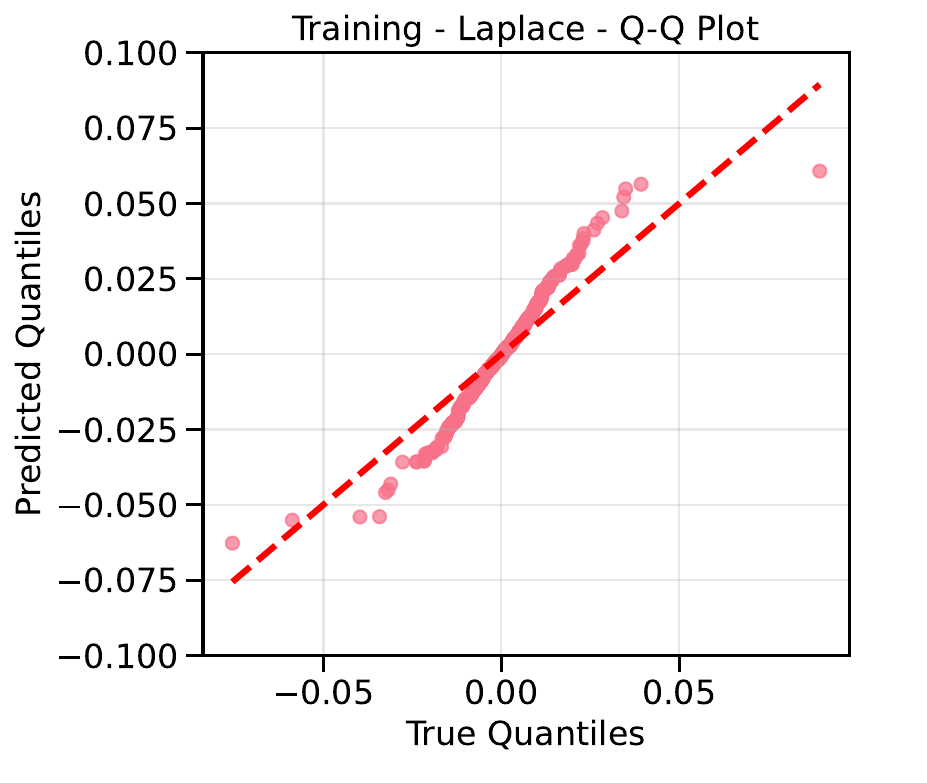}
        \caption{GOOGL - Laplace.}
    \end{subfigure}
    \hfill
    \begin{subfigure}[b]{0.24\textwidth}
        \includegraphics[width=\textwidth, trim=0cm 13cm 28cm 0cm, clip]{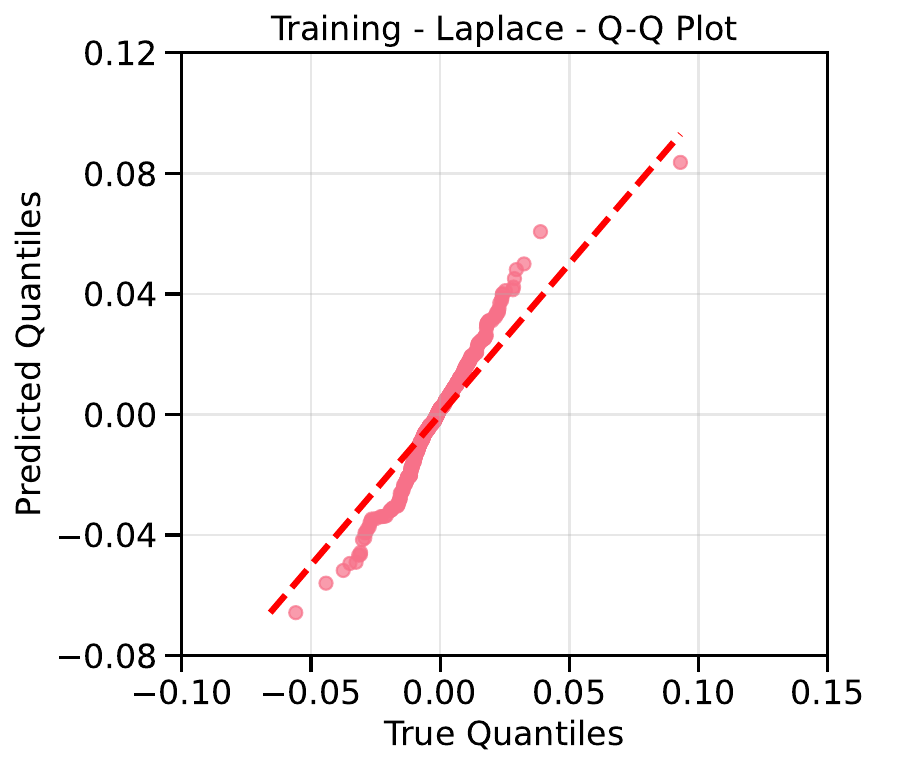}
        \caption{GS - Laplace}
    \end{subfigure}
    \vspace{0.5em}
    
    \begin{subfigure}[b]{0.24\textwidth}
        \includegraphics[width=\textwidth, trim=0cm 13cm 28cm 0cm, clip]{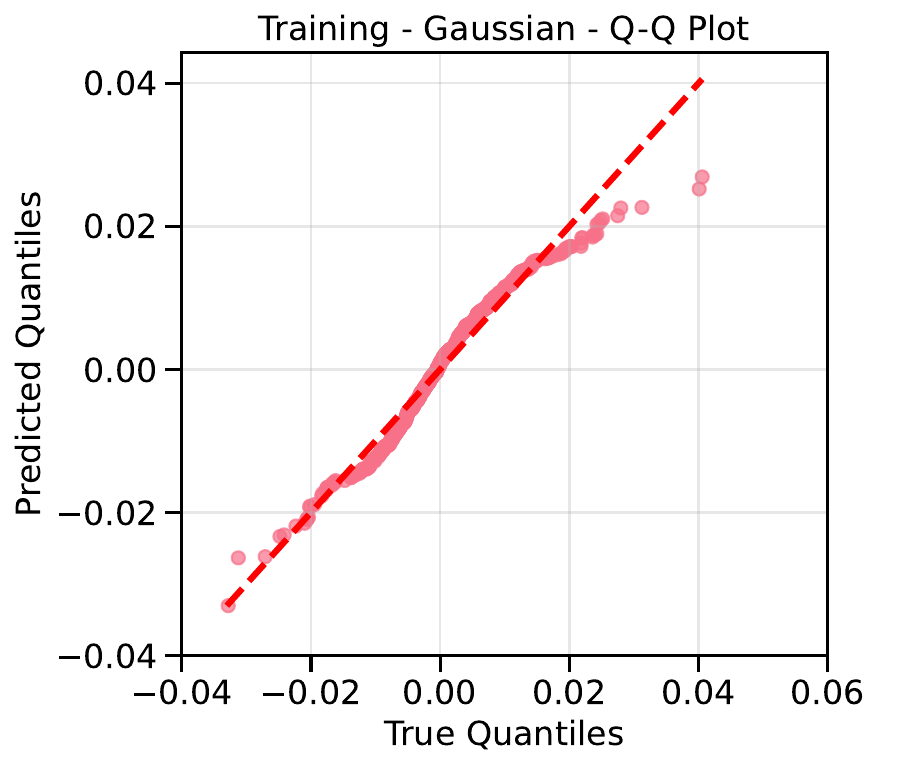}
        \caption{JPM - Gaussian}
    \end{subfigure}
    \hfill
    \begin{subfigure}[b]{0.24\textwidth}
        \includegraphics[width=\textwidth, trim=0cm 13cm 28cm 0cm, clip]{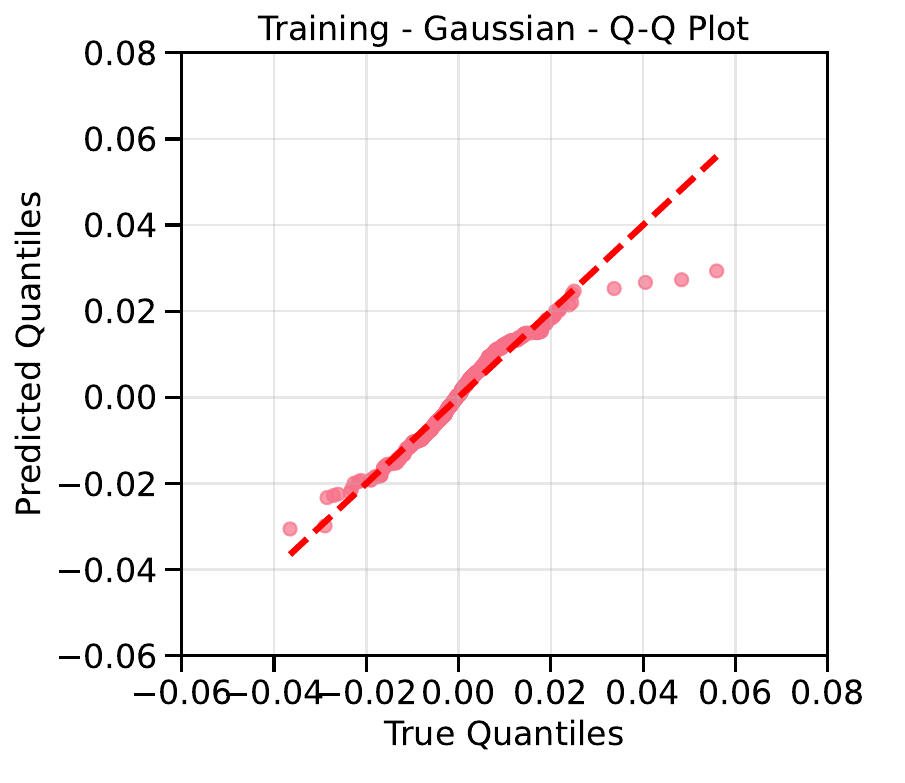}
        \caption{MSFT - Gaussian}
    \end{subfigure}
    \hfill
    \begin{subfigure}[b]{0.24\textwidth}
        \includegraphics[width=\textwidth, trim=0cm 13cm 28cm 0cm, clip]{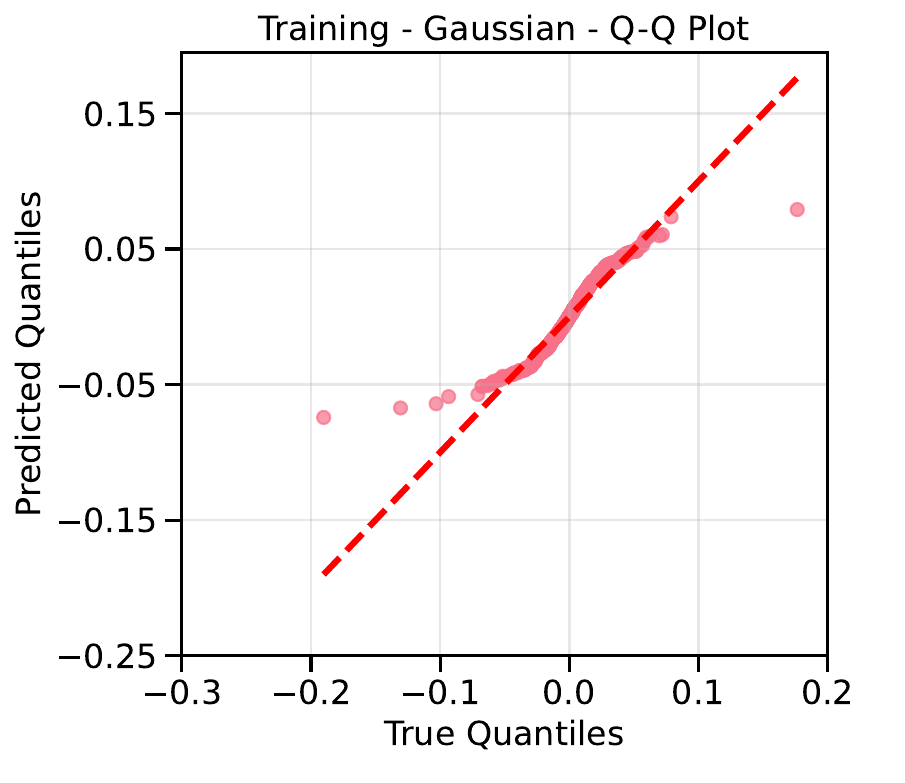}
        \caption{NVDA - Gaussian.}
    \end{subfigure}
    \hfill
    \begin{subfigure}[b]{0.24\textwidth}
        \includegraphics[width=\textwidth, trim=0cm 13cm 28cm 0cm, clip]{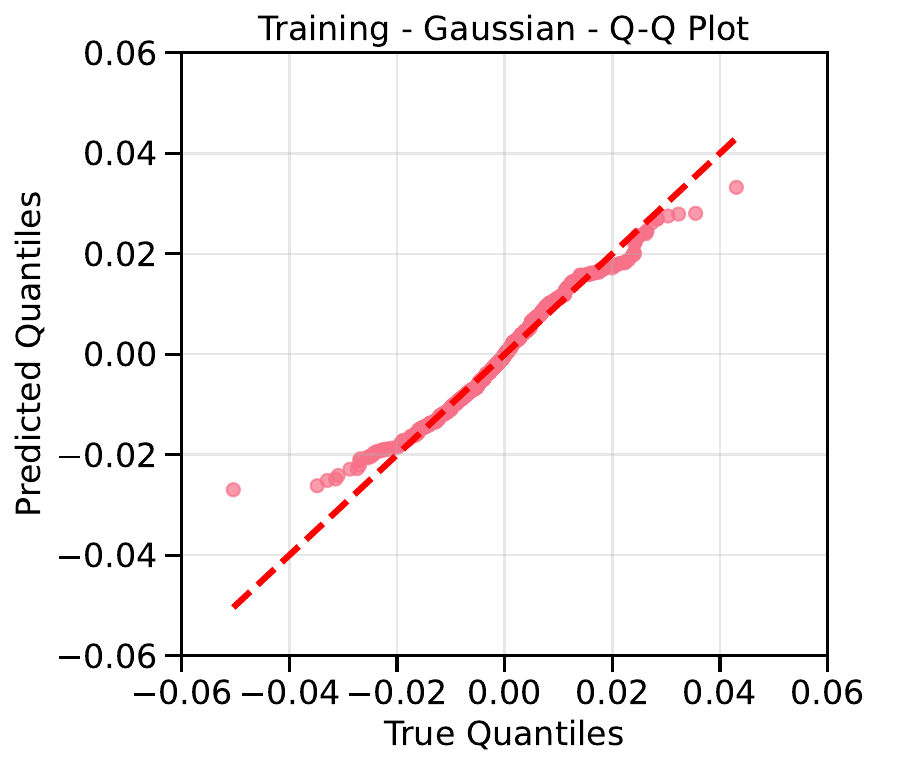}
        \caption{WFC - Gaussian}
    \end{subfigure}
    
    \vspace{0.5em}
    
    \begin{subfigure}[b]{0.24\textwidth}
        \includegraphics[width=\textwidth, trim=0cm 13cm 28cm 0cm, clip]{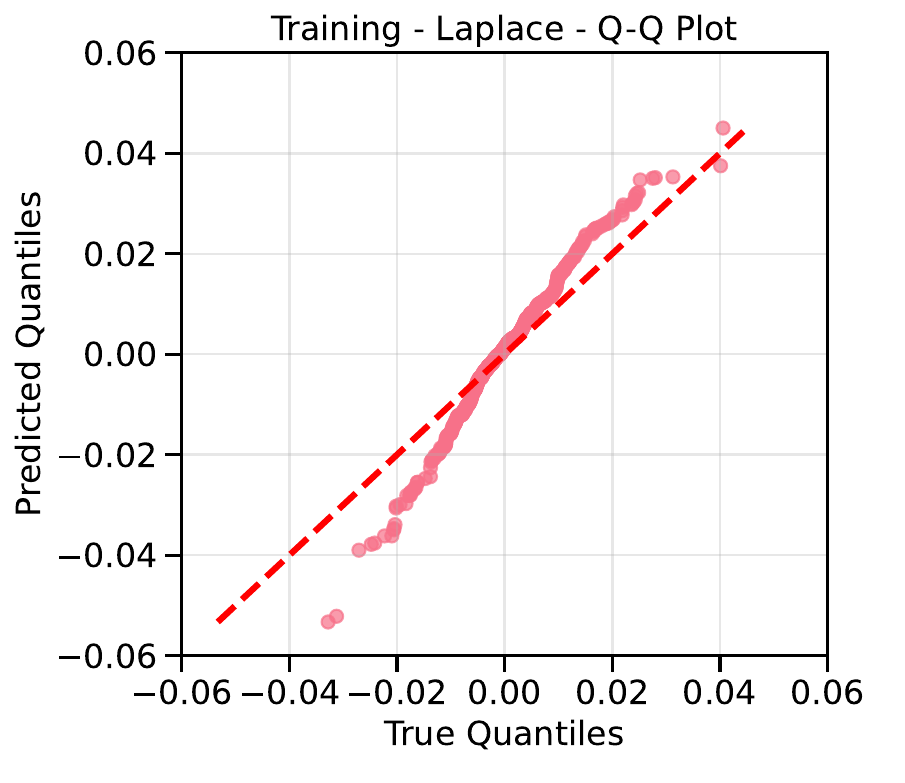}
        \caption{JPM - Laplace}
    \end{subfigure}
    \hfill
    \begin{subfigure}[b]{0.24\textwidth}
        \includegraphics[width=\textwidth, trim=0cm 13cm 28cm 0cm, clip]{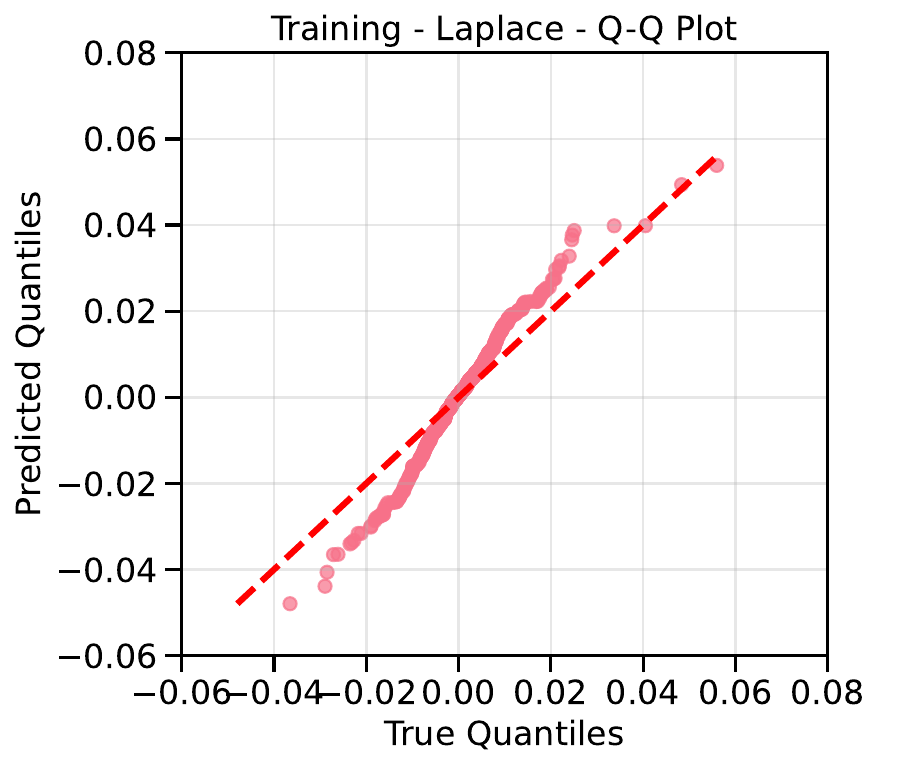}
        \caption{MSFT - Laplace}
    \end{subfigure}
    \hfill
    \begin{subfigure}[b]{0.24\textwidth}
        \includegraphics[width=\textwidth, trim=0cm 13cm 28cm 0cm, clip]{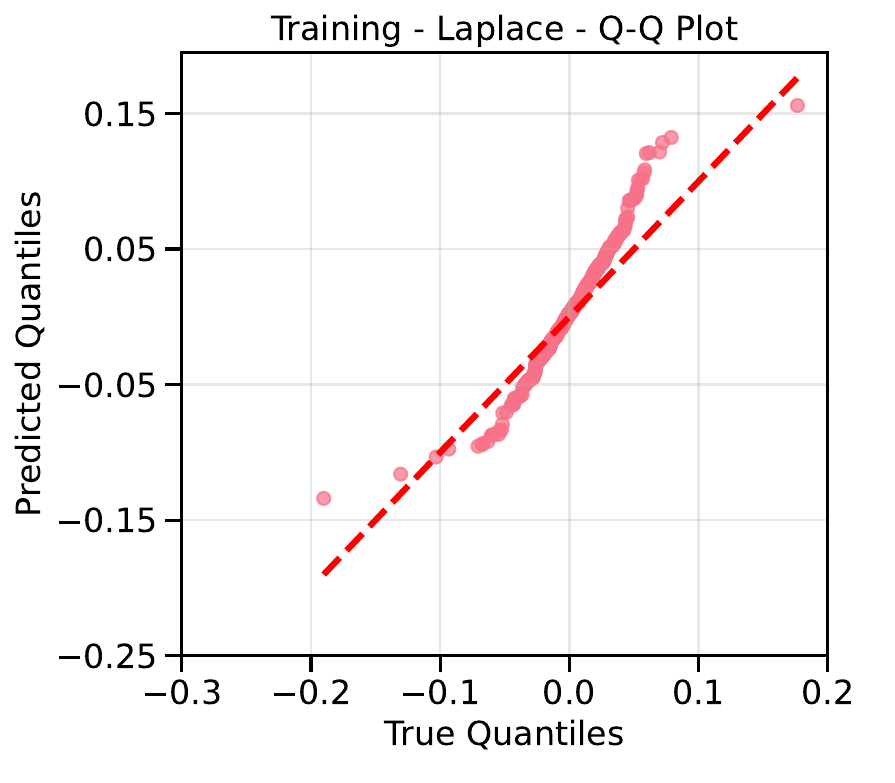}
        \caption{NVDA - Laplace.}
    \end{subfigure}
    \hfill
    \begin{subfigure}[b]{0.24\textwidth}
        \includegraphics[width=\textwidth, trim=0cm 13cm 28cm 0cm, clip]{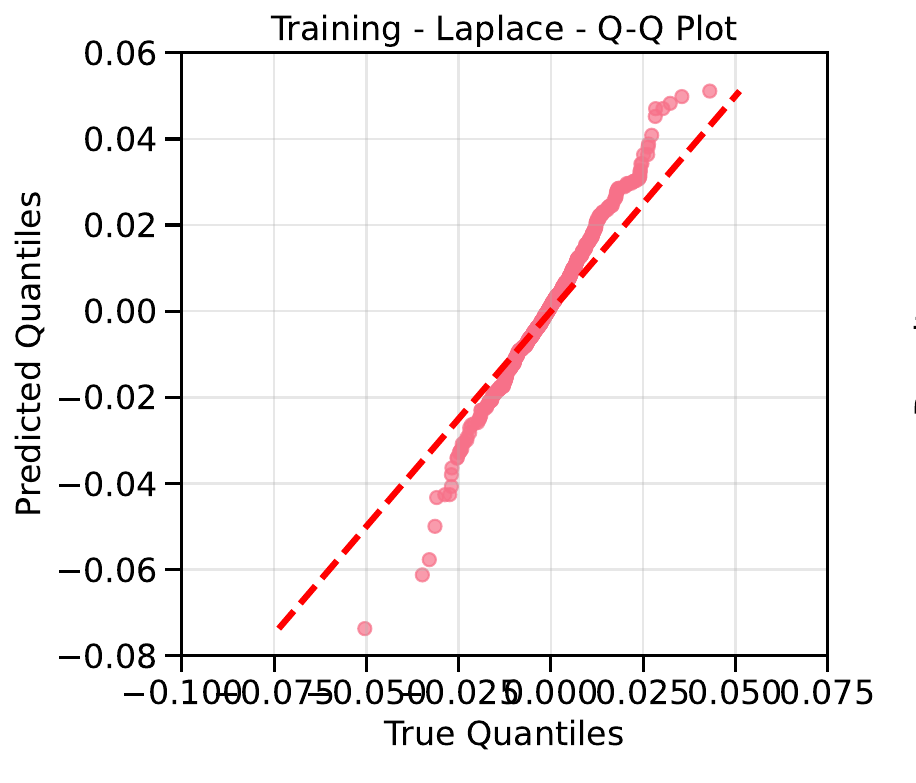}
        \caption{WFC - Laplace}
    \end{subfigure}
    \caption{QQ plots on training datasets for the COVID period across all stocks. When comparing the use of a Gaussian base distribution to a Laplace base distribution, we observe that the Gaussian model exhibits better calibration in the bulk of the distribution for most stocks, though with notable deviations in the extremes. Another notable observation is that the Laplace base distribution consistently produces overdispersion in the tails across multiple stocks, while the Gaussian base leads to underdispersion at the extremes; similar to the observation made for the GFC period.}
    \label{fig: qq_plots_covid_training_base}
\end{figure}

\begin{figure}[htbp]
    \centering
    \begin{subfigure}[b]{0.24\textwidth}
        \includegraphics[width=\textwidth, trim=0cm 13cm 28cm 0cm, clip]{figures/stock_data/covid_AAPL_run_1/gaussian_test_performance.pdf}
        \caption{AAPL - Gaussian}
    \end{subfigure}
    \hfill
    \begin{subfigure}[b]{0.24\textwidth}
        \includegraphics[width=\textwidth, trim=0cm 13cm 28cm 0cm, clip]{figures/stock_data/covid_AMZN_run_1/gaussian_test_performance.pdf}
        \caption{AMZN - Gaussian}
    \end{subfigure}
    \hfill
    \begin{subfigure}[b]{0.24\textwidth}
        \includegraphics[width=\textwidth, trim=0cm 13cm 28cm 0cm, clip]{figures/stock_data/covid_GOOGL_run_1/gaussian_test_performance.pdf}
        \caption{GOOGL - Gaussian.}
    \end{subfigure}
    \hfill
    \begin{subfigure}[b]{0.24\textwidth}
        \includegraphics[width=\textwidth, trim=0cm 13cm 28cm 0cm, clip]{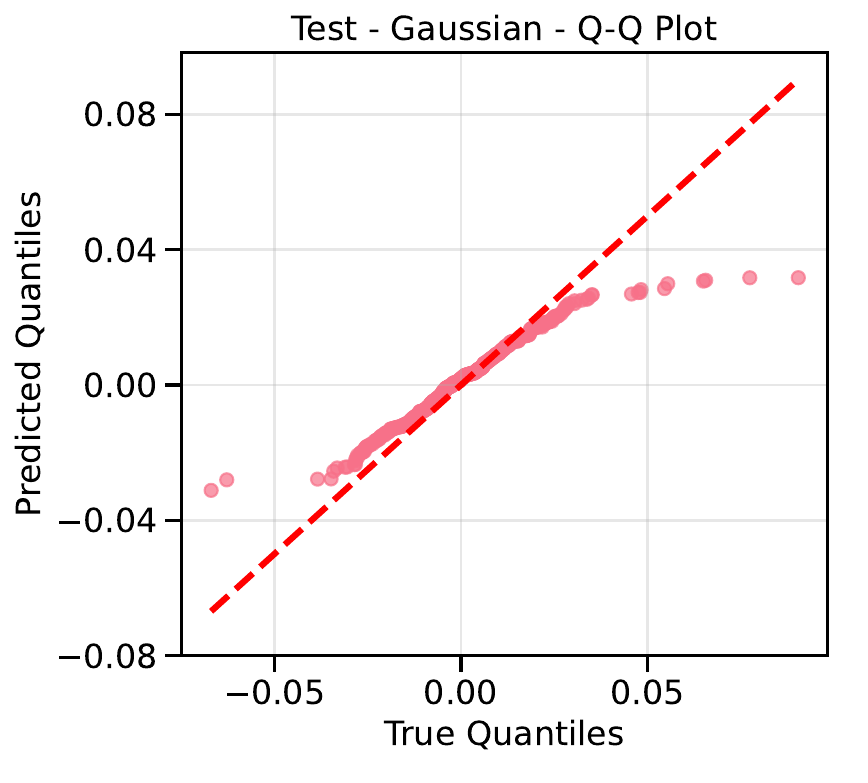}
        \caption{GS - Gaussian}
    \end{subfigure}
    \vspace{0.5em}
    \begin{subfigure}[b]{0.24\textwidth}
        \includegraphics[width=\textwidth, trim=0cm 13cm 28cm 0cm, clip]{figures/stock_data/covid_AAPL_run_1/laplace_test_performance.pdf}
        \caption{AAPL - Laplace}
    \end{subfigure}
    \hfill
    \begin{subfigure}[b]{0.24\textwidth}
        \includegraphics[width=\textwidth, trim=0cm 13cm 28cm 0cm, clip]{figures/stock_data/covid_AMZN_run_1/laplace_test_performance.pdf}
        \caption{AMZN - Laplace}
    \end{subfigure}
    \hfill
    \begin{subfigure}[b]{0.24\textwidth}
        \includegraphics[width=\textwidth, trim=0cm 13cm 28cm 0cm, clip]{figures/stock_data/covid_GOOGL_run_1/laplace_test_performance.pdf}
        \caption{GOOGL - Laplace.}
    \end{subfigure}
    \hfill
    \begin{subfigure}[b]{0.24\textwidth}
        \includegraphics[width=\textwidth, trim=0cm 13cm 28cm 0cm, clip]{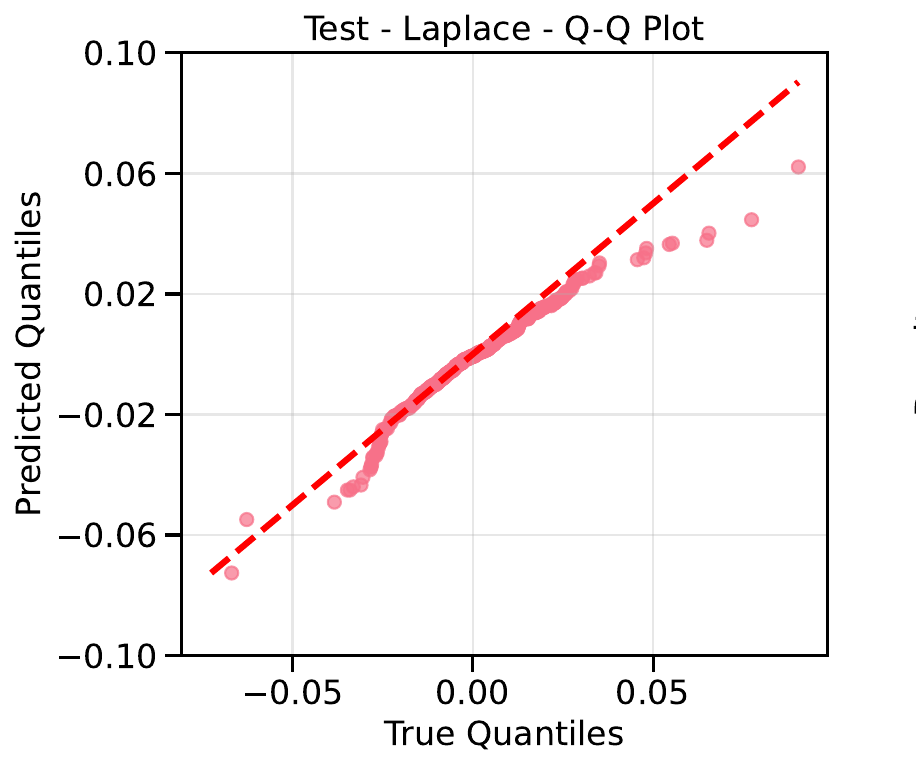}
        \caption{GS - Laplace}
    \end{subfigure}
    \vspace{0.5em}
    
    \begin{subfigure}[b]{0.24\textwidth}
        \includegraphics[width=\textwidth, trim=0cm 13cm 28cm 0cm, clip]{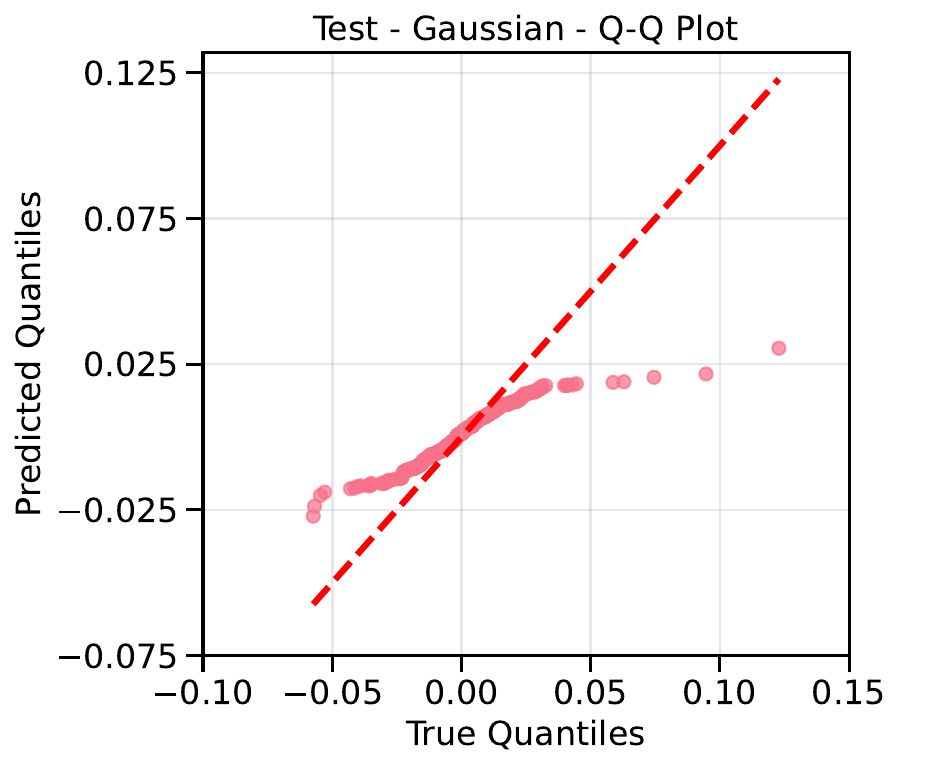}
        \caption{JPM - Gaussian}
    \end{subfigure}
    \hfill
    \begin{subfigure}[b]{0.24\textwidth}
        \includegraphics[width=\textwidth, trim=0cm 13cm 28cm 0cm, clip]{figures/stock_data/covid_MSFT_run_1/gaussian_test_performance.pdf}
        \caption{MSFT - Gaussian}
    \end{subfigure}
    \hfill
    \begin{subfigure}[b]{0.24\textwidth}
        \includegraphics[width=\textwidth, trim=0cm 13cm 28cm 0cm, clip]{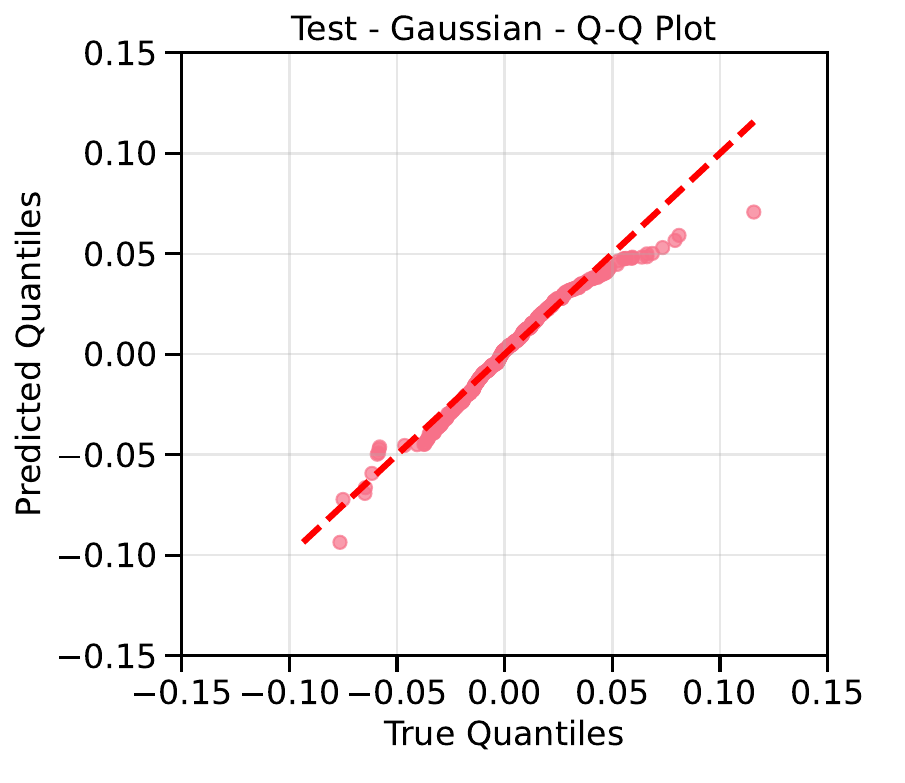}
        \caption{NVDA - Gaussian.}
    \end{subfigure}
    \hfill
    \begin{subfigure}[b]{0.24\textwidth}
        \includegraphics[width=\textwidth, trim=0cm 13cm 28cm 0cm, clip]{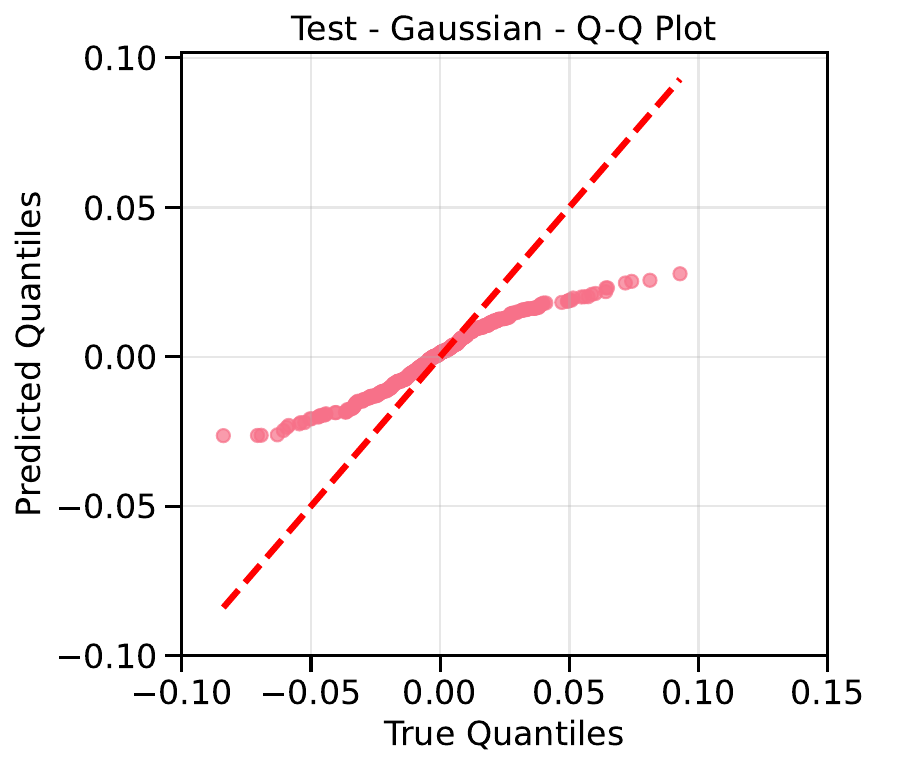}
        \caption{WFC - Gaussian}
    \end{subfigure}
    
    \vspace{0.5em}
    
    \begin{subfigure}[b]{0.24\textwidth}
        \includegraphics[width=\textwidth, trim=0cm 13cm 28cm 0cm, clip]{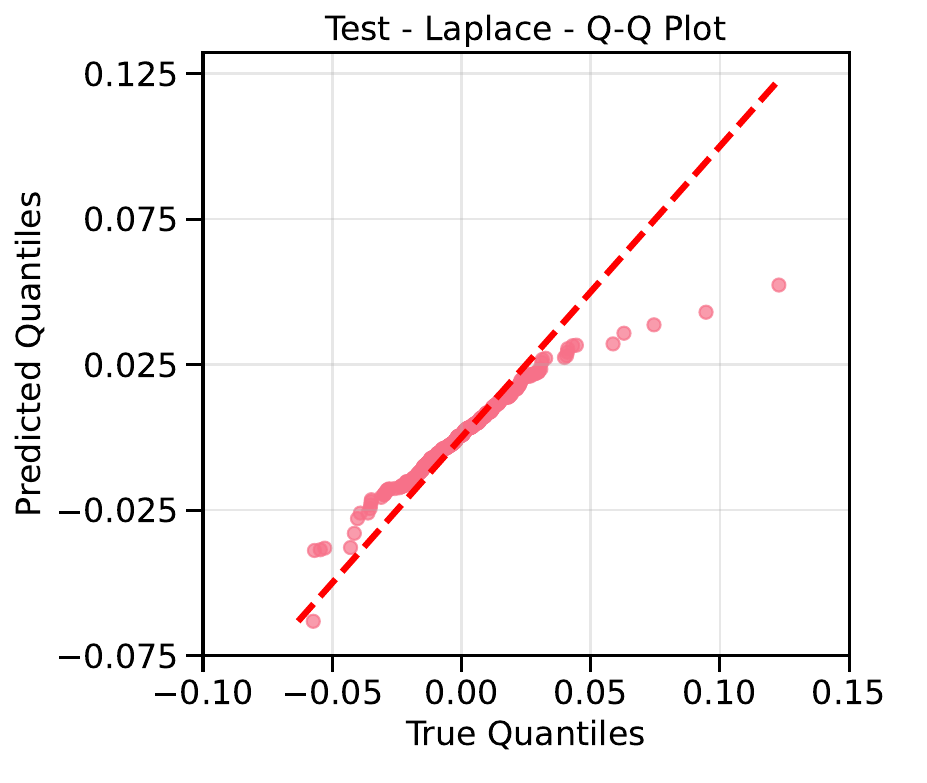}
        \caption{JPM - Laplace}
    \end{subfigure}
    \hfill
    \begin{subfigure}[b]{0.24\textwidth}
        \includegraphics[width=\textwidth, trim=0cm 13cm 28cm 0cm, clip]{figures/stock_data/covid_MSFT_run_1/laplace_test_performance.pdf}
        \caption{MSFT - Laplace}
    \end{subfigure}
    \hfill
    \begin{subfigure}[b]{0.24\textwidth}
        \includegraphics[width=\textwidth, trim=0cm 13cm 28cm 0cm, clip]{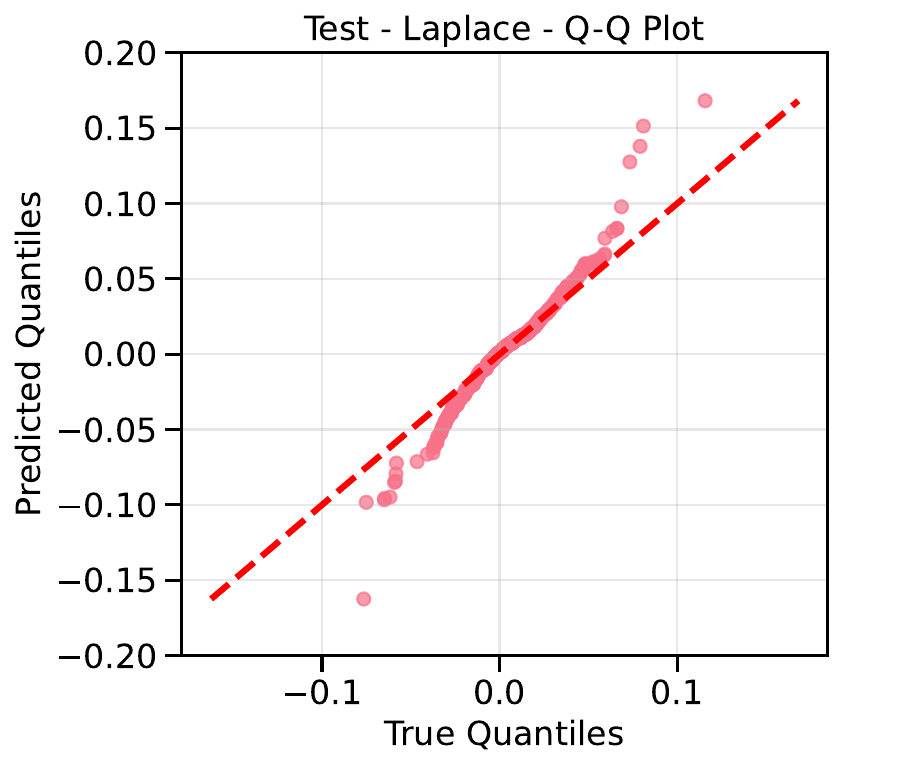}
        \caption{NVDA - Laplace.}
    \end{subfigure}
    \hfill
    \begin{subfigure}[b]{0.24\textwidth}
        \includegraphics[width=\textwidth, trim=0cm 13cm 28cm 0cm, clip]{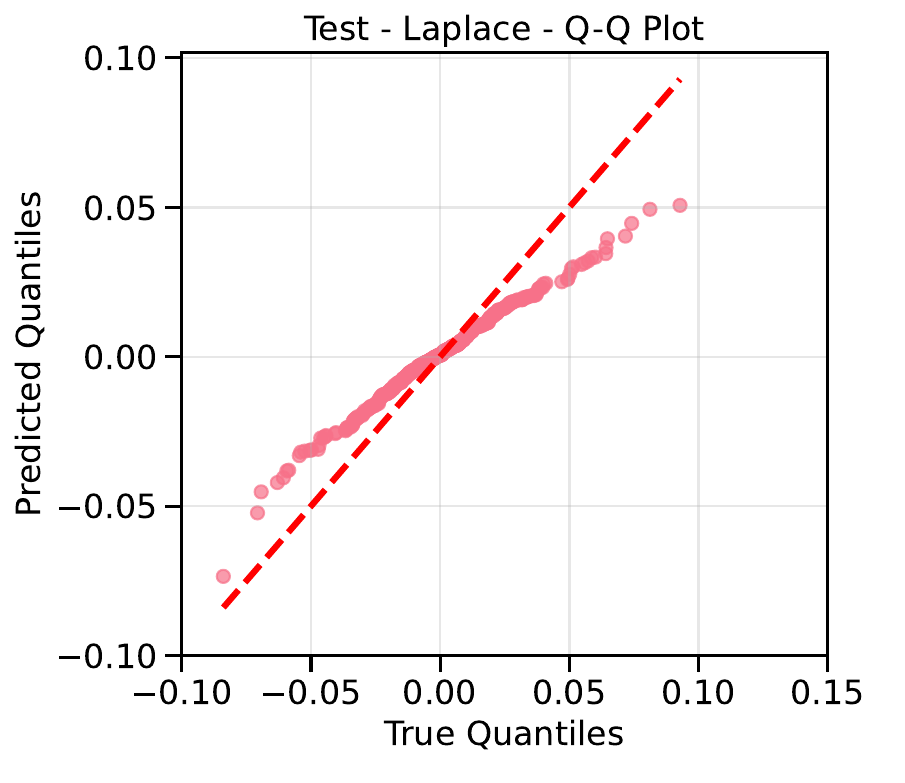}
        \caption{WFC - Laplace}
    \end{subfigure}
    \caption{QQ plots on testing datasets for the COVID period across all stocks. When comparing the use of a Gaussian base distribution to a Laplace base distribution, we observe that the Gaussian model significantly underestimates the tail behavior, particularly evident in technology stocks like JPM (i), MSFT (j), and WFC (l) where predicted quantiles fall below the diagonal reference line at extremes. We hypothesize that this underdispersion reflects the Gaussian distribution's inability to capture the heightened market volatility characteristic of the COVID crisis period. Another notable observation is that the Laplace base distribution provides a markedly better fit to the tail behavior for most stocks, especially visible in AAPL (e), GOOGL (g), and JPM (m), though it still exhibits some deviations from perfect calibration. This pattern aligns with theoretical expectations, as the COVID period featured extreme market movements that are better approximated by distributions with heavier tails, making the Laplace distribution's inherent properties more suitable for modeling the fat-tailed nature of returns during this market stress periods.}
    \label{fig: qq_plots_covid_testing}
\end{figure}

\subsubsection{Scatter Plots of Asset Returns vs. VIX Level (Conditional Evaluation)}
The use of QQ plots makes sense for evaluation of the calibration of the marginal distribution of returns (where we generate samples considering all conditions in the ground truth training and testing datasets); however, it does not provide insight into the performance of the conditional, as we vary the condition to extreme values. In the case of VIX, we are interested in the right-tail of the condition; when the VIX level grows to large positive values (around 40-80).  To evaluate the conditional performance, we use a scatter plot of the returns and the VIX level, and compare that the empirical percentiles of the conditional diffusion model for both the Gaussian and Laplace base distributions. We show these scatter plots for each ticker in Figure \ref{fig: conditional_scatter_gfc} and \ref{fig: conditional_scatter_covid} for the GFC and COVID periods, respectively. For both periods, we can observe that the use of Gaussian base leads to underestimation of the tails across almost all conditions, while the use of a Laplace is much closer.

\begin{figure}[htbp]
    \centering
    \begin{subfigure}[b]{0.48\textwidth}
        \includegraphics[width=\textwidth]{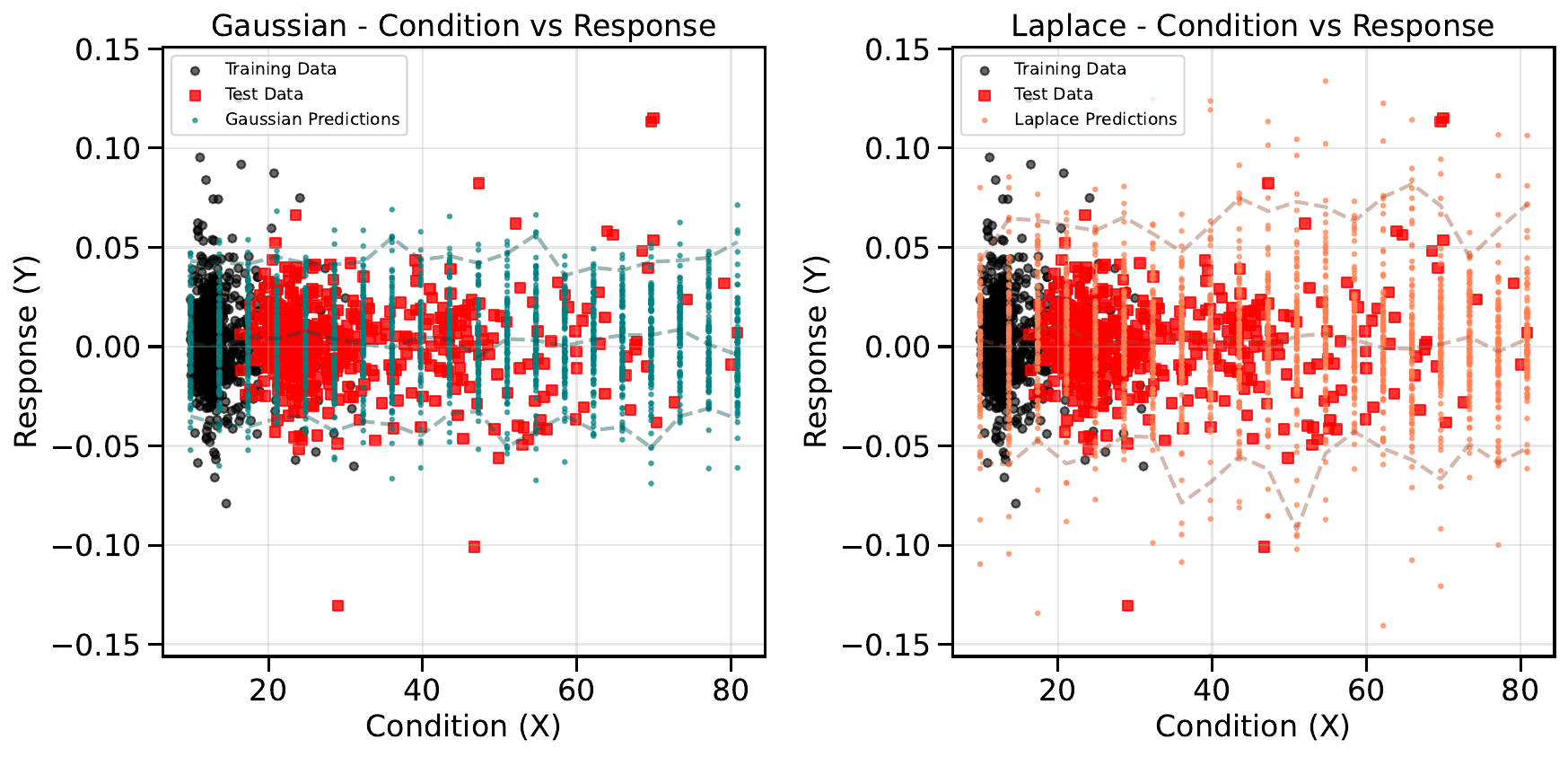}
        \caption{AAPL}
    \end{subfigure}
    \hfill
    \begin{subfigure}[b]{0.48\textwidth}
        \includegraphics[width=\textwidth]{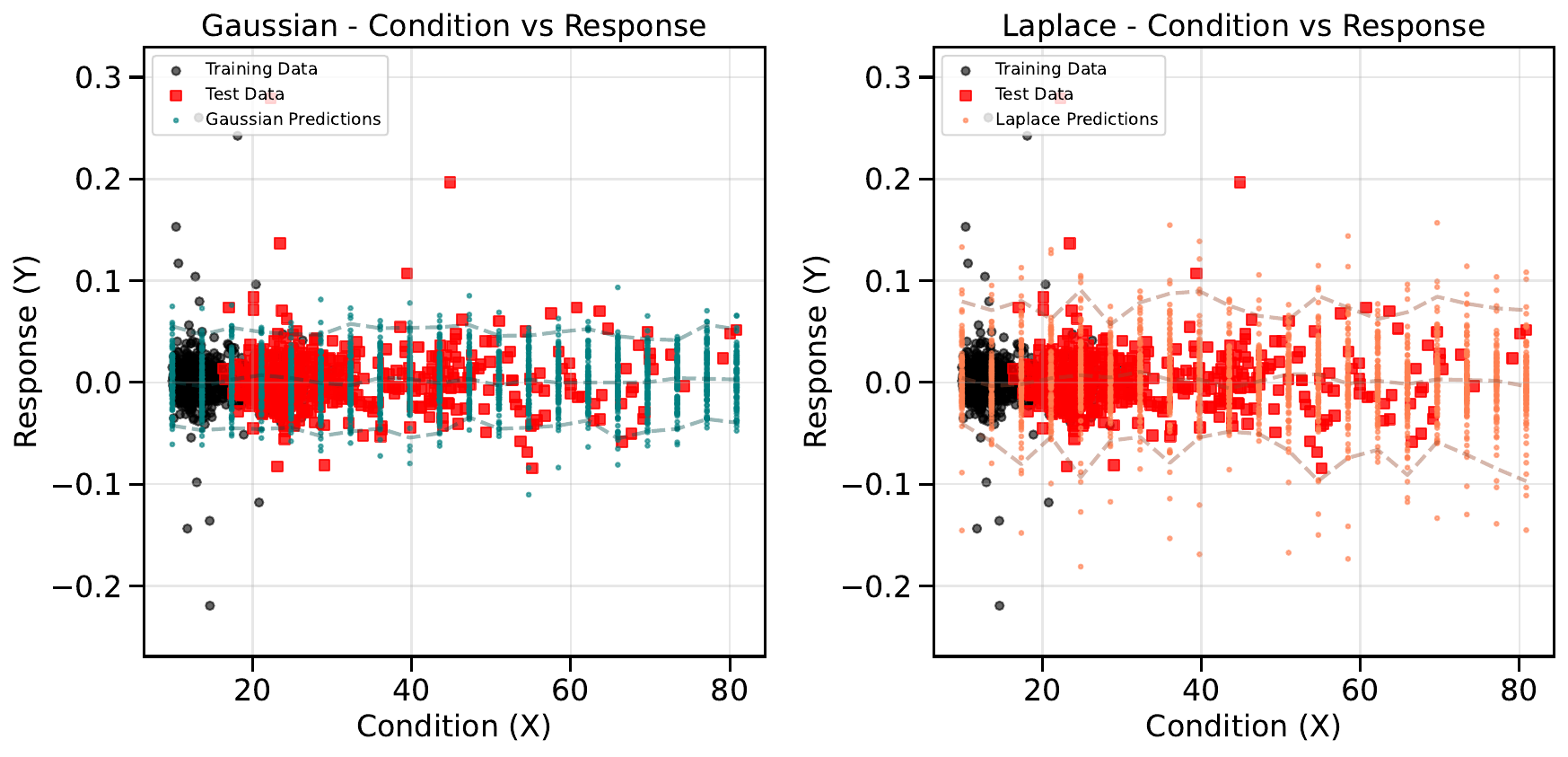}
        \caption{AMZN}
    \end{subfigure}
    \vspace{0.5em}
    \begin{subfigure}[b]{0.48\textwidth}
        \includegraphics[width=\textwidth]{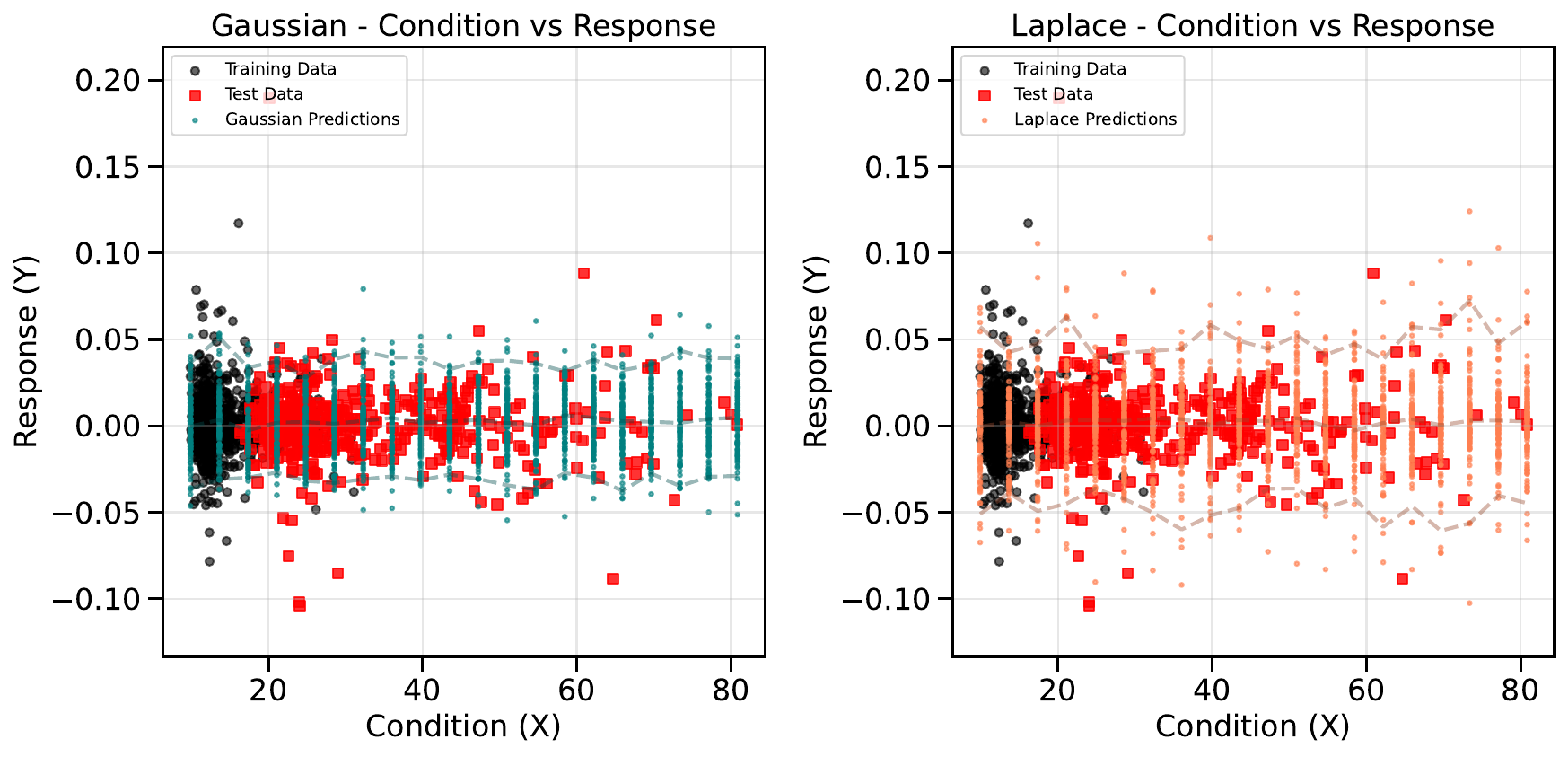}
        \caption{GOOGL}
    \end{subfigure}
    \hfill
    \begin{subfigure}[b]{0.48\textwidth}
        \includegraphics[width=\textwidth]{figures/stock_data/gfc_GS_run_1/condition_response_scatter.pdf}
        \caption{GS}
    \end{subfigure}
    \vspace{0.5em}
    \begin{subfigure}[b]{0.48\textwidth}
        \includegraphics[width=\textwidth]{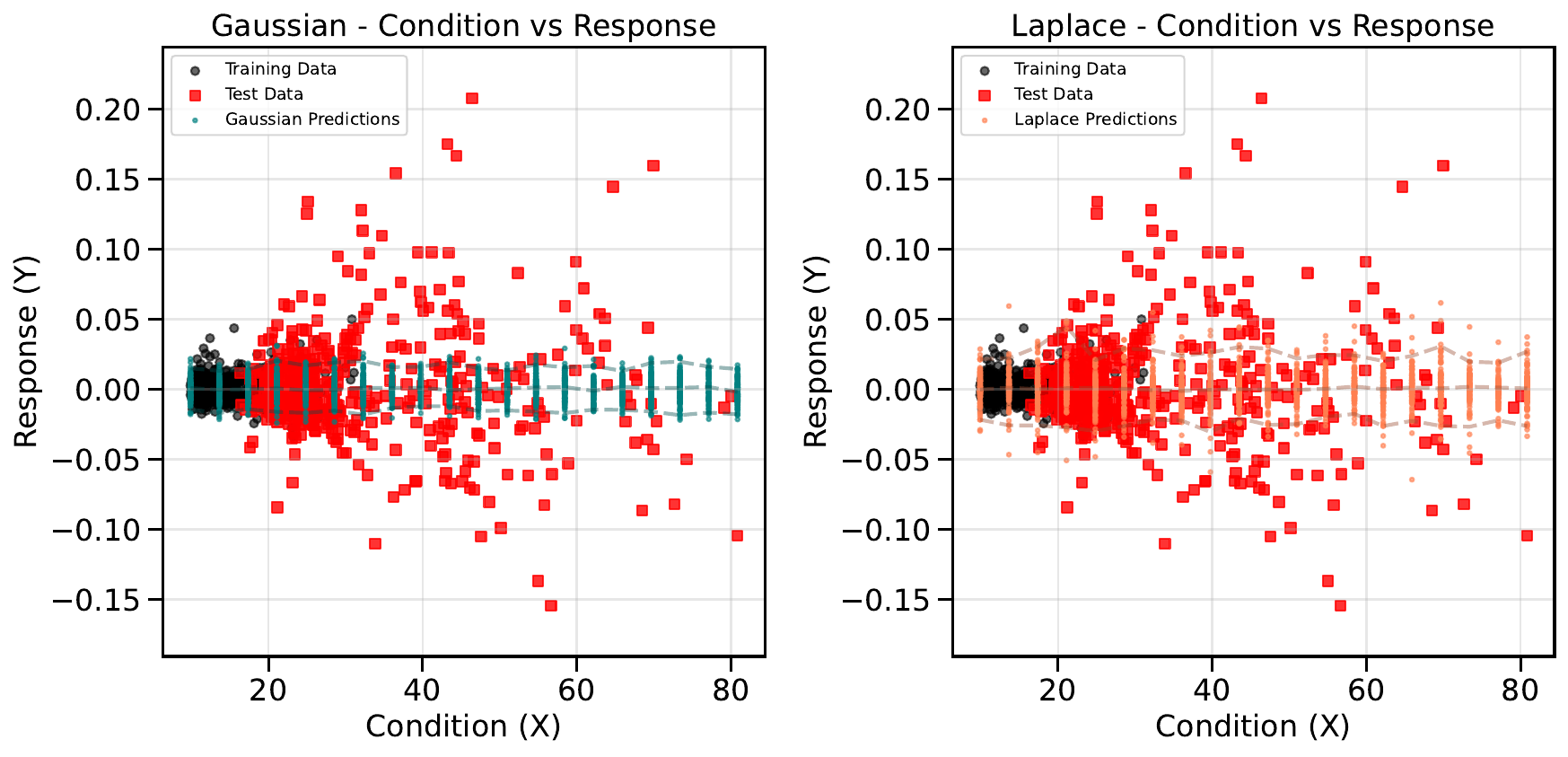}
        \caption{JPM}
    \end{subfigure}
    \hfill
    \begin{subfigure}[b]{0.48\textwidth}
        \includegraphics[width=\textwidth]{figures/stock_data/gfc_MSFT_run_1/condition_response_scatter.pdf}
        \caption{MSFT}
    \end{subfigure}
    \vspace{0.5em}
        \begin{subfigure}[b]{0.48\textwidth}
        \includegraphics[width=\textwidth]{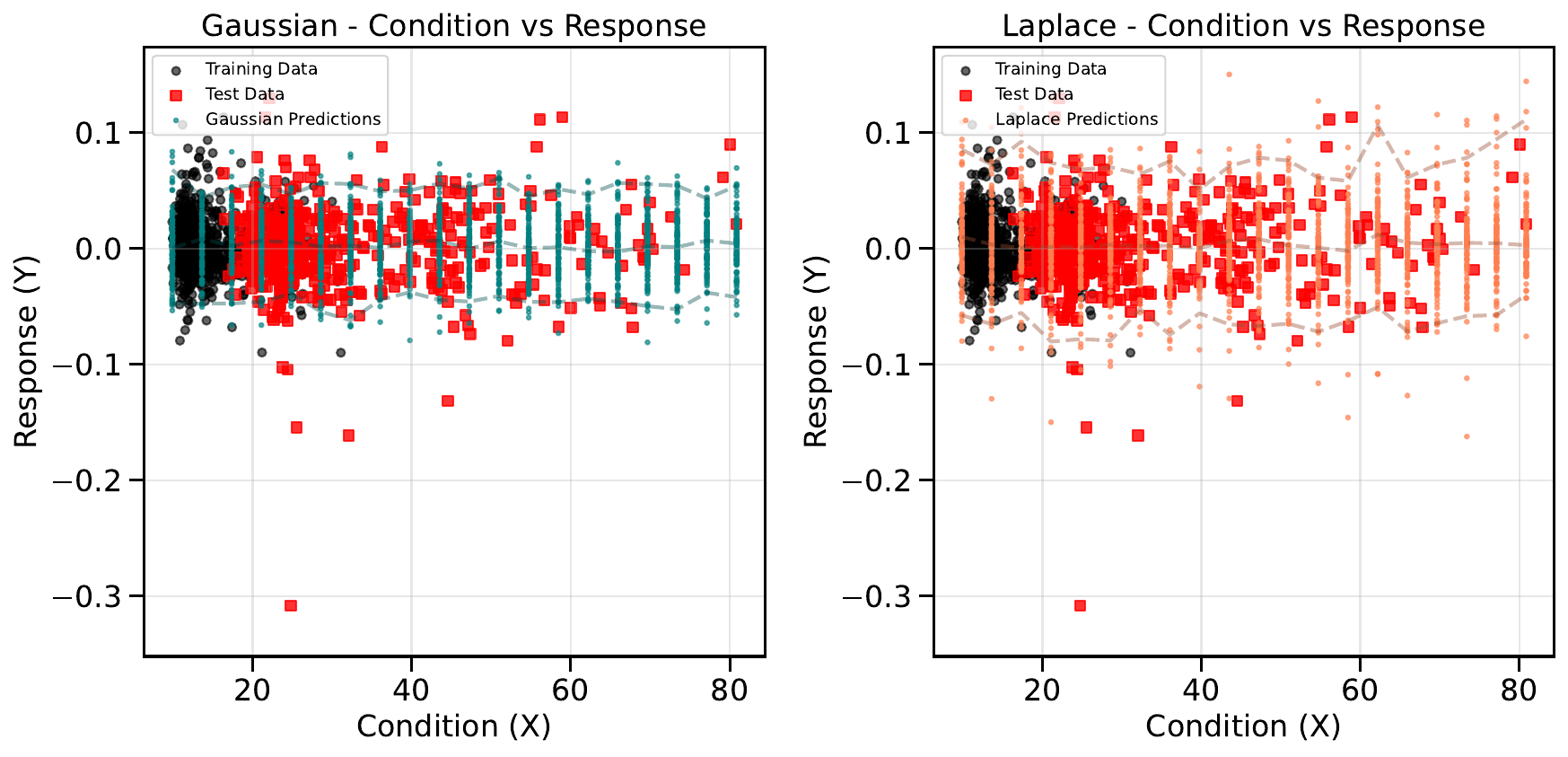}
        \caption{NVDA}
    \end{subfigure}
    \hfill
    \begin{subfigure}[b]{0.48\textwidth}
        \includegraphics[width=\textwidth]{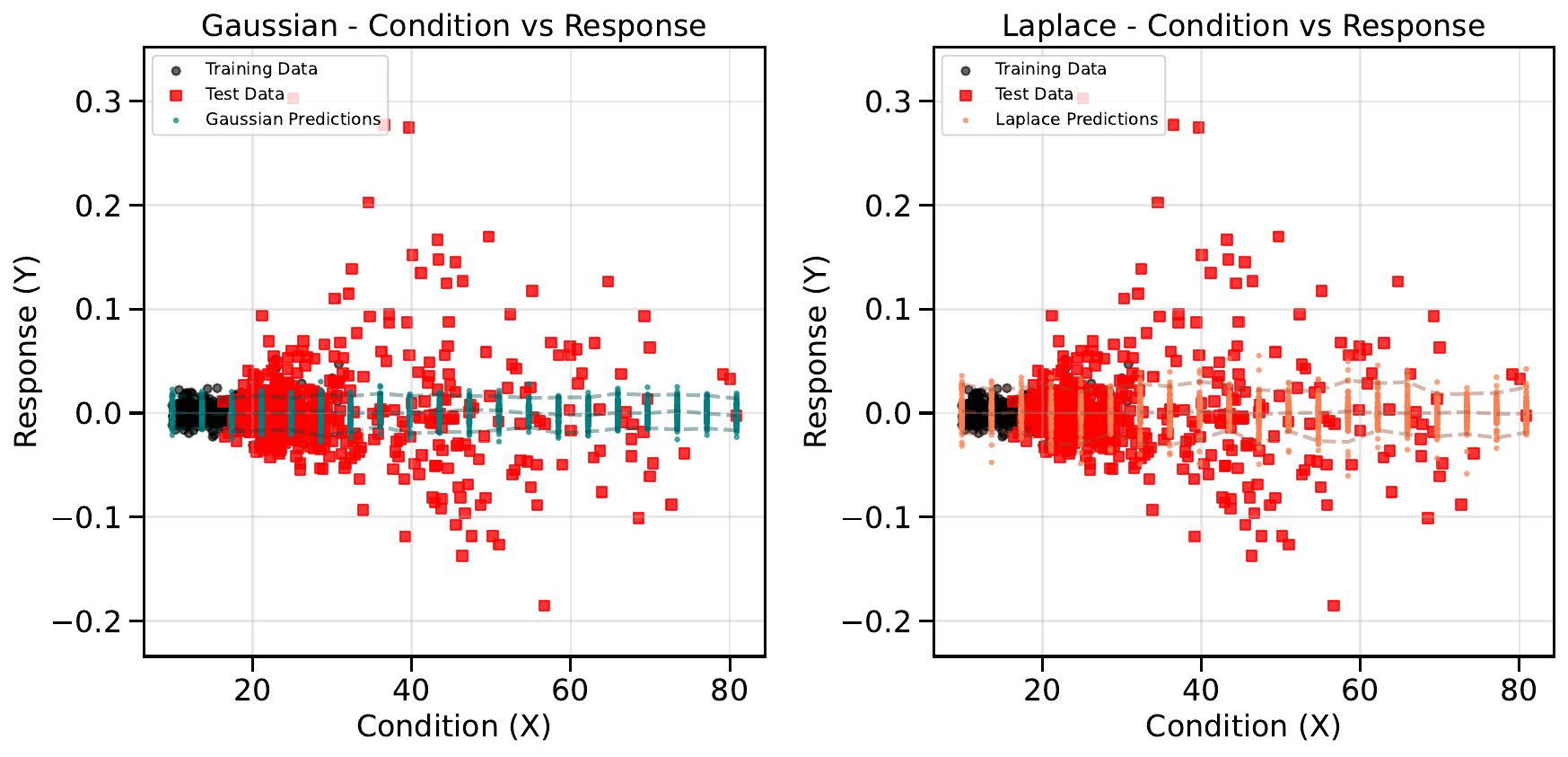}
        \caption{WFC}
    \end{subfigure}
    \caption{Scatter plots for visualization of conditional generation performance for GFC period.}
    \label{fig: conditional_scatter_gfc}
\end{figure}

\begin{figure}[htbp]
    \centering
    \begin{subfigure}[b]{0.48\textwidth}
        \includegraphics[width=\textwidth]{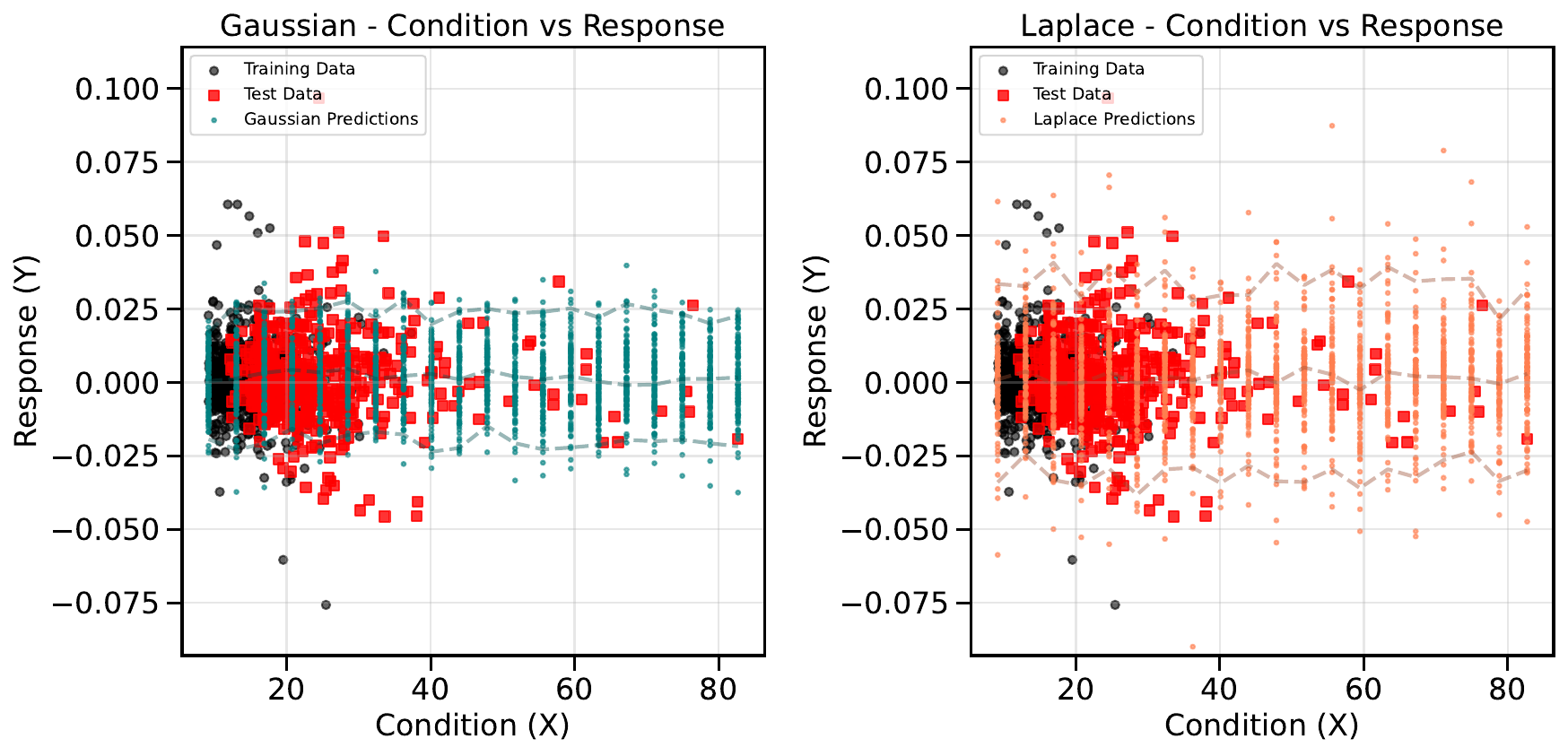}
        \caption{AAPL}
    \end{subfigure}
    \hfill
    \begin{subfigure}[b]{0.48\textwidth}
        \includegraphics[width=\textwidth]{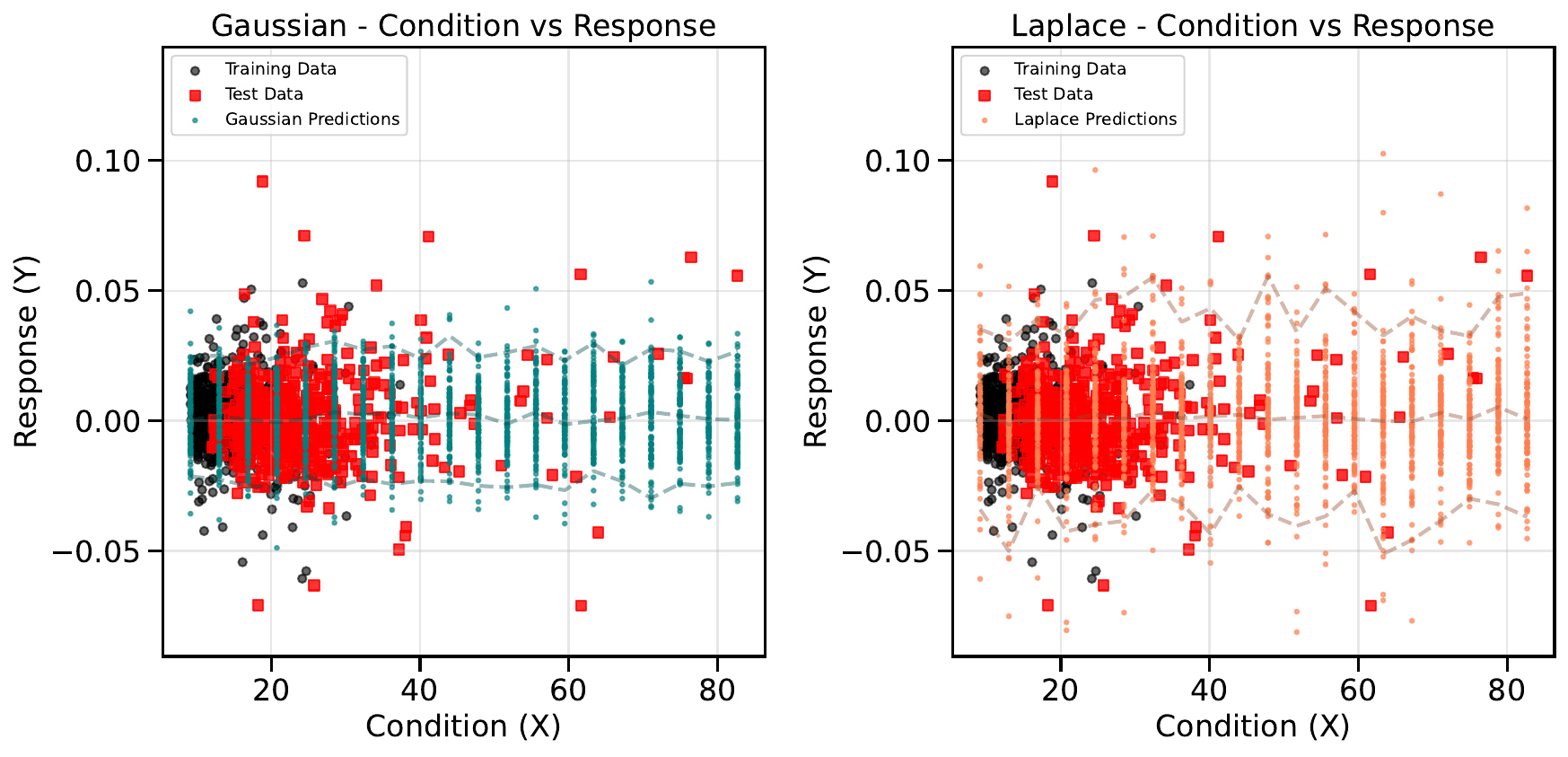}
        \caption{AMZN}
    \end{subfigure}
    \vspace{0.5em}
    \begin{subfigure}[b]{0.48\textwidth}
        \includegraphics[width=\textwidth]{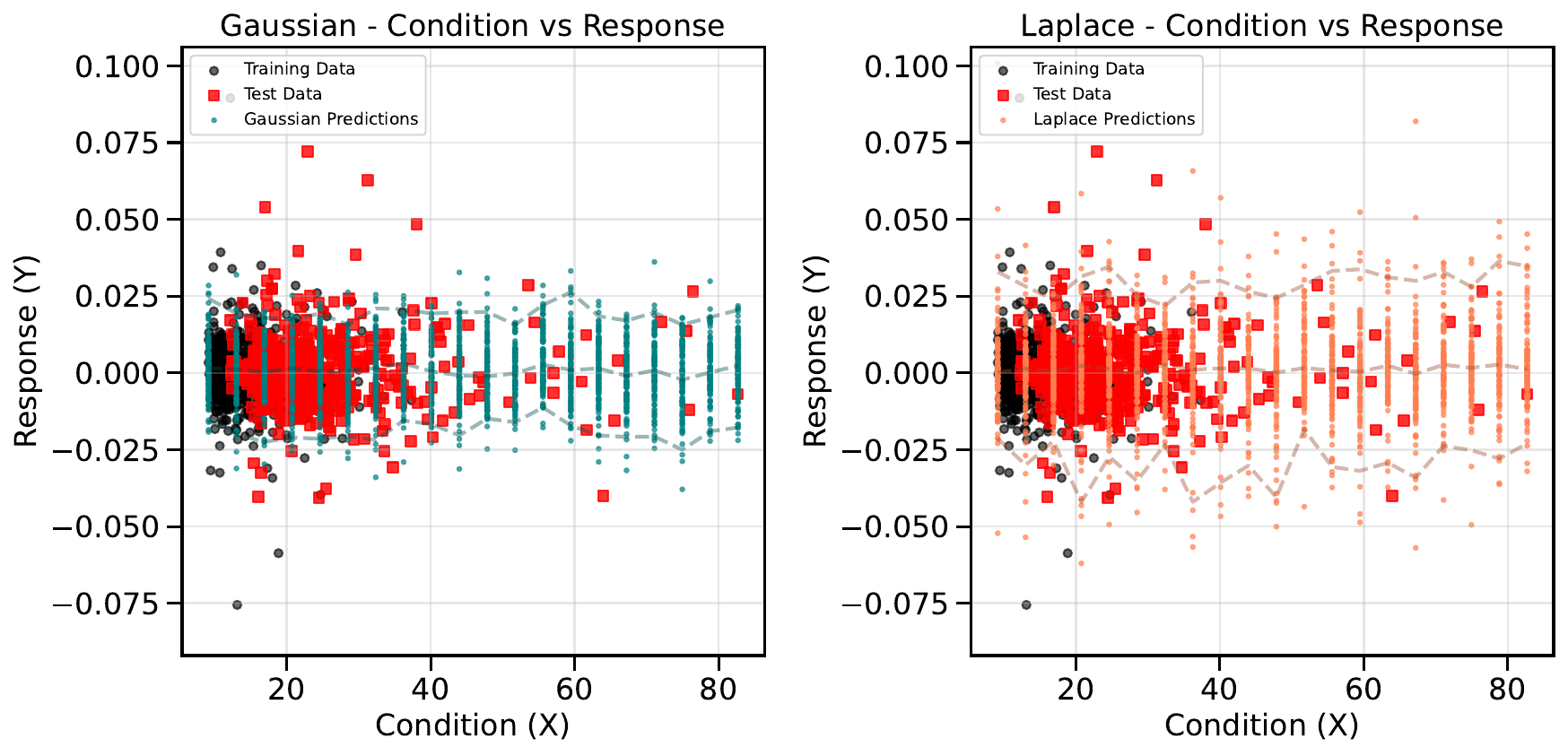}
        \caption{GOOGL}
    \end{subfigure}
    \hfill
    \begin{subfigure}[b]{0.48\textwidth}
        \includegraphics[width=\textwidth]{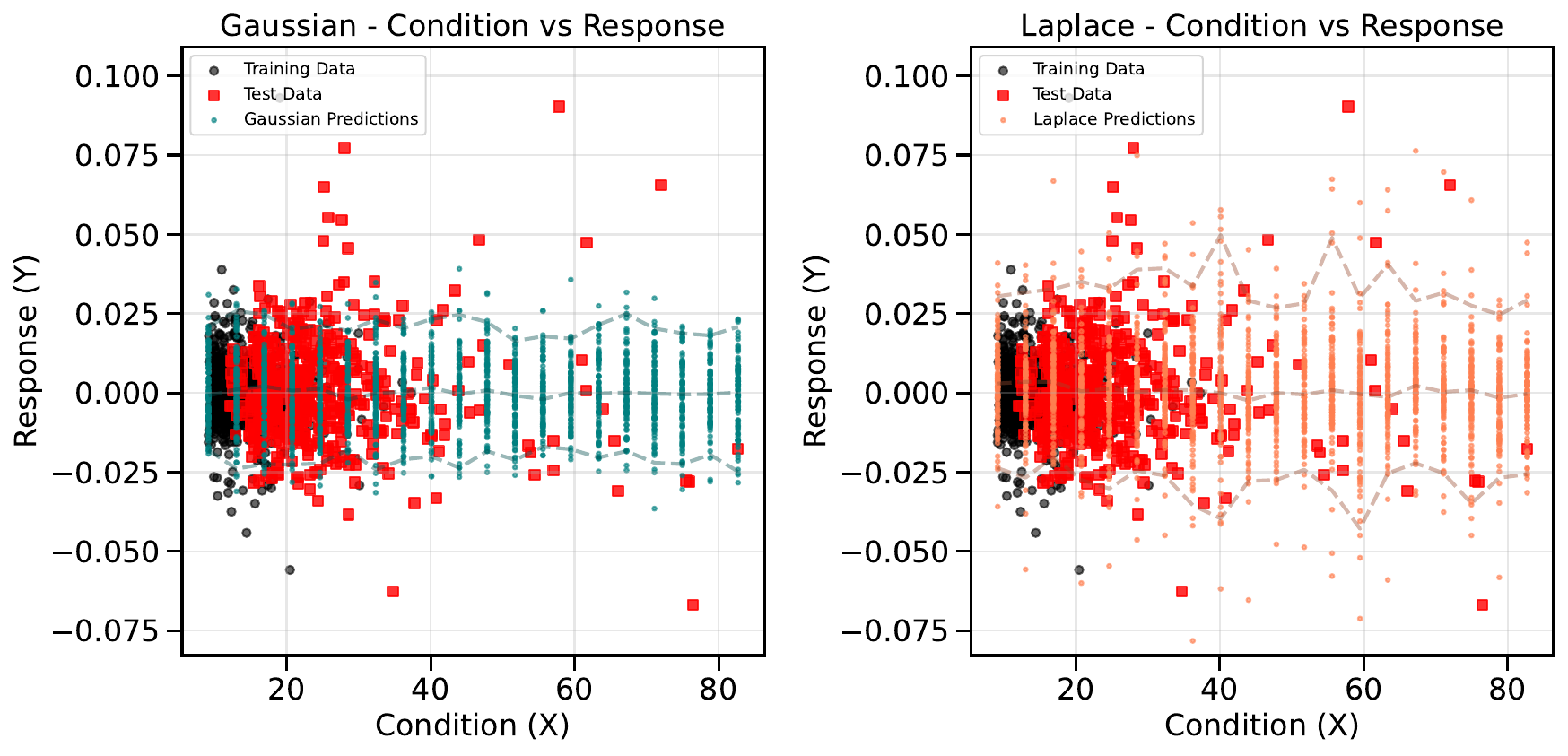}
        \caption{GS}
    \end{subfigure}
    \vspace{0.5em}
    \begin{subfigure}[b]{0.48\textwidth}
        \includegraphics[width=\textwidth]{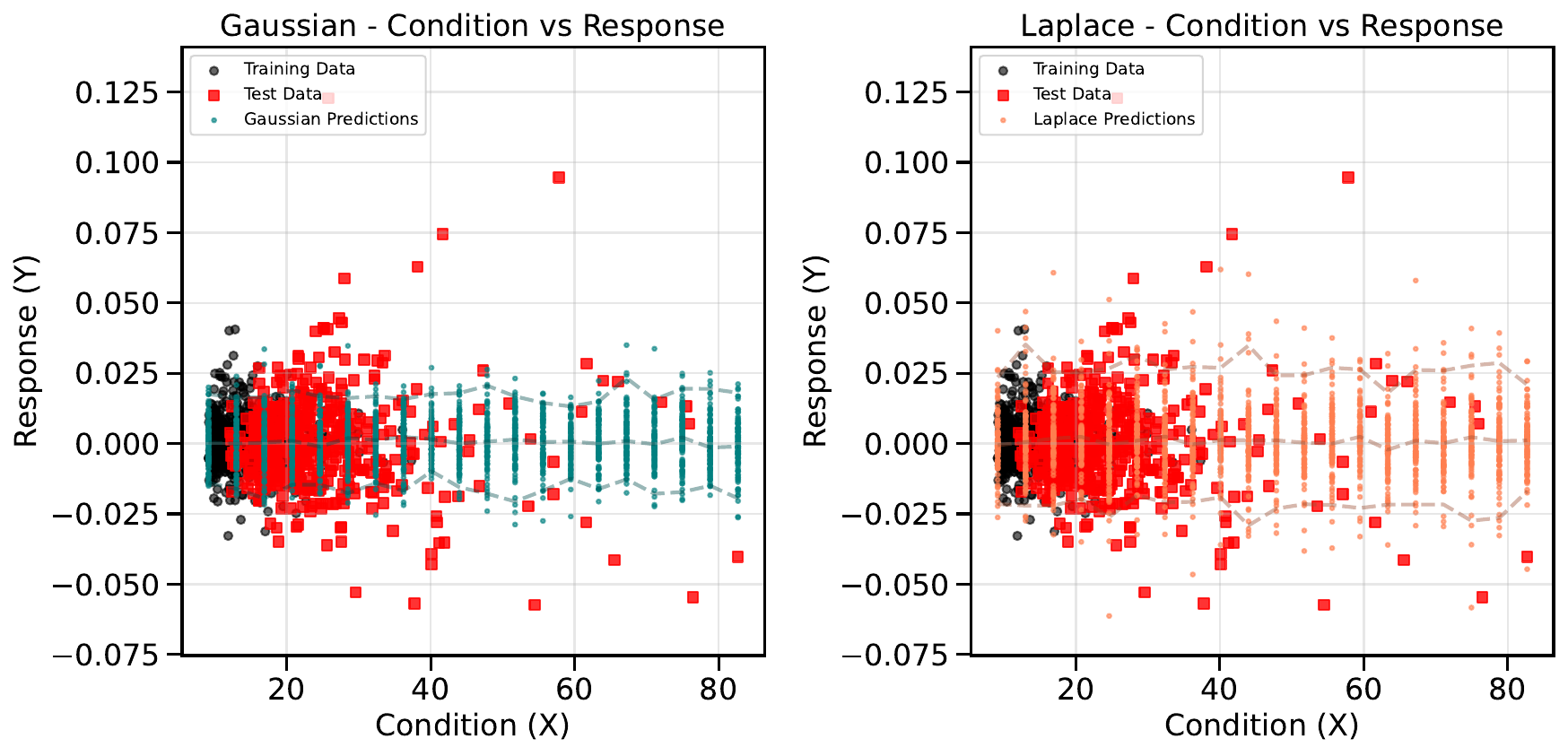}
        \caption{JPM}
    \end{subfigure}
    \hfill
    \begin{subfigure}[b]{0.48\textwidth}
        \includegraphics[width=\textwidth]{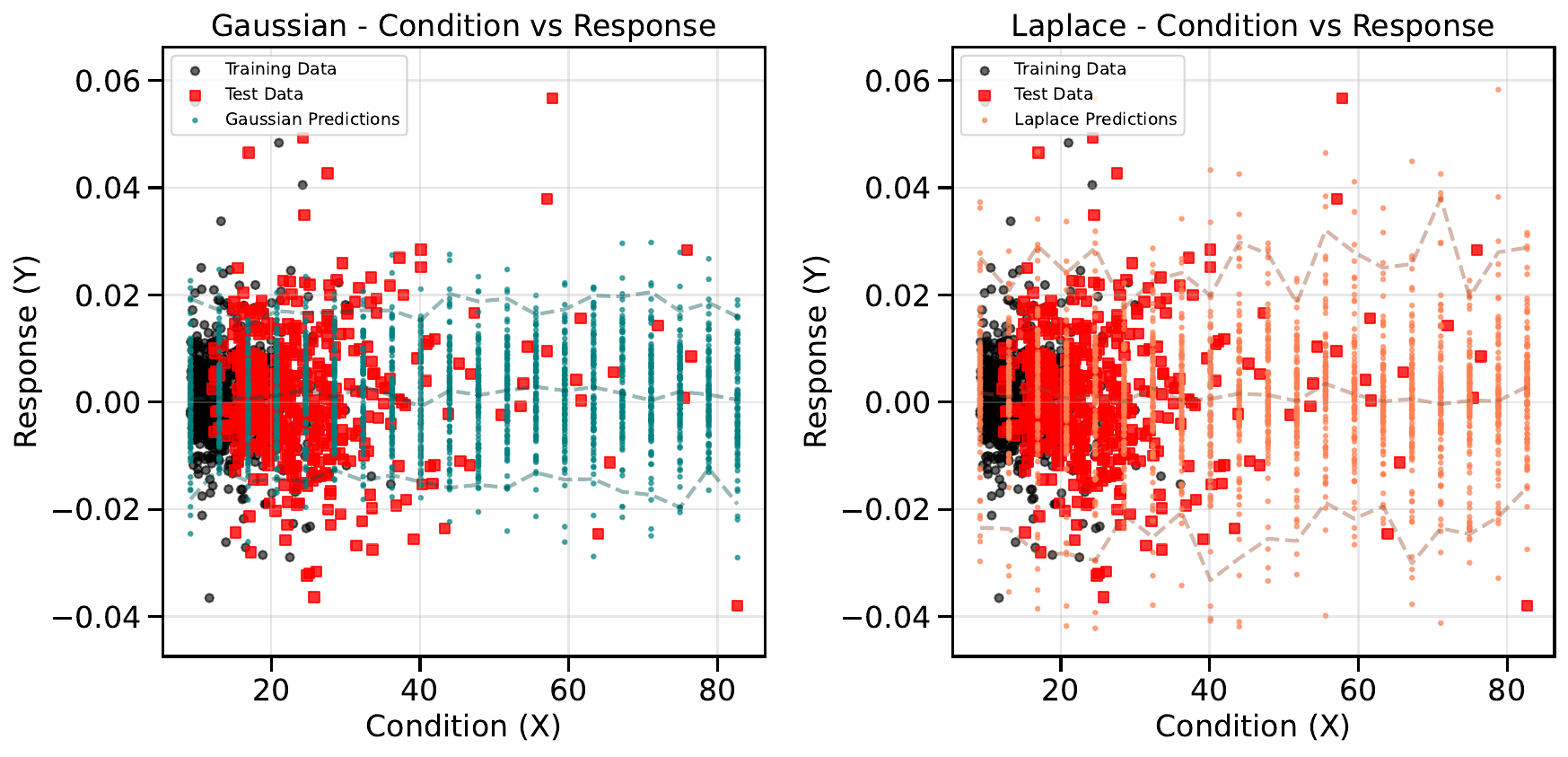}
        \caption{MSFT}
    \end{subfigure}
    \vspace{0.5em}
        \begin{subfigure}[b]{0.48\textwidth}
        \includegraphics[width=\textwidth]{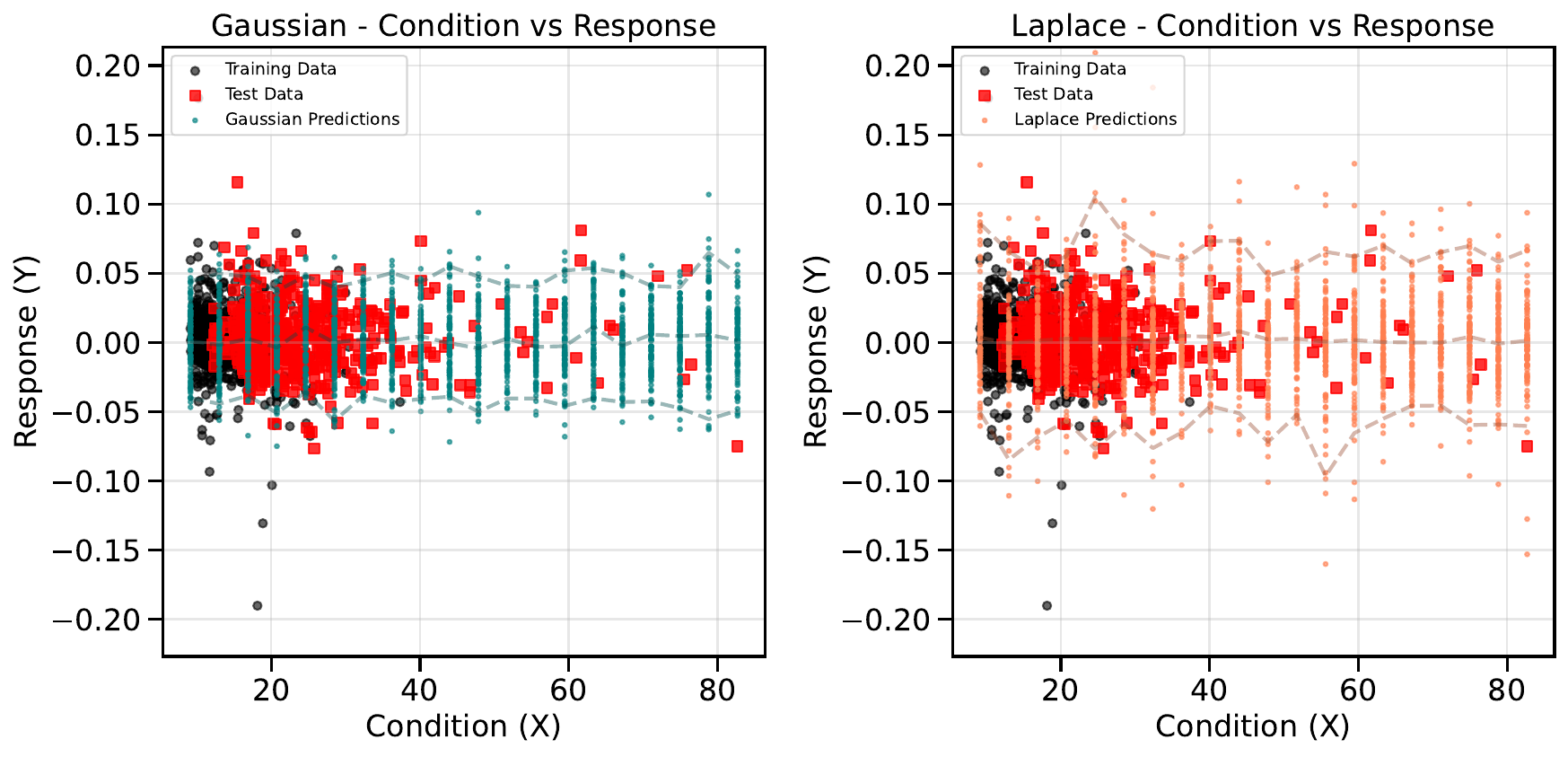}
        \caption{NVDA}
    \end{subfigure}
    \hfill
    \begin{subfigure}[b]{0.48\textwidth}
        \includegraphics[width=\textwidth]{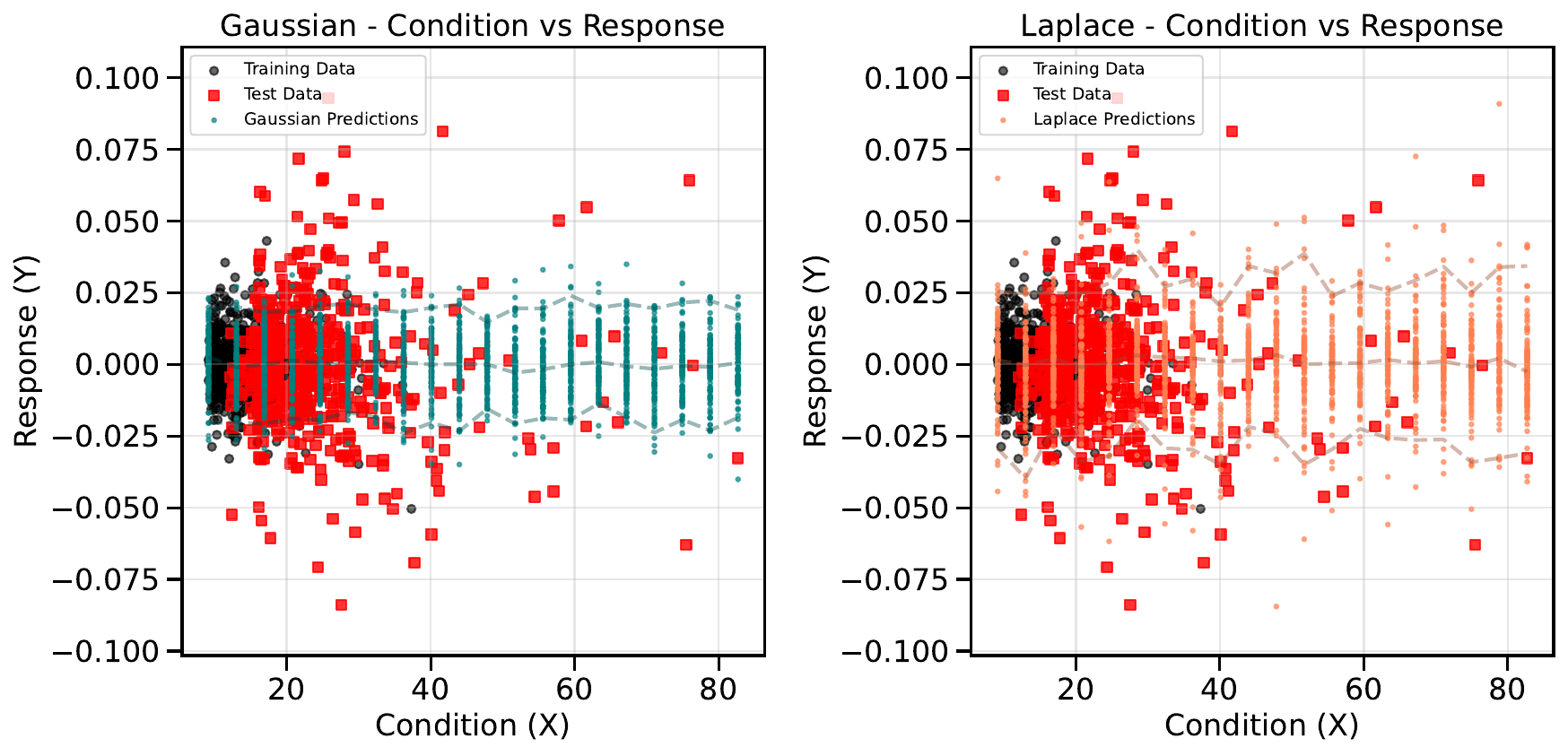}
        \caption{WFC}
    \end{subfigure}
    \caption{Scatter plots for visualization of conditional generation performance for COVID period.}
    \label{fig: conditional_scatter_covid}
\end{figure}

\end{document}